\documentclass{article}

\usepackage[final]{neurips_2025}

\usepackage{amsmath,amsfonts,bm}

\def\eqref#1{equation~\ref{#1}}

\def\1{\bm{1}}

\DeclareMathAlphabet{\mathsfit}{\encodingdefault}{\sfdefault}{m}{sl}
\SetMathAlphabet{\mathsfit}{bold}{\encodingdefault}{\sfdefault}{bx}{n}

\usepackage[linesnumbered,ruled,vlined]{algorithm2e}
\usepackage{setspace}
\usepackage{amsfonts}       
\usepackage{booktabs}       
\usepackage{enumitem}
\usepackage[T1]{fontenc}    
\usepackage[utf8]{inputenc} 
\usepackage{makecell}
\usepackage{microtype}      
\usepackage{multirow}
\usepackage{nicefrac}       
\usepackage{tabularx}
\usepackage{url}            

\usepackage{longtable}
\usepackage{filecontents,catchfile,pgffor}

\usepackage{environ}
\usepackage[percent]{overpic}
\usepackage[dvipsnames]{xcolor}

\usepackage{amsmath}
\usepackage{amssymb}
\usepackage{booktabs}
\usepackage{dsfont}
\usepackage{graphicx}
\usepackage{makecell}
\usepackage{multirow}
\usepackage{pifont}
\usepackage[group-separator={,}]{siunitx}
\usepackage{tabularx}
\usepackage{caption}
\usepackage{soul}
\usepackage{changepage,threeparttable}
\usepackage{colortbl}
\usepackage[accsupp]{axessibility}  
\usepackage{mathtools}
\usepackage{xfrac}
\usepackage{enumitem}
\usepackage{catchfile}
\usepackage{wrapfig}

\usepackage{longtable}
\usepackage{tabu}
\usepackage{tcolorbox}
\usepackage{supertabular}
\usepackage{caption}
\usepackage{subcaption}
\usepackage{pifont}
\usepackage{adjustbox}
\usepackage{titletoc}
\usepackage{multicol}
\newcolumntype{Y}{>{\centering\arraybackslash}X}
\usepackage{physics}

\DeclareMathOperator*{\argmax}{arg\,max}
\DeclareMathOperator*{\argmin}{arg\,min}

\usepackage{amsthm}
\newtheorem{theorem}{Theorem}
\newtheorem{proposition}{Proposition}

\newtheorem{corollary}{Corollary}

\newcommand{\B}{\mathbf}

\newcommand{\Ours}[0]{\texttt{RBF}}
\newcommand{\Oursbf}[0]{\fontfamily{lmtt}\fontseries{b}\selectfont \Ours{}}

\newcommand{\ie}{i.e.,}
\newcommand{\eg}{e.g.,}
\newcommand{\etal}{\emph{et al.\ }}

\definecolor{color_1}{RGB}{255,0,128}
\definecolor{color_2}{RGB}{128,128,0}
\definecolor{color_3}{RGB}{0,128,0}
\definecolor{color_4}{RGB}{128,0,0}
\definecolor{color_5}{RGB}{128,0,128}

\definecolor{Goldenrod}{RGB}{218, 165, 32}
\definecolor{ForestGreen}{RGB}{34, 139, 34}

\definecolor{highlight_color}{RGB}{255,0,128}
\definecolor{blue_highlight_color}{RGB}{0,128,255}
\definecolor{green_highlight_color}{RGB}{51,145,40}
\definecolor{orange_highlight_color}{RGB}{255,160,0}
\definecolor{purple_highlight_color}{RGB}{128,0,128}

\newcommand{\mycomment}[1]{}

\usepackage[utf8]{inputenc} 
\usepackage[T1]{fontenc}    
\usepackage{hyperref}       
\usepackage{url}            
\usepackage{booktabs}       
\usepackage{amsfonts}       
\usepackage{nicefrac}       
\usepackage{microtype}      

\title{Inference-Time Scaling for Flow Models via\\Stochastic Generation and Rollover Budget Forcing}

\author{%
  Jaihoon Kim\textsuperscript{$\ast$} $\quad$ Taehoon Yoon\textsuperscript{$\ast$} $\quad$ Jisung Hwang\textsuperscript{$\ast$} $\quad$ Minhyuk Sung \\
  KAIST \\
  {\tt\small \{jh27kim,taehoon,4011hjs,mhsung\}@kaist.ac.kr}
}

\newif\ifpaper
\paperfalse

\begin{document}

\maketitle

\begin{figure}[h!]
    \centering
    \captionsetup{type=figure}
    \vspace{-25pt}
    \includegraphics[width=0.86\textwidth]{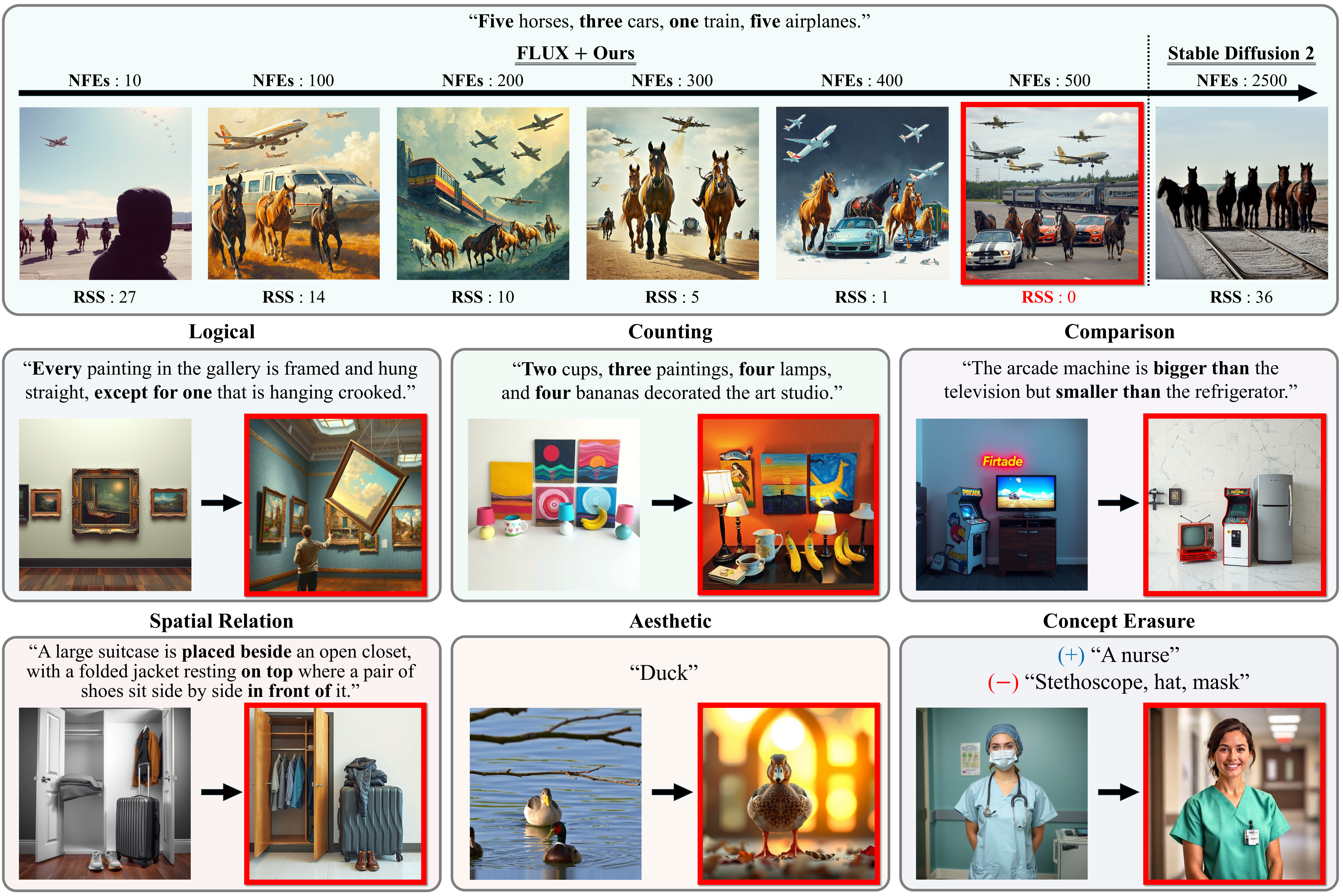}
    \vspace{-2pt}
    \caption{\textbf{Diverse applications of our inference-time scaling method.} 
    Pretrained flow models struggle to generate images that align with complex prompts (left side of each case), whereas our inference-time scaling effectively extends their capabilities to achieve precise alignment (red box). 
    }
    \label{fig:teaser}
    \vspace{-0.7\baselineskip}
\end{figure}

\def\thefootnote{*}\footnotetext{Equal contribution.}\def\thefootnote{\arabic{footnote}}

\begin{abstract}
\vspace{-0.8\baselineskip}
We propose an inference-time scaling approach for pretrained flow models. 
Recently, inference-time scaling has gained significant attention in LLMs and diffusion models, improving sample quality or better aligning outputs with user preferences by leveraging additional computation. 
For diffusion models, particle sampling has allowed more efficient scaling due to the stochasticity at intermediate denoising steps. 
On the contrary, while flow models have gained popularity--offering faster generation and high-quality outputs--efficient inference-time scaling methods used for diffusion models cannot be directly applied due to their deterministic generative process. 
To enable efficient inference-time scaling for flow models, we propose three key ideas: 1) SDE-based generation, enabling particle sampling in flow models, 2) Interpolant conversion, broadening the search space, and 3) Rollover Budget Forcing (RBF), maximizing compute utilization. 
Our experiments show that SDE-based generation and variance-preserving (VP) interpolant-based generation, improves the performance of particle sampling methods for inference-time scaling in flow models. 
Additionally, we demonstrate that RBF with VP-SDE achieves the best performance, outperforming all previous inference-time scaling approaches. Project page: \textcolor{magenta}{\url{flow-inference-time-scaling.github.io}}.
\end{abstract}

\vspace{-1.25\baselineskip}
\section{Introduction}
\label{sec:intro}
Over the past years, scaling laws of AI models have mainly focused on increasing model size and training data. However, recent advancements have shifted attention toward \emph{inference-time scaling}~\cite{Stiennon:2020BoN, Touvron:2023llama2}, leveraging computational resources during inference to enhance model performance. 
OpenAI o1~\cite{openai:o1} and DeepSeek R1~\cite{deekseek:r1} exemplify this approach, demonstrating consistent output improvements with increased inference computation. Recent research in LLMs~\cite{Muennighoff:2025s1} attempting to replicate such improvements has introduced \emph{test-time budget forcing}, achieving high efficiency with limited token sampling during inference. 

For diffusion models~\cite{Sohl-Dickstein:2025Thermo,Song:2021SDE}, which are widely used for generation tasks, research on inference-time scaling has been growing in the context of reward-based sampling~\cite{Kim:2025DAS, Li2024:SVDD, Singh:2025CoDE}. Given a reward function that measures alignment with user preferences~\cite{Kirstain2023:pickapic} or output quality~\cite{Schuhmann:aesthetics, Lin:2024CLIPFlanT5}, the goal is to find the sample from the learned data distribution that best aligns with the reward through repeated sampling. 
Fig.~\ref{fig:teaser} showcases diverse applications of inference-time scaling using our method, enabling the generation of faithful images that accurately align with complex user descriptions involving objects quantities, logical relationships, and conceptual attributes. Notably, na\"ive generation from text-to-image models~\cite{Rombach:2022LDM, BlackForestLabs:2024Flux} often fails to fully meet user specifications, highlighting the effectiveness of inference-time scaling. 

Our goal in this work is to extend the inference-time scaling capabilities of diffusion models to flow models. 
Flow models~\cite{Lipman:2023CFM} power state-of-the-art image~\cite{Esser2403:scaling, BlackForestLabs:2024Flux} and video generation~\cite{chen2025goku, opensora}, achieving high-quality synthesis with few inference steps, enabled by trajectory stratification techniques during training~\cite{Liu:2023RF}. 
Beyond just speed, recent pretrained flow models, equipped with enhanced text-image embeddings~\cite{Raffel:2020t5} and advanced architectures~\cite{Esser2403:scaling}, significantly outperform previous pretrained diffusion models in both image and video generation quality. 


Despite their advantages in generating high-quality results more efficiently than diffusion models, flow models have an inherent limitation in the context of inference-time scaling. Due to their ODE-based deterministic generative process, they cannot directly incorporate particle sampling at intermediate steps, a key mechanism for effective inference-time scaling in diffusion models. Building on the formulation of stochastic interpolant framework~\cite{Albergo:2023Interpolant}, we adopt an SDE-based sampling method for flow models at inference-time, enabling particle sampling for reward alignment. 


To further expand the exploration space, we consider not only stochasticity but also the choice of the \emph{interpolant}. While typical flow models use a linear interpolant, diffusion models commonly adopt a Variance-Preserving (VP) interpolant~\cite{Song:2021SDE, Ho:2020DDPM}.
Inspired by this, for the first time, we incorporate the VP interpolant into the particle sampling of flow models and demonstrate its effectiveness in increasing sample diversity, enhancing the likelihood of discovering high-reward samples. 

We emphasize that while we propose converting the generative process of a pretrained flow model to align with that of diffusion models—\ie VP-SDE-based generation—inference-time scaling with flow models offers significant advantages over diffusion models. 
Flow models, particularly those with rectification fine-tuning~\cite{Liu:2023RF, liu2024instaflow}, produce much clearer expected outputs at intermediate steps, enabling more precise future reward estimation and, in turn, more effective particle sampling. 

We additionally explore a strategy for tight budget enforcement in terms of the number of function evaluations (NFEs) of the velocity prediction network. Previous particle-sampling-based inference-time scaling approaches for diffusion models~\cite{Li2024:SVDD, Singh:2025CoDE} allocate the NFEs budget \emph{uniformly} across timesteps in the generative process, which we empirically found to be ineffective in practice. To optimize budget utilization, we propose \emph{Rollover Budget Forcing}, a method that adaptively reallocates NFEs across timesteps. 
Specifically, we perform a denoising step upon identifying a new particle with a higher expected future reward and allocate the remaining NFEs to subsequent timesteps. 


Experimentally, we demonstrate that our inference-time SDE conversion and VP interpolant conversion enable efficient particle sampling in flow models, leading to consistent improvements in reward alignment across two challenging tasks: compositional text-to-image generation and quantity-aware image generation. 
Additionally, our Rollover Budget Forcing (\Ours{}) provides further performance gains, outperforming all previous particle sampling approaches. 
We also demonstrate that for differentiable rewards, such as aesthetic image generation, integrating \Ours{} with a gradient-based method~\cite{Chung:2023DPS} creates a synergistic effect, leading to further performance improvements.

In summary, we introduce an inference-time scaling for flow models, analyzing three key factors:
\begin{itemize}[leftmargin=*]
\item ODE vs. SDE: We introduce an \emph{SDE generative process} for flow models to enable particle sampling. 
\item Interpolant: We demonstrate that replacing the linear interpolant of flow models with \emph{Variance Preserving interpolant} expands the search space, facilitating the discovery of higher-reward samples.
\item NFEs Allocation: We propose \emph{Rollover Budget Forcing} that adaptively allocates NFEs across timesteps to ensure efficient utilization of the available compute budget.

\end{itemize}

\vspace{-0.5\baselineskip}
\section{Related Work}
\vspace{-0.5\baselineskip}
\label{sec:related}
\vspace{-0.2\baselineskip}
\subsection{Reward Alignment in Diffusion Models}
\vspace{-0.5\baselineskip}
In the literature of diffusion models, reward alignment approaches can be broadly categorized into fine-tuning-based methods~\cite{Black2024:DDPO, Yang2024:D3PO, Wallace:2024DiffusionDPO, Clark2024:DRaFT, Prabhudesai2023:AlignProp, Xu2023:ImageReward} and inference-time-scaling-based methods~\cite{Li2024:SVDD, Singh:2025CoDE, Dou2024:FPS, Wu:2023TDS, Cardoso2024:MCGDiff}. 
While fine-tuning diffusion models enables the generation of samples aligned with user preferences, it requires fine-tuning for each task, potentially limiting scalability. In contrast, inference-time scaling approaches offer a significant advantage as they can be applied to any reward without requiring additional fine-tuning. 
Moreover, inference-time scaling can also be applied to fine-tuned models to further enhance alignment with the reward. 
Since our proposed approach is an inference-time scaling method, we focus our review on related literature in this domain. 

When the reward is differentiable, gradient-based methods~\cite{Chung:2023DPS, Bansal:2023UGD, Yu:2023FreeDOM, Eyring2024:ReNO, Guo:2024InitNO, Wallace:2023DOODL, Benhamu:2024DFlow} have been extensively studied. 
We note that inference-time scaling can be integrated with gradient-based approaches to achieve synergistic performance improvements. 

\vspace{-0.55\baselineskip}
\subsection{Particle Sampling with Diffusion Models}
\vspace{-0.5\baselineskip}
\label{sec:inference_time_scaling}
The simplest iterative sampling method that can be applied to any generative model is Best-of-N (BoN)~\cite{Stiennon:2020BoN, Touvron:2023llama2, Tang:2024Realfill}, which generates $N$ samples and selects the one with the highest reward. For diffusion models, however, incorporating particle sampling during the denoising process has been shown to be far more effective than na\"ive BoN~\cite{Singh:2025CoDE, Li2024:SVDD}. This idea has been further developed through various approaches that sample particles at intermediate steps. For instance, SVDD~\cite{Li2024:SVDD} proposed selecting the particle with the highest reward at every step. 
CoDe~\cite{Singh:2025CoDE} extends this idea by selecting the highest-reward particle only at specific intervals. 
On the other hand, methods based on Sequential Monte Carlo (SMC)~\cite{Wu:2023TDS, Cardoso2024:MCGDiff, Kim:2025DAS, Dou2024:FPS} employ a probabilistic selection approach, in which particles are sampled from a multinomial distribution according to their importance weights. 
Despite the success of particle sampling approaches for diffusion models, they have not been applicable to flow models due to the absence of stochasticity in their generative process. In this work, we present the first inference-time scaling method for flow models based on particle sampling by introducing stochasticity into the generative process and further increasing sampling diversity through trajectory modification.

\vspace{-0.55\baselineskip}
\subsection{Inference-Time Scaling with Flow Models}
\vspace{-0.5\baselineskip}
To our knowledge, Search over Paths (SoP)~\cite{Ma2025:SoP} is the only inference-time scaling method proposed for flow models, which applies a forward kernel to sample particles from the deterministic sampling process of flow models. 
However, SoP does not explore the possibility of modifying the reverse kernel, which could enable the application of more diverse particle-sampling-based methods~\cite{Li2024:SVDD, Singh:2025CoDE, Kim:2025DAS}.
To the best of our knowledge, we are the \emph{first} to investigate the application of particle sampling to flow models through the lens of the reverse kernel.

\vspace{-0.5\baselineskip}
\section{Problem Definition and Background}
\vspace{-0.6\baselineskip}

\subsection{Inference-Time Reward Alignment}
\vspace{-0.4\baselineskip}
\label{sec:problem_definition}
Given a pretrained flow model that maps the source distribution, a standard Gaussian distribution $p_1$, into the data distribution $p_0$, our objective is to generate high-reward samples $\B{x}_0 \in \mathbb{R}^d$ from the pretrained flow model without additional training--a task known as inference-time reward alignment. We denote the given reward function as $r: \mathbb{R}^{d} \rightarrow \mathbb{R}$, 
which measures text alignment or user preference for a generated sample. 
Following previous works~\cite{Korbak:2022RLDM,Uehara:2024Finetuning,Uehara:2024Bridging}, our objective can be formulated as finding the following target distribution:
{\small
    \begin{align}
        \label{eq:reward_max_obj}
        p^{*}_0 &= \argmax_{q} \  \mathbb{E}_{\B{x}_0 \sim q} \underbrace{\left[ r(\B{x}_0) \right]}_{\text{Reward}} -\beta \underbrace{\mathcal{D}_{\text{KL}} \left[ q \| p_0 \right]}_{\text{KL Regularization}},
    \end{align}
}%
which maximizes the expected reward while the KL divergence term prevents $p^{*}_0(\B{x}_0)$ from deviating too far from $p_0(\B{x}_0)$, with its strength controlled by the hyperparameter $\beta$. 
As shown in previous work~\cite{Rafailov2023:DPO}, the target distribution $p^{*}_0$ can be computed as:
{\small
\begin{align}
    \label{eq:target_distribution}
    p^{*}_0(\B{x}_0) &= \frac{1}{Z} p_0(\B{x}_0) \exp \left( \frac{r(\B{x}_0)}{\beta} \right),
\end{align}
}%
where $Z$ is a normalization constant. We present details in Appendix~\ref{sec:appendix_optimal_marginal_target}. However, sampling from the target distribution is non-trivial. 

A notable approach for sampling from the target distribution is \emph{particle sampling}, which maintains a set of candidate samples—referred to as particles—and iteratively propagates high-reward samples while discarding lower-reward ones. 
When combined with the denoising process of diffusion models, particle sampling can improve the efficiency of limited computational resources in inference-time scaling. 
In the next section, we review particle sampling methods used in diffusion models and, \emph{for the first time,} we explore insights for adapting them to flow models. 

\vspace{-0.4\baselineskip}
\subsection{Particle Sampling Using Diffusion Model}
\vspace{-0.5\baselineskip}
\label{sec:motivation}

A pretrained diffusion model generates data by drawing an initial sample from the standard Gaussian distribution and iteratively sampling from the learned conditional distribution $p_{\theta}(\B{x}_{t-\Delta t} | \B{x}_{t})$. 
Building on this, previous works~\cite{Levine:2018RL, Uehara:2024Bridging} have shown that data from the target distribution in Eq.~\ref{eq:target_distribution} can be generated by performing the same denoising process while replacing the conditional distribution $p_{\theta}(\B{x}_{t-\Delta t} | \B{x}_{t})$ with the \emph{optimal policy}:

{\small
    \begin{align}
    \label{eq:optimal_policy}
        p^{*}_\theta(\B{x}_{t-\Delta t} | \B{x}_{t}) = \frac{ p_\theta(\B{x}_{t-\Delta t} | \B{x}_{t}) \exp \left(\frac{v(\B{x}_{t-\Delta t})}{\beta}  \right)}{\int p_\theta(\B{x}_{t-\Delta t} | \B{x}_{t}) \exp \Bigl(\frac{v(\B{x}_{t-\Delta t})}{\beta} \Bigr) \mathrm{d}\B{x}_{t-\Delta t}},
    \end{align}
}%
where the details are presented in Appendix~\ref{sec:appendix_optimal_soft_policy}. 
We denote $v(\cdot): \mathbb{R}^d \rightarrow \mathbb{R}$ as the optimal value function that estimates the expected future reward of the generated samples at current timestep. 
Following previous works~\cite{Chung:2023DPS,Kim:2025DAS, Li2024:SVDD,Bansal:2023UGD}, we approximate the value function using the posterior mean computed via Tweedie’s formula~\cite{Robbins1992}, given by $v(\B{x}_t) \approx r(\B{x}_{0|t})$, where $\B{x}_{0|t} \coloneq \mathbb{E}_{\B{x}_0 \sim p_{\theta}(\B{x}_0 | \B{x}_t)} \left[ \B{x}_0 \right]$.

Since directly sampling from the optimal policy distribution in Eq.~\ref{eq:optimal_policy} is nontrivial, one can first approximate the distribution using importance sampling while taking $p_{\theta}(\B{x}_{t-\Delta t} | \B{x}_{t})$ as the \emph{proposal distribution}:
{\small
\begin{align}
    \label{eq:importance_sampling}
    p^{*}_\theta(\B{x}_{t-\Delta t} | \B{x}_{t}) \approx \sum_{i=1}^K \frac{w^{(i)}_{t-\Delta t}}{\sum_{j=1}^K w^{(j)}_{t-\Delta t}} \delta_{\B{x}^{(i)}_{t-\Delta t}}, \quad
    \{ \B{x}^{(i)}_{t-\Delta t} \}_{i=1}^{K} \sim p_\theta(\B{x}_{t-\Delta t} | \B{x}_{t}), 
\end{align}
}%
where $K$ is the number of particles, $w^{(i)}_{t-\Delta t} = \exp( v ( \B{x}^{(i)}_{t-\Delta t} ) / \beta )$ is the weight, and $\delta_{\B{x}^{(i)}_{t-\Delta t}}$ is a Dirac distribution. 
SVDD~\cite{Li2024:SVDD} proposed an approximate sampling method for the optimal policy by selecting the sample with the largest weight from Eq.~\ref{eq:importance_sampling}. 

Notably, a key factor in seeking high-reward samples using particle sampling is defining the proposal distribution to sufficiently cover the distribution of high-reward samples.
Consider a scenario where high-reward samples reside in a low density region of the original data distribution, which is common when generating complex or highly specific samples that deviate from the mode of the pretrained model distribution. 
In this case, the proposal distribution must have a sufficiently large variance to effectively explore these low density regions.
This highlights the importance of the \textit{stochasticity} of the proposal distribution, which has been instrumental in the successful adoption of particle sampling in diffusion models. 
In contrast, flow models~\cite{Lipman:2023CFM} employ a \textit{deterministic} sampling process, where all particles $\B{x}_{t-\Delta t}$ drawn from $\B{x}_{t}$ are identical. 
This restricts the applicability of particle sampling methods in flow models. 
One of the main contributions is the investigation of how these particle sampling methods can be efficiently applied to flow models.

To this end, we propose an inference-time approach that introduces stochasticity into the generative process of flow models to enable particle sampling. 
We first transform the deterministic sampling process of flow models into a stochastic process (Sec.~\ref{sec:ode_sde_conversion}). 
We further identify a sampling trajectory that expands the search space of the flow models (Sec.~\ref{sec:scheduler_conversion}). 
Note that while stochastic sampling and trajectory conversion have been studied in prior works, their primary goals have been to improve sample quality~\cite{xu2023restart, yeo2025stochsync, salimans2022progressive, kingma2021variational, Ma2024:sit} or to accelerate inference~\cite{holderrieth2024generator, shaul2023bespoke, shaul2023kinetic, Karras2022:EDM}. To the best of our knowledge, we are the first to investigate sampling stochasticity and trajectory conversion for efficient particle-based sampling in flow models.

Additionally, previous particle sampling methods in diffusion models allocated a fixed computational budget (i.e., a uniform number of particles) across all denoising timesteps, potentially limiting exploration.
We explore sampling with the rollover strategy, which adaptively allocates compute across timesteps during the sampling process (Sec.~\ref{sec:roll_over}).

\noindent 
\begin{figure}[t]
    \small
    \centering
    \begin{minipage}[t!]{0.48\linewidth}
        \includegraphics[width=\linewidth]{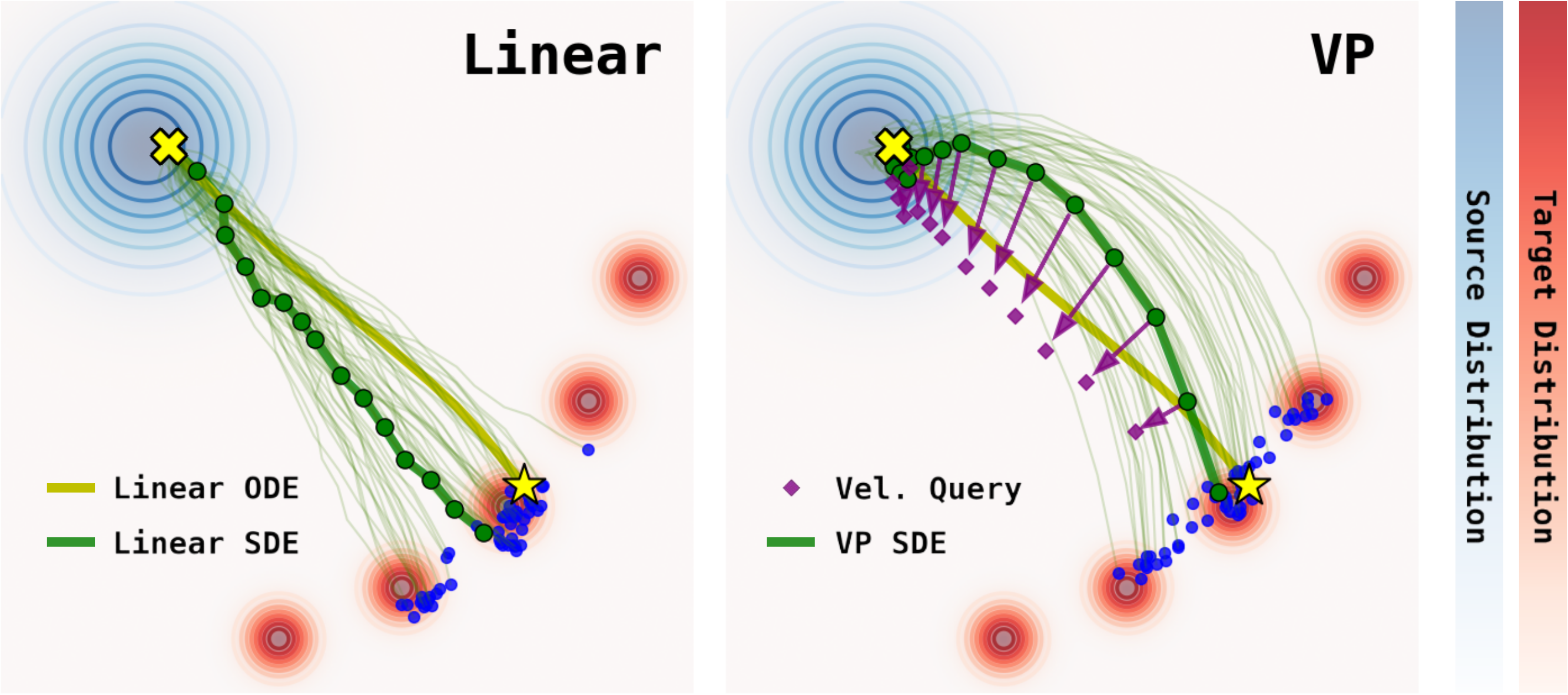}
        \caption{
            \textbf{Comparison of Linear-ODE, Linear-SDE, and VP-SDE.}
            The visualization shows how trajectories evolve under different dynamics starting from the same noise latent.
        }
        \label{fig:ode_sde_viz}
    \end{minipage}
    \hfill
    \begin{minipage}[t!]{0.48\linewidth}\vspace{-\topsep}
        {\scriptsize
        \setlength{\tabcolsep}{0.2em}
        \def\arraystretch{0.0}
        \newcolumntype{Z}{>{\centering\arraybackslash}m{0.32\linewidth}}
            \begin{tabularx}{\textwidth}{Z Z Z}
              \includegraphics[width=\linewidth]{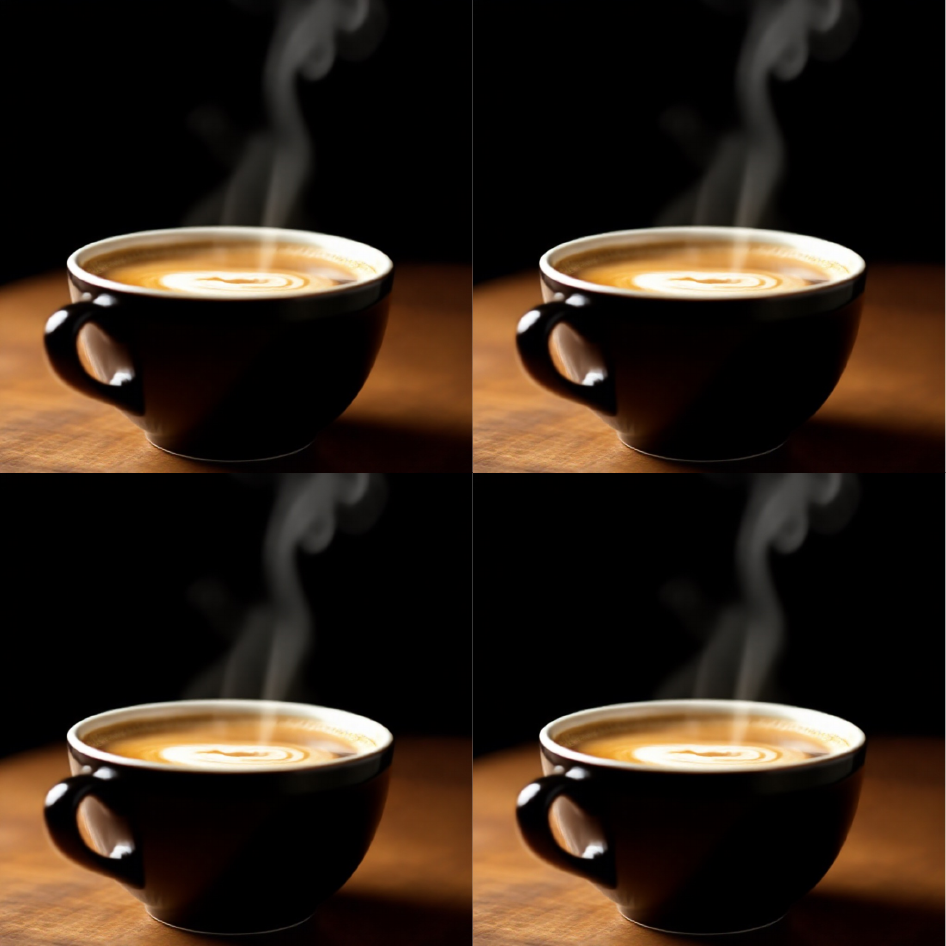}
              &  \includegraphics[width=\linewidth]{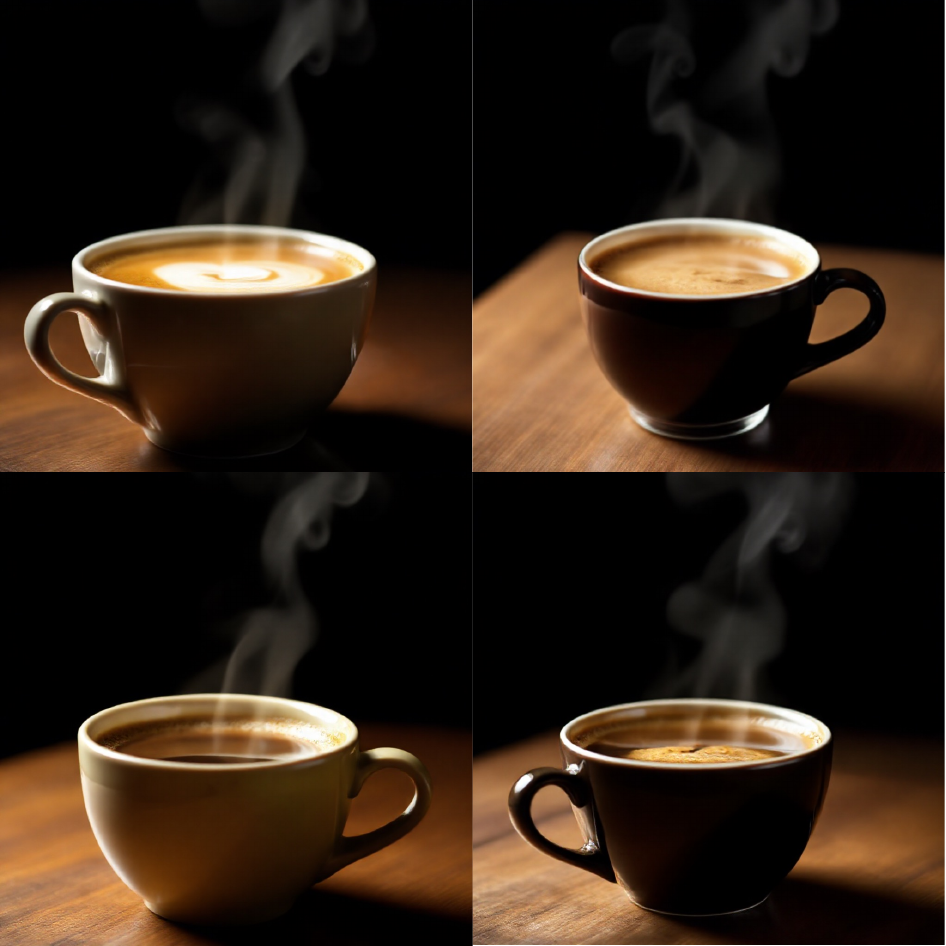}
              &  \includegraphics[width=\linewidth]{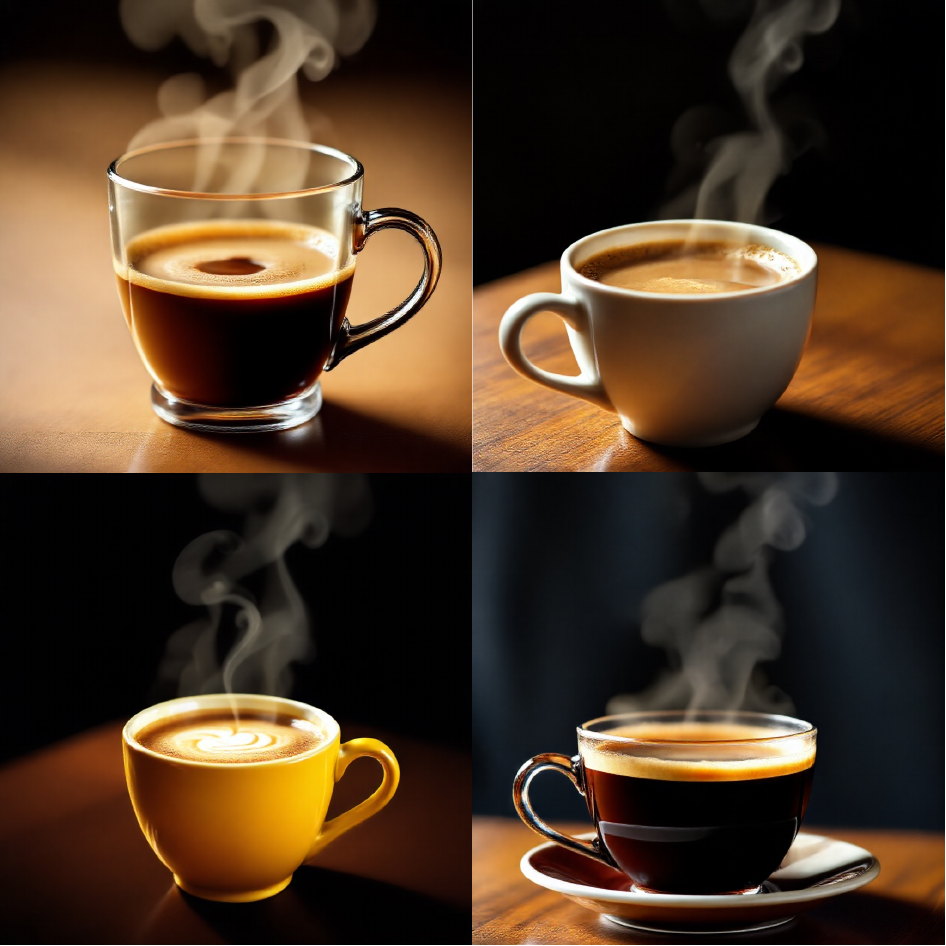} \\
              \rule{0pt}{3ex} (a) Linear-ODE
              &
              \rule{0pt}{3ex} (b) Linear-SDE
              &
              \rule{0pt}{3ex} (c) VP-SDE \\
            \end{tabularx}
        }
        \captionof{figure}{
            \textbf{Sample diversity test using FLUX~\cite{BlackForestLabs:2024Flux} under linear and VP interpolant.} All samples share the same initial latent. Prompt: \textit{``A steaming cup of coffee''}.
        }
    \label{fig:variance_test}
    \end{minipage}
    \vspace{-1.1\baselineskip}
\end{figure}



\vspace{-1.9\baselineskip}
\section{SDE-Based Particle Sampling in Flow Models}
\vspace{-0.6\baselineskip}
\label{sec:method}
In this section, we review flow and diffusion models within the unified stochastic interpolant framework (Sec.~\ref{sec:stochastic_interpolant}) and introduce our inference-time approaches for efficient particle sampling in flow models (Sec.~\ref{sec:ode_sde_conversion} and ~\ref{sec:scheduler_conversion}). 

\vspace{-0.4\baselineskip}
\subsection{Background: Stochastic Interpolant Framework}
\vspace{-0.5\baselineskip}
\label{sec:stochastic_interpolant}
At the core of both diffusion and flow models is the construction of probability paths $\{ p_t \}_{0 \leq t \leq 1}$, where $\B{x}_t \sim p_t$ serves as a bridge between $\B{x}_1 \sim p_1$ and $\B{x}_0 \sim p_0$:
\begin{align}
    \label{eq:stochastic_interpolant}
    \B{x}_t=\alpha_t \B{x}_0 + \sigma_t \B{x}_1,
\end{align}
where $\alpha_t$ and $\sigma_t$ are smooth functions satisfying $\alpha_0 = \sigma_1 = 1$, $\alpha_1 = \sigma_0 = 0$, and $\Dot{\alpha}_t < 0, \Dot{\sigma}_t > 0$; we denote the dot as a time derivative. This formulation provides a flexible choice of \textit{interpolant} $(\alpha_t, \sigma_t)$ which determines the sampling trajectory. 

\vspace{-0.4\baselineskip}
\subsection{Inference-Time SDE Conversion}
\vspace{-0.5\baselineskip}
\label{sec:ode_sde_conversion}
Flow models~\cite{Lipman:2023CFM, Liu:2023RF} learn the velocity field $u_{t} : \mathbb{R}^d \rightarrow \mathbb{R}^d$, which enables sampling of $\B{x}_0$ by solving the Probability Flow-ODE~\cite{Song:2021SDE} backward in time:
\begin{align}
    \label{eq:pf-ode}
    \mathrm{d} \B{x}_t &= u_{t}(\B{x}_t) \mathrm{d} t.
\end{align}
The deterministic process in Eq.~\ref{eq:pf-ode} accelerates the sampling process enabling few-step generation of high-fidelity samples. However, as discussed in Sec.~\ref{sec:motivation}, the deterministic nature of this sampling process limits the applicability of particle sampling in flow models. 

To address this, we transform the deterministic sampling process into a stochastic process. 
The reverse-time SDE that shares the same marginal densities as the deterministic process in Eq.~\ref{eq:pf-ode}:
\begin{align}
    \label{eq:reverse_sde}
    \mathrm{d} \B{x}_t = \B{f}_t(\B{x}_t) \mathrm{d} t + g_t \mathrm{d} \B{w}, \quad
    \B{f}_t(\B{x}_t) = u_{t}(\B{x}_t) - \frac{g_t^2}{2} \nabla \log p_t(\B{x}_t),
\end{align}
where $\B{f}_t(\B{x}_t)$ and $g_t$ represent the drift and diffusion coefficient, respectively, and $\B{w}$ is the standard Wiener process. 
This conversion introduces a noise schedule $g_t$, which can be freely chosen. Although SiT~\cite{Ma2024:sit} arrives at the same conclusion, we provide a more comprehensive proof in Appendix~\ref{sec:appendix_diffusion_coefficient_choice}. 
In our case, we set $g_t = t^2$, scaled by a factor of 3. 
Note that in the case where $g_t=0$ the process reduces to deterministic sampling in Eq.~\ref{eq:pf-ode}.

Using the velocity $u_{t}(\B{x}_t)$ predicted by a pretrained flow model, the score function $\nabla \log p_t(\B{x}_t)$ appearing in the drift coefficient $\B{f}_t(\B{x}_t)$ can be computed as:
\begin{align}
    \label{eq:score_velocity}
    \nabla \log p_t (\B{x}_t) = \frac{1}{\sigma_t} \frac{\alpha_t u_{t}(\B{x}_t) - \dot{\alpha}_t \B{x}_t}{\dot{\alpha}_t \sigma_t - \alpha_t \dot{\sigma}_t}.
\end{align}
This enables the conversion of the deterministic sampling to stochastic sampling, which we refer to as inference-time SDE conversion. 
Given the drift coefficient term $\B{f}_t(\B{x}_t)$ and diffusion coefficient $g_t$, the proposal distribution in the discrete-time domain is derived as follows:
\begin{align}
    \label{eq:proposal_dist_sde}
    p_\theta(\B{x}_{t-\Delta t} | \B{x}_t) = \mathcal{N}(\B{x}_t - \B{f}_t(\B{x}_t) \Delta t,\ 
     g_t^2 \Delta t\ \mathbf{I}).
\end{align}
While previous works have proposed converting an SDE to an ODE to improve sampling efficiency~\cite{Karras2022:EDM, Song2020:DDIM, Lu2022:DPM, Song:2021SDE}, the reverse approach—transforming an ODE into an SDE—has received relatively less attention and has primarily focused on improving sample quality~\cite{xu2023restart, Ma2024:sit}. 
To the best of our knowledge, this work is the \emph{first} to explore SDE conversion in flow models specifically to expand the search space of proposal distribution for efficient particle sampling. 

Since flow models utilize the linear interpolant $(\alpha_t = 1 - t, \sigma_t = t)$, we refer to the generative processes of the flow models using Eq.~\ref{eq:pf-ode} and Eq.~\ref{eq:reverse_sde} as Linear-ODE and Linear-SDE, respectively. 
In Fig.~\ref{fig:ode_sde_viz} (left), we visualize the sampling trajectories of Linear-ODE and Linear-SDE. 
The samples generated using Linear-ODE are identical and collapse to a single point, restricting exploration. 
In contrast, Linear-SDE introduces sample variance, allowing for broader exploration and increasing the likelihood of discovering high-reward samples. 

In Fig.~\ref{fig:variance_test}~(a-b), we visualize images sampled from Linear-ODE and Linear-SDE using FLUX~\cite{BlackForestLabs:2024Flux}. 
As discussed previously, the particles drawn from the proposal distribution of Linear-ODE are identical. In contrast, Linear-SDE introduces variation across different particles, thereby expanding the search space for identifying high-reward samples.
In the next section, we introduce \emph{inference-time interpolant conversion}, which further increases the search space.


\vspace{-0.4\baselineskip}
\subsection{Inference-Time Interpolant Conversion}
\vspace{-0.5\baselineskip}
\label{sec:scheduler_conversion} 
To further expand the search space of Linear-SDE, we draw inspiration from the effective use of particle sampling in diffusion models, where we identified a key difference: the \textit{interpolant}.
While the forward process in diffusion models follows the Variance Preserving (VP) interpolant $(\alpha_t = \exp^{-\frac{1}{2} \int_0^t \beta_s \mathrm{d}s}, \sigma_t = \sqrt{1 - \exp^{-\int_0^t \beta_s \mathrm{d}s}})$, with $\beta_s$ denoting a predefined variance schedule, flow models adopt a linear interpolant.

As shown in the previous works~\cite{Lipman2024:flowguide, shaul2023bespoke}, we note that given a velocity model $u_t$ based on an interpolant $(\alpha_t$, $\sigma_t)$ (\eg~linear), one can transform the vector field and generate a sample based on a new interpolant $(\Bar{\alpha}_s, \Bar{\sigma}_s)$ (\eg~VP) at inference-time. 
The two paths $ \Bar{\B{x}}_s = \Bar{\alpha}_s \B{x}_0 + \Bar{\sigma}_s \B{x}_1 $ and $ \B{x}_t = \alpha_t \B{x}_0 + \sigma_t \B{x}_1 $ are connected through scale-time transformation:
\begin{align}
    \label{eq:scale_time_trans}
    \Bar{\B{x}}_s &= c_s \B{x}_{t_s} \quad t_s = \rho^{-1} (\Bar{\rho}(s)) \quad  c_s = \Bar{\sigma}_s / \sigma_{t_s},
\end{align}
where $\rho(t) = \frac{\alpha_t}{\sigma_t}$ and $\Bar{\rho}(s)=\frac{\Bar{\alpha}_s}{\Bar{\sigma}_s}$ define the signal-to-noise ratio of the original and the new interpolant, respectively. The velocity for the new interpolant is given as:
{\small
\begin{align}
    \label{eq:velocity_transform}
    \Bar{u}_s(\Bar{\B{x}}_s) = \frac{\dot{c}_s}{c_s} \Bar{\B{x}}_s & + c_s \dot{t}_s u_{t_s} \left(\frac{\Bar{\B{x}}_{s}}{c_s} \right), \quad 
    \dot{c}_s = \frac{\sigma_{t_s} \dot{\Bar{\sigma}}_s - \Bar{\sigma}_s \dot{\sigma}_{t_s} \dot{t}_s}{\sigma_{t_s}^2} & 
    \quad \dot{t}_s = \frac{\sigma_{t_s}^2 \left( \Bar{\sigma}_s \dot{\Bar{\alpha}}_s - \Bar{\alpha}_s \dot{\Bar{\sigma}}_s \right)}{\Bar{\sigma}_s^2 \left( \sigma_{t_s} \dot{\alpha}_{t_s} - \alpha_{t_s} \dot{\sigma}_{t_s} \right)}.
\end{align}
}%
Plugging the transformed velocity into the proposal distribution in Eq.~\ref{eq:proposal_dist_sde} after computing the score using Eq.~\ref{eq:score_velocity} gives our efficient proposal distribution.
\begin{align}
    \label{eq:vp_sde_proposal}
    \Bar{p}_\theta(\Bar{\B{x}}_{s-\Delta s} | \Bar{\B{x}}_s) = \mathcal{N}\left(\Bar{\B{x}}_s - \left[ \Bar{u}_s(\Bar{\B{x}}_s) - \frac{g_s^2}{2} \nabla \log \Bar{p}_s(\Bar{\B{x}}_s) \right] \Delta s,\ 
     g_s^2 \Delta s\ \mathbf{I}\right).
\end{align}
Since the new trajectory follows the VP interpolant, we refer to this as VP-SDE. 
We visualize VP-SDE sampling in Fig.~\ref{fig:ode_sde_viz}~(right). At inference-time, we query the velocity of the new interpolant from the original interpolant (purple arrow). 
In Fig.~\ref{fig:variance_test}~(c), we visualize the sample diversity under VP-SDE using FLUX~\cite{BlackForestLabs:2024Flux} which generates more diverse samples than Linear-SDE. 
This property of VP-SDE effectively expands the search space, improving particle sampling efficiency in flow models. 
In Sec.~\ref{sec:interpolant_diversity_analysis}, we provide further analysis on how interpolant conversion contributes to sample diversity.

Previous works focused on interpolant conversion that enables stable training~\cite{salimans2022progressive, dhariwal2021diffusion, kingma2021variational} and accelerated inference~\cite{shaul2023kinetic, Karras2022:EDM, shaul2023bespoke}. 
We utilize interpolant conversion to enhance the sample diversity in particle sampling, which has \emph{not} been unexplored before. 
Importantly, while we modify the generative process of flow models to align with that of diffusion models, inference-time scaling with flow models still provides distinct advantages. 
The rectified trajectories of flow models~\cite{Liu:2023RF, liu2024instaflow, BlackForestLabs:2024Flux} allow for a much clearer posterior mean, leading to more precise future reward estimation and, in turn, more effective particle filtering.

\section{Analysis of the Interpolant Conversion and Sample Diversity}
\vspace{-0.5\baselineskip}
\label{sec:interpolant_diversity_analysis}

\begin{wrapfigure}{r}{0.45\textwidth}
\centering
\vspace{-1.2em}
\includegraphics[width=\linewidth]{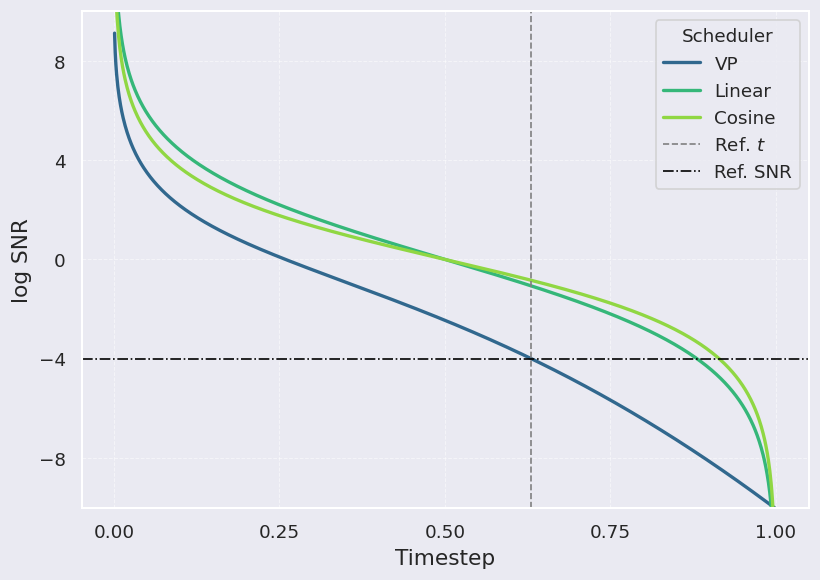}
\captionof{figure}{\textbf{Interpolant log-SNR.} Dashed lines show a reference SNR and timestep. 
}
\vspace{-2.0em}
\label{fig:interpolant_analysis}
\vspace{-0.5em}
\end{wrapfigure}

In this section, we analyze how the interpolant conversion affects the variance of the proposal distribution and explain why VP-SDE yields higher sample diversity than Linear-SDE, as illustrated in Fig.~\ref{fig:variance_test}. 
To investigate this behavior, Fig.~\ref{fig:interpolant_analysis} visualizes the log-SNR ($\log(\alpha_t^2/\sigma_t^2)$) of commonly used interpolants including linear and VP across timesteps $t \in (0, 1)$. 

Then consider initializing the Linear-SDE timesteps $\{t_s\}_{0 \le s \le 1}$ using the timestep conversion in Eq.~\ref{eq:scale_time_trans}.
This ensures that the log-SNR of the corresponding latents $\mathbf{x}_{t_s}$ matches that of the VP-SDE latents $\bar{\mathbf{x}}_s$ at each step (see the horizontal dashed line in Fig.~\ref{fig:interpolant_analysis}). 
Under this condition, the proposal distributions of the two processes are expressed as follows:
\begin{alignat}{2}
    \nonumber
    \text{Linear-SDE ($t_s$):} \quad 
    & p_\theta(\B{x}_{t_s - \Delta t_s} \mid \B{x}_{t_s}) 
    &&= \mathcal{N} \left( \B{x}_{t_s} - \B{f}_{t_s}(\B{x}_{t_s}) \Delta t_s,\ g_{t_s}^2 \Delta t_s\ \mathbf{I} \right) \\
    \nonumber
    \text{VP-SDE:} \quad 
    & \Bar{p}_\theta(\Bar{\B{x}}_{s - \Delta s} \mid \Bar{\B{x}}_s) 
    &&= \mathcal{N} \left( \Bar{\B{x}}_s - \Bar{\B{f}}_{s}(\Bar{\B{x}}_{s}) \Delta s,\ 
    g_s^2 \Delta s\ \mathbf{I} \right) = \Bar{p}_\theta(\cdot \mid c_s \B{x}_{t_s}). 
\end{alignat}
Since $\Delta t_s < \Delta s$ at early denoising steps, timestep conversion results in smaller variance $g_{t_s}^2 \Delta t_s$ than VP-SDE for a fixed diffusion coefficient. 
Hence, to match the variance of Linear-SDE ($t_s$) to that of VP-SDE, one can scale the diffusion coefficient to $g'_{t_s} = g_s/c_s \sqrt{\Delta s / \Delta t_s}$. 
Note that this scaling significantly increases the stochasticity at early denoising steps as $c_s \approx 1$ and $\sqrt{\Delta s / \Delta t_s} \gg 1$. 
While this scaling can enhance sample diversity, applying it in isolation injects excessive noise, causing samples to deviate from the predefined denoising trajectory and ultimately degrading output quality.

In fact, interpolant conversion counteracts excessive noise injection by pairing diffusion coefficient scaling with timestep conversion. 
The two mechanisms act synergistically to increase the sample diversity without harming the sample quality. 
In Sec.~\ref{sec:applications}, we validate this analysis with an ablation study that isolates the effect of each factor.

\paragraph{Comparison under Identical Timestep and Diffusion Coefficient.}
We next analyze the case where both Linear-SDE and VP-SDE operate under identical, fixed timestep schedules and diffusion coefficients. 
While this setting yields identical proposal distribution variances, Fig.~\ref{fig:interpolant_analysis} shows that the VP interpolant maintains a consistently lower log-SNR, indicating that at any given timestep, VP-SDE samples contain a larger noise component (see the vertical dashed line). 
Consequently, the VP-SDE proposal distribution effectively samples from a noisier latent at each step, resulting in higher sample diversity. 
This reflects the observation that noisier latents produce more diverse samples~\cite{meng2021sdedit}. 
While this work focuses on the interpolant perspective, a systematic exploration of timestep scheduling and diffusion coefficient scaling remains a promising direction for future research.

\vspace{-0.75\baselineskip}
\section{Rollover Budget Forcing}
\vspace{-0.5\baselineskip}
\label{sec:roll_over}
In the previous sections, we have introduced our inference-time approaches to expand the search space of proposal distribution. Here, we propose a new budget-forcing strategy to maximize the use of limited compute in inference-time scaling. 
To the best of our knowledge, previous particle sampling methods for diffusion models~\cite{Li2024:SVDD, Singh:2025CoDE} employ a fixed number of particles across all denoising steps. 
However, our analysis shows that this uniform allocation may lead to inefficiency, where the NFEs required at each denoising step to obtain a sample $\B{x}_{t-\Delta t}$ with a higher reward than the current sample $\B{x}_t$ significantly varies across different runs. We present the analysis results in Appendix~\ref{sec:appendix_adaptive_time_and_rollover}.

This motivates us to adopt a \emph{rollover} strategy that adaptively allocates NFEs across timesteps. 
Given a total NFEs budget, the NFEs quota $Q$ is allocated uniformly across timesteps. 
Then at each timestep, if a particle $\B{x}_{t-\Delta t}$ yields a higher reward than the current sample $\B{x}_t$ within the quota, we immediately proceed to the next timestep from the newly identified high-reward sample, rolling over the remaining NFEs to the next step. 
If the allocated quota is exhausted without identifying a better sample, we select the particle with the highest expected future reward from the current set, following the strategy used in SVDD~\cite{Li2024:SVDD}.
The pseudocode of \Ours{} is presented in Appendix~\ref{sec:appendix_inference_time_scaling}. 
In the next section, we demonstrate the effectiveness of~\Ours{}, along with SDE conversion and interpolant conversion.

\vspace{-0.6\baselineskip}
\section{Applications}
\vspace{-0.6\baselineskip}
\label{sec:applications}
In this section, we present the experimental results of particle sampling methods for inference-time reward alignment. 
In Appendix, we present i) implementation details of the search algorithms, ii) aesthetic image generation, iii) comparisons between diffusion and flow models, iv) scaling behavior comparison of Best-of-N (BoN) and \Ours{}, and v) additional qualitative results. 
\vspace{-0.3\baselineskip}
\subsection{Experiment Setup} 
\vspace{-0.3\baselineskip}
\paragraph{Tasks.}
In this section, we present the results for the following applications: compositional text-to-image generation and quantity-aware image generation, where the rewards are non-differentiable. 
For the differentiable reward case, we consider aesthetic image generation (Appendix~\ref{sec:appendix_diff_reward}). 
In compositional text-to-image generation, we use all $121$ text prompts from GenAI-Bench~\cite{Jiang:2024GenAI} that contain three or more advanced compositional elements. 
For quantity-aware image generation, we use $100$ randomly sampled prompts from T2I-CompBench++~\cite{Huang:2025T2ICompBench++} numeracy category. 

For all applications, we use FLUX~\cite{BlackForestLabs:2024Flux} as the pretrained flow model. We fix the total number of function evaluations (NFEs) to $500$ and set the number of denoising steps to $10$, which allocates $50$ NFEs per denoising step. As a reference, we also include the results of the base pretrained models without inference-time scaling. Additionally, we present a comparison between flow models and diffusion models in Appendix~\ref{sec:appendix_diff_flow}. 
\vspace{-0.5\baselineskip}
\paragraph{Baselines.} We evaluate inference-time search algorithms discussed in Sec.~\ref{sec:related}, including Best-of-N (BoN), Search over Paths (SoP)~\cite{Ma2025:SoP}, SMC~\cite{Kim:2025DAS}, CoDe~\cite{Singh:2025CoDE}, and SVDD~\cite{Li2024:SVDD}. 
We categorize BoN and SoP as Linear-ODE-based methods, as their generative processes follow the deterministic process in Eq.~\ref{eq:pf-ode}. 
For SMC, we adopt DAS~\cite{Kim:2025DAS}; however, when the reward is non-differentiable, we use the reverse transition kernel of the pretrained model as the proposal distribution.

\begin{figure}[t]
    \definecolor{YellowGreen}{RGB}{49,126,53}
    \definecolor{Rhodamine}{RGB}{243,156,146}
    \definecolor{Dandelion}{RGB}{242,193,106}
    \vspace{-0.5\baselineskip}
    \centering
    \includegraphics[width=1.0\linewidth]{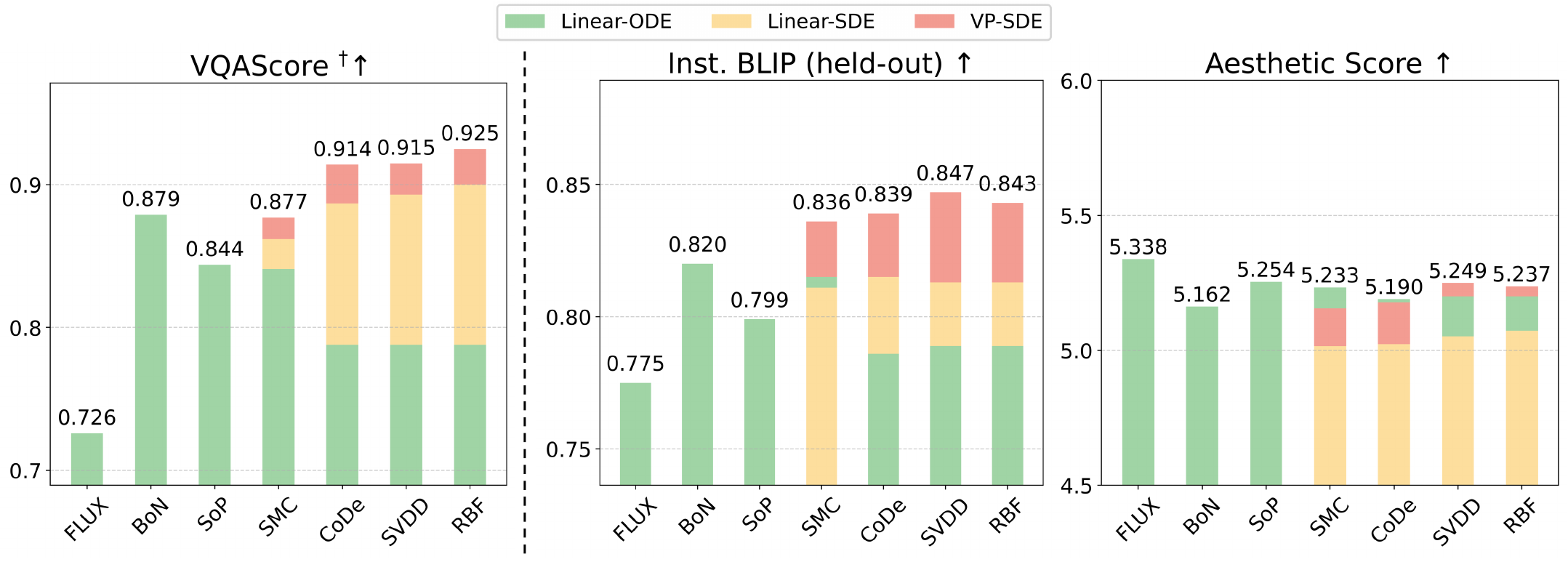}
    \caption{\textbf{Quantitative results of compositional text-to-image generation.} \textsuperscript{\textdagger} denotes the given reward used in inference-time scaling (left). 
    Notably, performance consistently improves from {\color{YellowGreen}Linear-ODE} to {\color{Dandelion}Linear-SDE} and {\color{Rhodamine}VP-SDE} for both given and held-out rewards (left, middle), without significant quality degradation, as evidenced by the comparable aesthetic score~\cite{Schuhmann:aesthetics} (right).} 
    \label{fig:composition_main}
    \vspace{-0.75\baselineskip}
\end{figure}

\begin{figure}[ht!]
\centering
{\scriptsize
\setlength{\tabcolsep}{0pt}
\renewcommand{\arraystretch}{0.0}
\newcolumntype{Y}{>{\centering\arraybackslash}m{0.025\textwidth}}
\newcolumntype{Z}{>{\centering\arraybackslash}m{0.155\textwidth}}
\newcolumntype{X}{>{\centering\arraybackslash}m{0.0002\textwidth}}

\begin{tabularx}{\textwidth}{Y Z Z Z X Y Z Z Z}
\toprule
& Linear-ODE & Linear-SDE & VP-SDE & & & Linear-ODE & Linear-SDE & VP-SDE \\
\midrule
& \multicolumn{3}{c}{\textit{``Three small, non-blue boxes on a large blue box.''}} & & 
& \multicolumn{3}{c}{\makecell{\textit{``Five origami cranes hang from the ceiling,}\\ \textit{only one of which is red, and the others are all white.''}}} \\
\rotatebox{90}{\makecell{\scriptsize{SMC}~\cite{Kim:2025DAS}}} &
\includegraphics[width=0.15\textwidth]{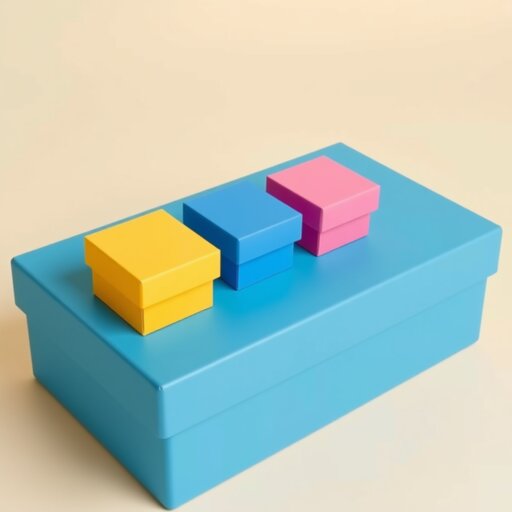} &
\includegraphics[width=0.15\textwidth]{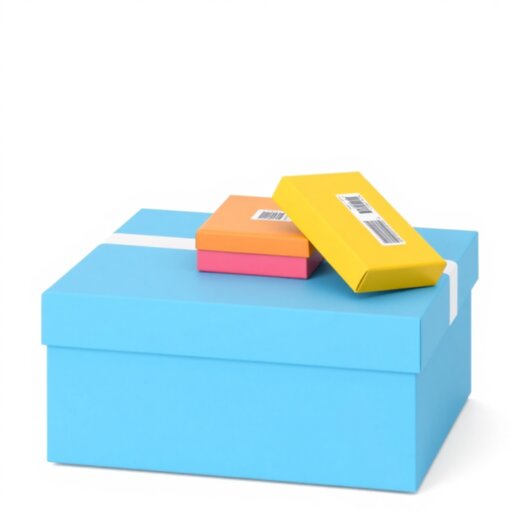} &
\includegraphics[width=0.15\textwidth]{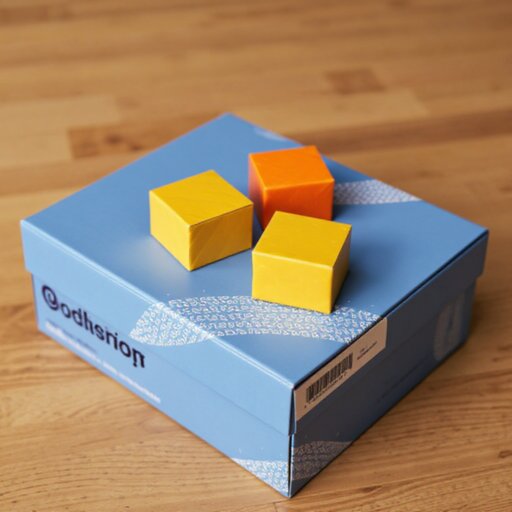} & & 
\rotatebox{90}{\makecell{\scriptsize{CoDe}~\cite{Singh:2025CoDE}}} &
\includegraphics[width=0.15\textwidth]{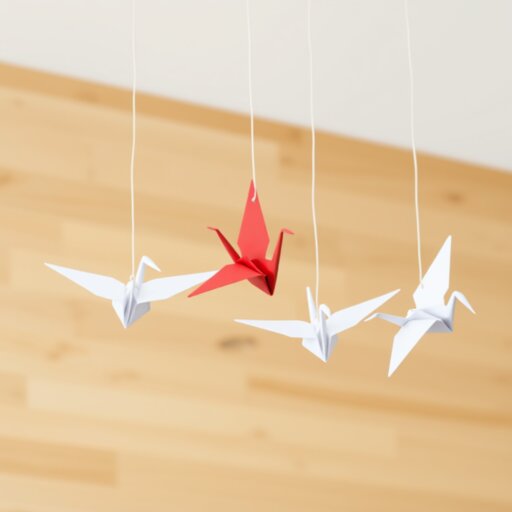} &
\includegraphics[width=0.15\textwidth]{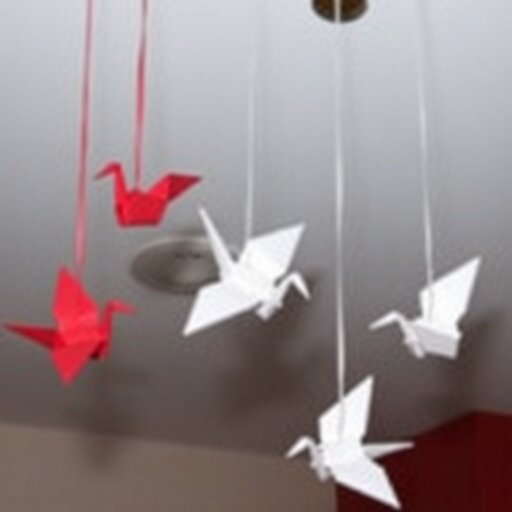} &
\includegraphics[width=0.15\textwidth]{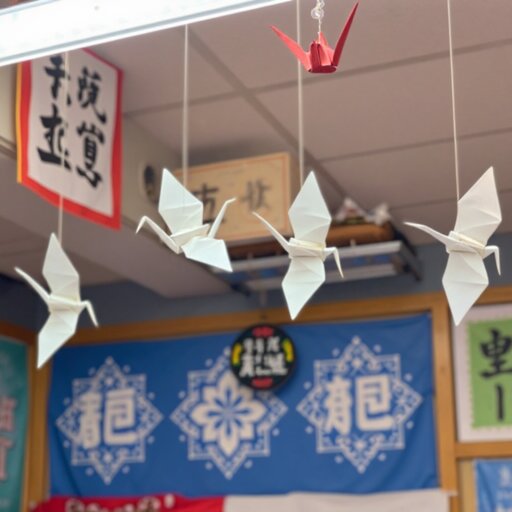} \\
\midrule
& \multicolumn{3}{c}{\makecell{\textit{``A mouse pad has two pencils on it,}\\ \textit{the shorter pencil is green and the longer one is not.''}}} & & 
& \multicolumn{3}{c}{\makecell{\textit{``Five ants are carrying biscuits, and an ant that}\\\textit{is not carrying biscuits is standing on a green leaf directing them.''}}} \\
\rotatebox{90}{\makecell{\scriptsize{SVDD}~\cite{Li2024:SVDD}}} &
\includegraphics[width=0.15\textwidth]{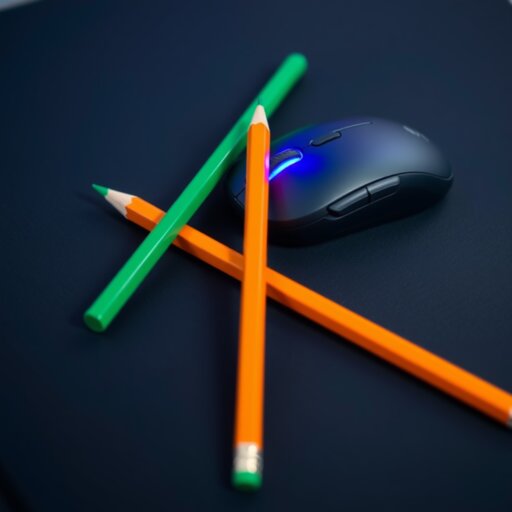} &
\includegraphics[width=0.15\textwidth]{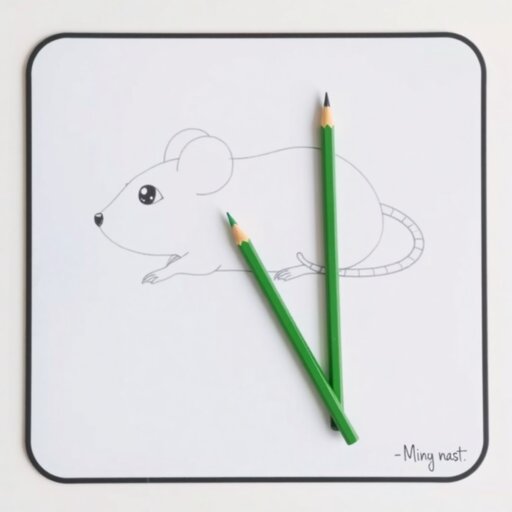} &
\includegraphics[width=0.15\textwidth]{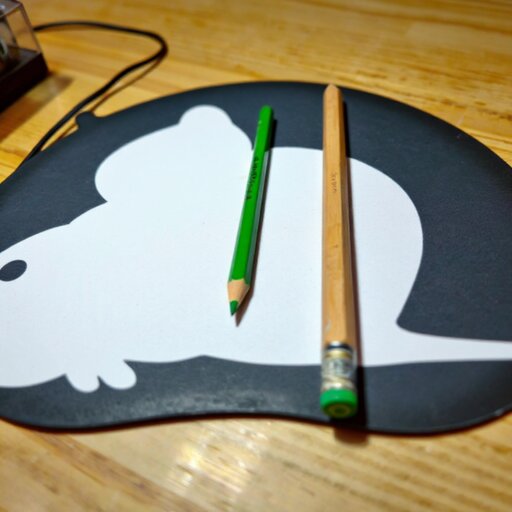} & & 
\rotatebox{90}{\makecell{\scriptsize{\Oursbf{}}}} &
\includegraphics[width=0.15\textwidth]{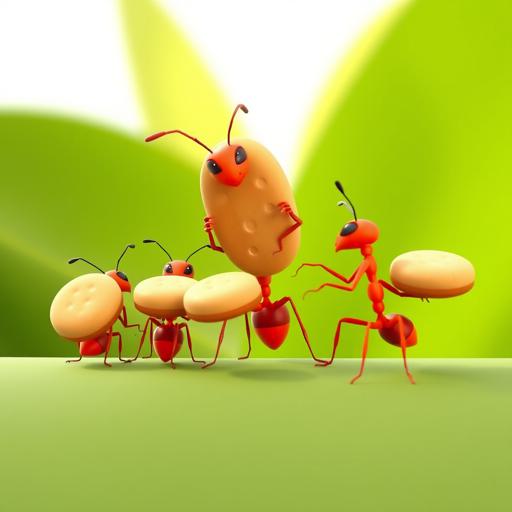} &
\includegraphics[width=0.15\textwidth]{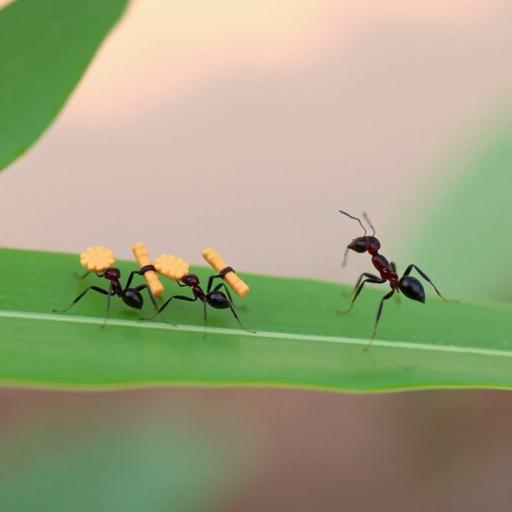} &
\includegraphics[width=0.15\textwidth]{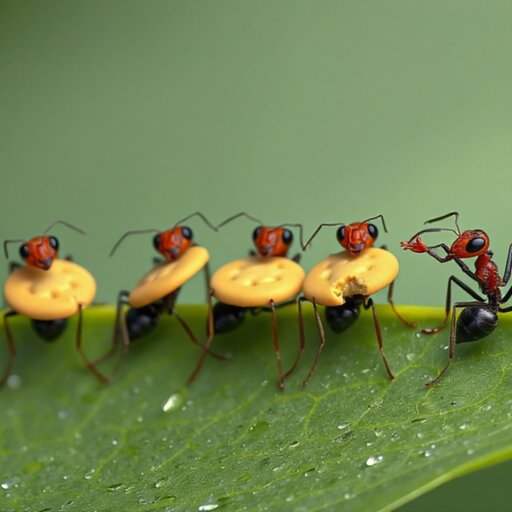} \\
\bottomrule
\end{tabularx}
}
\caption{\textbf{Qualitative results of compositional text-to-image generation.} 
    We use VQAScore~\cite{Lin:2024CLIPFlanT5}, which measures text-image alignment, as the given reward for inference-time scaling. 
    SDE and interpolant conversion enable more effective exploration during inference, enhancing the performance of all particle sampling methods~\cite{Kim:2025DAS, Singh:2025CoDE, Li2024:SVDD}, including~\Ours{}.
}
\label{fig:composition_methods_vp}
\vspace{-1.25\baselineskip}
\end{figure}

\vspace{-0.3\baselineskip}
\subsection{Compositional Text-to-Image Generation}
\vspace{-0.3\baselineskip}
\label{sec:compositional_generation}
\paragraph{Evaluation Metrics.} 
In this work, we refer to the reward used for inference-time scaling as the given reward. 
Here, the given reward is VQAScore, measured with CLIP-FlanT5~\cite{Lin:2024CLIPFlanT5}, which evaluates text-image alignment. 
For the held-out reward, which is not used during inference, we evaluate the score using a different model, InstructBLIP~\cite{Dai:2023InstructBLIP}. 
Additionally, we evaluate aesthetic score~\cite{Schuhmann:aesthetics} to assess the quality of the generated images. 

\vspace{-0.5\baselineskip}
\paragraph{Inference-Time SDE and Interpolant Conversion.} 
The quantitative and qualitative results of compositional text-to-image generation are presented in Fig.~\ref{fig:composition_main} and Fig.~\ref{fig:composition_methods_vp}, respectively. 
As discussed in Sec.~\ref{sec:ode_sde_conversion}, the deterministic sampling process in flow models limits the effectiveness of particle sampling, whereas introducing stochasticity significantly expands the search space and improves performance—highlighting a key contribution of our work: enabling effective particle sampling in flow models.
The results in Fig.~\ref{fig:composition_main} support this finding, showing that Linear-SDE (yellow) consistently improves the given reward (left in Fig.~\ref{fig:composition_main}) over the Linear-ODE (green) across all particle sampling methods, even surpassing BoN and SoP~\cite{Ma2025:SoP}, which were previously the only available inference-time scaling approaches for ODE-based flow models.
Additionally, through inference-time interpolant conversion, VP-SDE (red) further improves performance across all particle sampling methods on both given and held-out rewards (left, middle in Fig.~\ref{fig:composition_main}) by expanding the search space, demonstrating the effectiveness of our proposed distribution.
Notably, particle sampling methods with Linear-SDE and VP-SDE generate high-reward samples without significantly compromising image quality, as evidenced by aesthetic scores that remain comparable to the base FLUX model~\cite{BlackForestLabs:2024Flux} (right in Fig.~\ref{fig:composition_main}).
Qualitatively, SDE conversion and interpolant conversion shown in Fig.~\ref{fig:composition_methods_vp} bring consistent performance improvements (see Appendix~\ref{sec:appendix_quali_ode_sde_vp} for additional results). 

\vspace{-0.5\baselineskip}
\paragraph{Rollover Budget Forcing.} 
As discussed in Sec.~\ref{sec:roll_over}, instead of fixing the number of particles throughout the denoising process, we explore adaptive budget allocation through~\Ours{}. We demonstrate that budget forcing provides additional performance improvements, outperforming the previous particle sampling methods in the given reward (left in Fig.~\ref{fig:composition_main}). 
We present qualitative comparisons of inference-time scaling methods in Appendix~\ref{sec:appendix_quali_search_alg}.

\begin{wraptable}{r}{0.48\textwidth}
    \centering
    \vspace{-1.25em}
    \scriptsize
    \caption{
        \textbf{Ablation Study of Interpolant Conversion.}
        \textsuperscript{\textdagger} denotes the given reward.
    }
    \label{tab:abl_sde_interpolant}
    \vspace{-0.5em}
    \renewcommand{\arraystretch}{1.0}
    \setlength{\tabcolsep}{0pt}
    \newcolumntype{Y}{>{\centering\arraybackslash}m{0.24\linewidth}}
    \newcolumntype{X}{>{\centering\arraybackslash}m{0.12\linewidth}}
    \begin{tabularx}{\linewidth}{Y | Y Y Y}
        \toprule
        Method &
        \makecell{LPIPS-MPD $\uparrow$} &
        VQAScore\textsuperscript{\textdagger} $\uparrow$ &
        \makecell{Inst. BLIP $\uparrow$} \\
        \midrule
        Linear-ODE & -- & 0.788 & 0.789 \\
        Linear-SDE & 0.158 & 0.900 & 0.813 \\
        + Adapt. Time. & 0.270 & 0.908 & 0.813 \\
        + Adapt. Diff. & 0.429 & 0.702 & 0.571 \\
        VP-SDE & \textbf{0.509} & \textbf{0.925} & \textbf{0.843} \\
        \bottomrule
    \end{tabularx}
    \vspace{-1.25em}
\end{wraptable}

\vspace{-0.5\baselineskip}
\paragraph{Ablation study of interpolant conversion.}
Building on the analysis in Sec.~\ref{sec:interpolant_diversity_analysis}, we examine how interpolant conversion contributes to sample diversity and reward alignment through its two underlying mechanisms, timestep conversion and diffusion coefficient scaling. 
Tab.~\ref{tab:abl_sde_interpolant} extends the results of Fig.~\ref{fig:composition_main} by isolating the effect of each component to sample diversity, measured by LPIPS-MPD~\cite{Kim:2025DAS}, and reward alignment.

We observe that timestep conversion (row 3) yields only modest diversity gains: the benefit of sampling at lower log-SNR (Fig.~\ref{fig:interpolant_analysis}) is offset by smaller discretization steps that reduce proposal variance, limiting improvements in reward alignment. 
On the other hand, applying diffusion coefficient scaling without timestep conversion (row 4) increases sample diversity but simultaneously leads to a significant drop in reward alignment indicating excessive noise injection. 
Lastly, the VP-SDE interpolant conversion (row 5) synergistically combines both components, achieving high sample diversity without sacrificing quality and consequently yielding the highest reward.

\begin{figure}[t]
    \definecolor{YellowGreen}{RGB}{49,126,53}
    \definecolor{Rhodamine}{RGB}{243,156,146}
    \definecolor{Dandelion}{RGB}{242,193,106}
    \vspace{-0.5\baselineskip}
    \centering
    \includegraphics[width=1.0\linewidth]{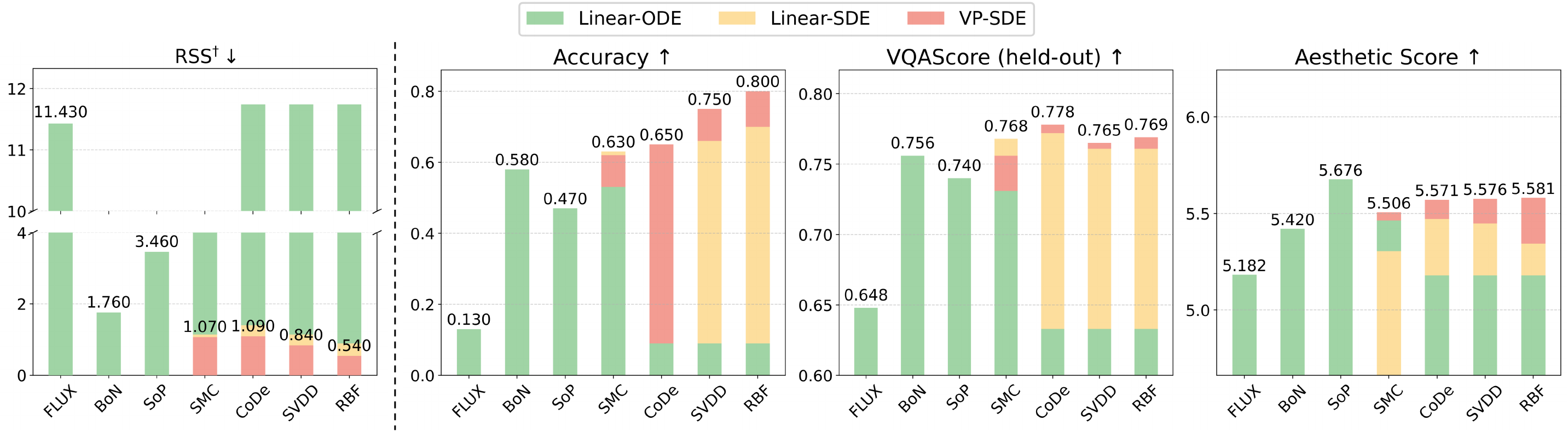}
    \caption{\textbf{Quantitative results of quantity-aware image generation.} \textsuperscript{\textdagger} denotes the given reward, RSS~\cite{Liu:2024GroundingDINO}, with the y-axis truncated for better visualization (left). 
    We observe consistent performance improvements by converting {\color{YellowGreen}Linear-ODE} to {\color{Dandelion}{Linear-SDE}}, and {\color{Rhodamine}VP-SDE} for most cases.} 
    \label{fig:counting_main}
    \vspace{-0.75\baselineskip}
\end{figure}

\begin{figure}[ht!]
\centering
{\scriptsize
\setlength{\tabcolsep}{0pt}
\renewcommand{\arraystretch}{0.0}
\newcolumntype{Y}{>{\centering\arraybackslash}m{0.025\textwidth}}
\newcolumntype{Z}{>{\centering\arraybackslash}m{0.155\textwidth}}
\newcolumntype{X}{>{\centering\arraybackslash}m{0.0002\textwidth}}

\begin{tabularx}{\textwidth}{Y Z Z Z X Y Z Z Z}
\toprule
& Linear-ODE & Linear-SDE & VP-SDE & & & Linear-ODE & Linear-SDE & VP-SDE \\
\midrule
& \multicolumn{3}{c}{\makecell{\textit{``Four balloons, one cup,}\\\textit{four desks, two dogs and four microwaves.''}}} & & 
& \multicolumn{3}{c}{\textit{\makecell{``Four candles, two balloons, one dog,\\two tomatoes and three helicopters.''}}} \\
\rotatebox{90}{\makecell{\scriptsize{SMC}~\cite{Kim:2025DAS}}} &
\includegraphics[width=0.15\textwidth]{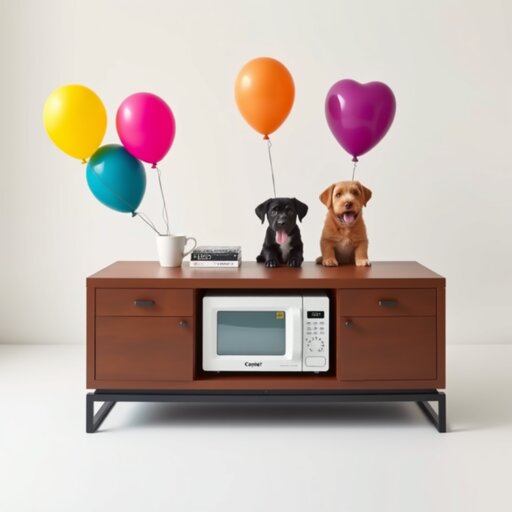} &
\includegraphics[width=0.15\textwidth]{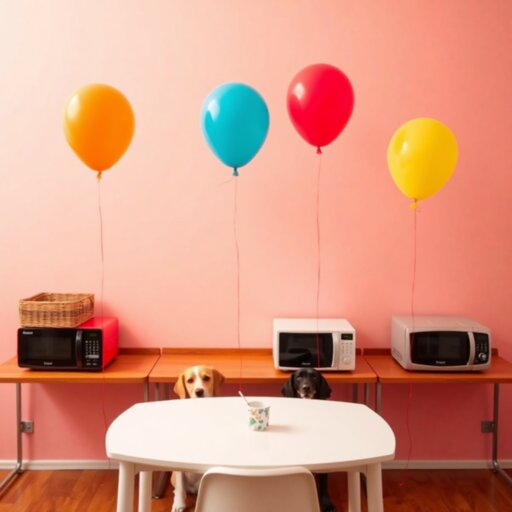} &
\includegraphics[width=0.15\textwidth]{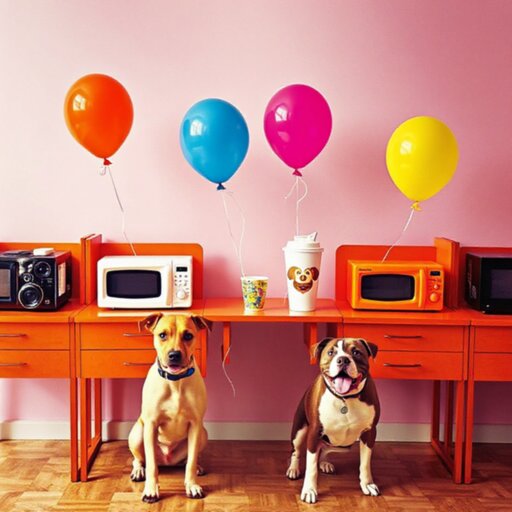} & & 
\rotatebox{90}{\makecell{\scriptsize{CoDe}~\cite{Singh:2025CoDE}}} &
\includegraphics[width=0.15\textwidth]{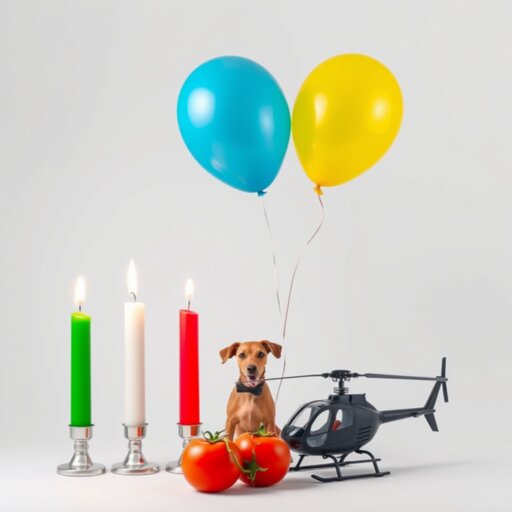} &
\includegraphics[width=0.15\textwidth]{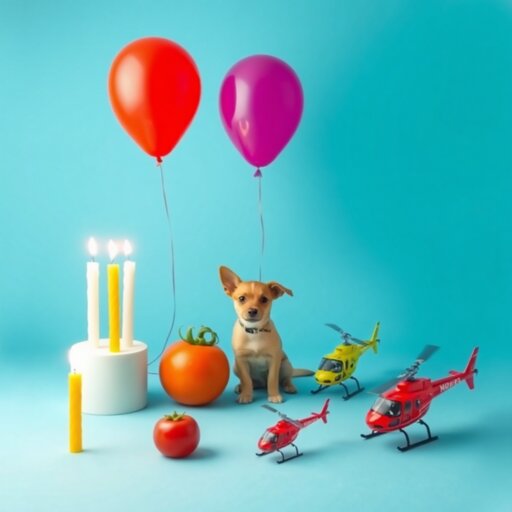} &
\includegraphics[width=0.15\textwidth]{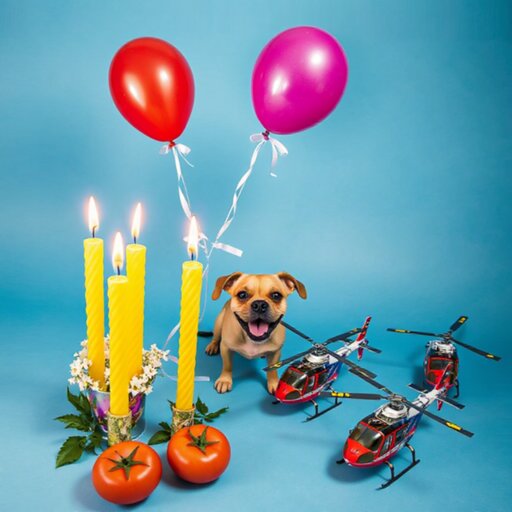} \\
\midrule
& \multicolumn{3}{c}{\textit{``Eight chairs.''}} & & 
& \multicolumn{3}{c}{\textit{\makecell{``One egg, three camels, four cars and four pillows.''}}} \\
\rotatebox{90}{\makecell{\scriptsize{SVDD}~\cite{Li2024:SVDD}}} &
\includegraphics[width=0.15\textwidth]{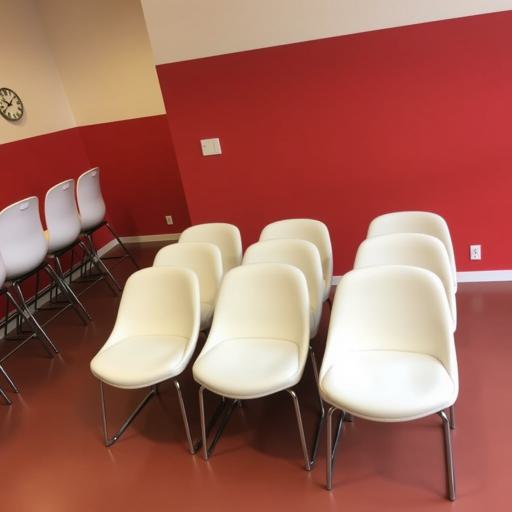} &
\includegraphics[width=0.15\textwidth]{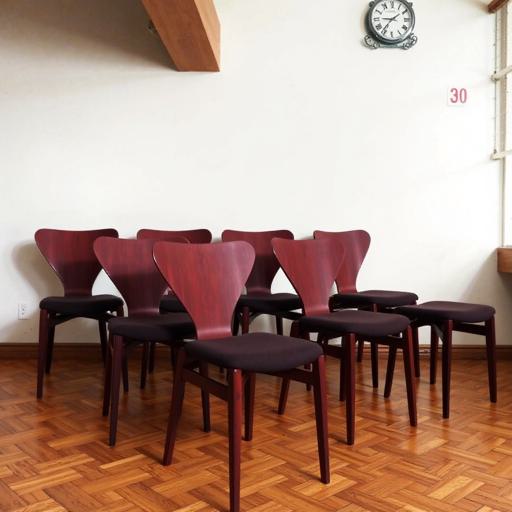} &
\includegraphics[width=0.15\textwidth]{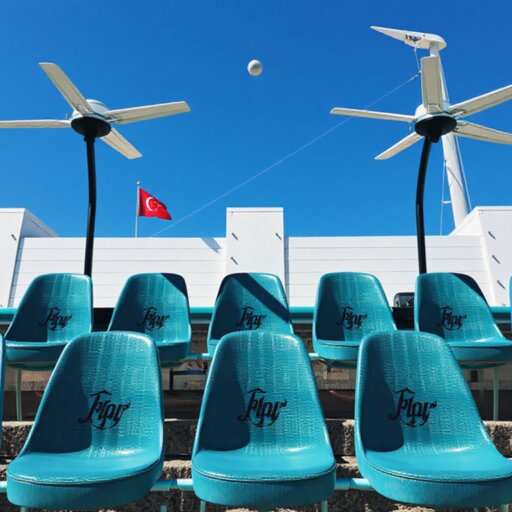} & & 
\rotatebox{90}{\makecell{\scriptsize{\Oursbf{}}}} &
\includegraphics[width=0.15\textwidth]{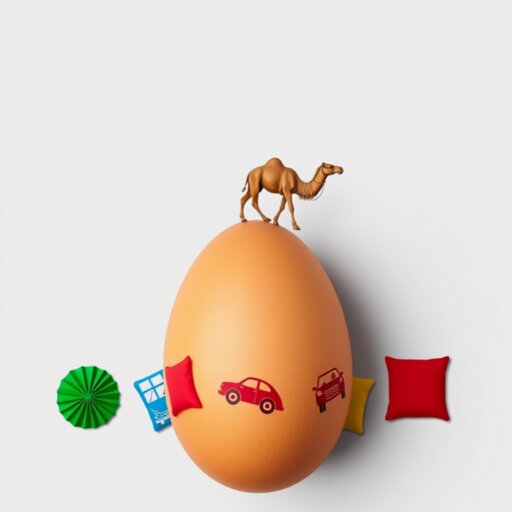} &
\includegraphics[width=0.15\textwidth]{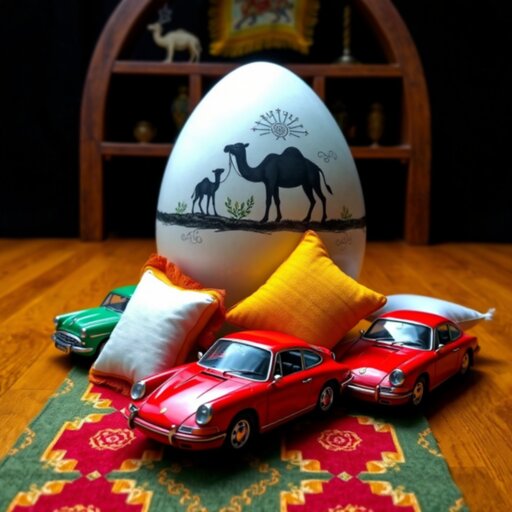} &
\includegraphics[width=0.15\textwidth]{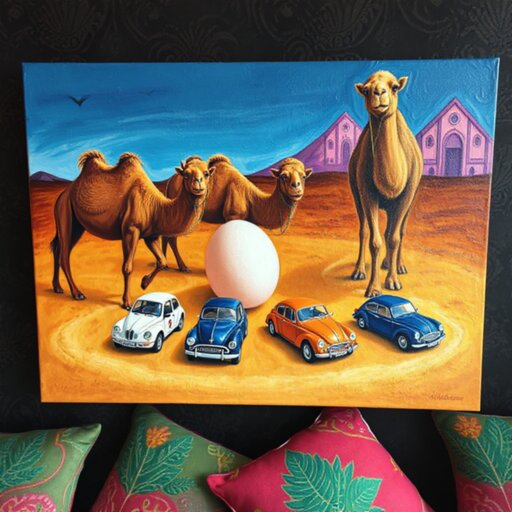} \\
\bottomrule
\end{tabularx}
}
\caption{
    \textbf{Qualitative results of quantity-aware image generation.}
    At inference-time, we guide generation using the negation of RSS~\cite{Liu:2024GroundingDINO} (Residual Sum of Squares) as the given reward, which measures the discrepancy between detected and target object counts. 
    SDE and interpolant conversion expands the search space to identify high reward samples. 
}
\label{fig:counting_methods_vp}
\vspace{-1.0\baselineskip}
\end{figure}

\vspace{-0.4\baselineskip}
\subsection{Quantity-Aware Image Generation}
\vspace{-0.3\baselineskip}
\label{sec:counting_generation}
\paragraph{Evaluation Metrics.} 
Here, the given reward is the negation of the Residual Sum of Squares (RSS) between the target counts and the detected object counts, computed using GroundingDINO~\cite{Liu:2024GroundingDINO} and SAM~\cite{Kirillov2023:SAM} (details in Appendix~\ref{sec:appendix_impl_details}). Additionally, we report object count accuracy, which evaluates whether all object quantities are correctly shown in the image. For the held-out reward, we report VQAScore measured with CLIP-FlanT5~\cite{Lin:2024CLIPFlanT5}. 
As in the previous application, we evaluate the quality of the generated images using the aesthetic score~\cite{Schuhmann:aesthetics}. 

\vspace{-0.5\baselineskip}
\paragraph{Results.} The quantitative and qualitative results are presented in Fig.~\ref{fig:counting_main} and Fig.~\ref{fig:counting_methods_vp}, respectively. 
The trend in Fig.~\ref{fig:counting_main} align with those in Sec.~\ref{sec:compositional_generation}, demonstrating that SDE conversion and interpolant conversion synergistically enhance the identification of high-reward samples. 
Notably, particle sampling methods with Linear-SDE already outperform Linear-ODE-based methods (BoN and SoP~\cite{Ma2025:SoP}), while interpolant conversion further improves accuracy, achieving a $4\sim6 \times$ improvement over the base model~\cite{BlackForestLabs:2024Flux}. 
Our \Ours{} achieves the highest accuracy, outperforming all other particle sampling methods. 
Qualitatively, Fig.~\ref{fig:counting_methods_vp} shows that SDE and interpolant conversion effectively identify high-reward samples that accurately match the specified object categories and quantities. Additional qualitative comparisons of the inference-time scaling methods are provided in Appendix~\ref{sec:appendix_quali_search_alg}. 


\vspace{-0.7\baselineskip}
\section{Conclusion and Limitation}
\vspace{-0.9\baselineskip}
\label{sec:conclusion}
We introduced a novel inference-time scaling method for flow models with three key contributions: (1) ODE-to-SDE conversion for particle sampling in flow models, (2) Linear-to-VP interpolant conversion for enhanced diversity and search efficiency, and (3) Rollover Budget Forcing (RBF) for adaptive compute allocation. We demonstrated the effectiveness of VP-SDE-based generation in applying off-the-shelf particle sampling to flow models and showed that our \Ours{} combined with VP-SDE generation outperforms previous methods. 
However, our method introduces additional inference-time overhead, which could become a bottleneck when the base model prediction is computationally intensive. Also, since the pretrained model may have been trained on uncurated datasets, our approach may produce undesirable outputs upon malicious attempts. 


\section*{Acknowledgments}
This work was supported by the NRF of Korea (RS-2023-00209723); IITP grants (RS-2022-II220594, RS-2023-00227592, RS-2024-00399817, RS-2025-25441313, RS-2025-25443318, RS-2025-02653113); and the Technology Innovation Program (RS-2025-02317326), all funded by the Korean government (MSIT and MOTIE), as well as by the DRB-KAIST SketchTheFuture Research Center.

{\small
\bibliographystyle{plain}
\bibliography{main}
}

\clearpage
\newpage
\section*{NeurIPS Paper Checklist}

The checklist is designed to encourage best practices for responsible machine learning research, addressing issues of reproducibility, transparency, research ethics, and societal impact. Do not remove the checklist: {\bf The papers not including the checklist will be desk rejected.} The checklist should follow the references and follow the (optional) supplemental material.  The checklist does NOT count towards the page
limit. 

Please read the checklist guidelines carefully for information on how to answer these questions. For each question in the checklist:
\begin{itemize}
    \item You should answer \answerYes{}, \answerNo{}, or \answerNA{}.
    \item \answerNA{} means either that the question is Not Applicable for that particular paper or the relevant information is Not Available.
    \item Please provide a short (1–2 sentence) justification right after your answer (even for NA). 
\end{itemize}

{\bf The checklist answers are an integral part of your paper submission.} They are visible to the reviewers, area chairs, senior area chairs, and ethics reviewers. You will be asked to also include it (after eventual revisions) with the final version of your paper, and its final version will be published with the paper.

The reviewers of your paper will be asked to use the checklist as one of the factors in their evaluation. While "\answerYes{}" is generally preferable to "\answerNo{}", it is perfectly acceptable to answer "\answerNo{}" provided a proper justification is given (e.g., "error bars are not reported because it would be too computationally expensive" or "we were unable to find the license for the dataset we used"). In general, answering "\answerNo{}" or "\answerNA{}" is not grounds for rejection. While the questions are phrased in a binary way, we acknowledge that the true answer is often more nuanced, so please just use your best judgment and write a justification to elaborate. All supporting evidence can appear either in the main paper or the supplemental material, provided in appendix. If you answer \answerYes{} to a question, in the justification please point to the section(s) where related material for the question can be found.

IMPORTANT, please:
\begin{itemize}
    \item {\bf Delete this instruction block, but keep the section heading ``NeurIPS Paper Checklist"},
    \item  {\bf Keep the checklist subsection headings, questions/answers and guidelines below.}
    \item {\bf Do not modify the questions and only use the provided macros for your answers}.
\end{itemize}


\begin{enumerate}

\item {\bf Claims}
    \item[] Question: Do the main claims made in the abstract and introduction accurately reflect the paper's contributions and scope?
    \item[] Answer: \answerYes{} 
    \item[] Justification: Our analysis and experiments support the claims in the abstract and introduction. 
    \item[] Guidelines:
    \begin{itemize}
        \item The answer NA means that the abstract and introduction do not include the claims made in the paper.
        \item The abstract and/or introduction should clearly state the claims made, including the contributions made in the paper and important assumptions and limitations. A No or NA answer to this question will not be perceived well by the reviewers. 
        \item The claims made should match theoretical and experimental results, and reflect how much the results can be expected to generalize to other settings. 
        \item It is fine to include aspirational goals as motivation as long as it is clear that these goals are not attained by the paper. 
    \end{itemize}

\item {\bf Limitations}
    \item[] Question: Does the paper discuss the limitations of the work performed by the authors?
    \item[] Answer: \answerYes{} 
    \item[] Justification: We present limitations in the last section. 
    \item[] Guidelines:
    \begin{itemize}
        \item The answer NA means that the paper has no limitation while the answer No means that the paper has limitations, but those are not discussed in the paper. 
        \item The authors are encouraged to create a separate "Limitations" section in their paper.
        \item The paper should point out any strong assumptions and how robust the results are to violations of these assumptions (e.g., independence assumptions, noiseless settings, model well-specification, asymptotic approximations only holding locally). The authors should reflect on how these assumptions might be violated in practice and what the implications would be.
        \item The authors should reflect on the scope of the claims made, e.g., if the approach was only tested on a few datasets or with a few runs. In general, empirical results often depend on implicit assumptions, which should be articulated.
        \item The authors should reflect on the factors that influence the performance of the approach. For example, a facial recognition algorithm may perform poorly when image resolution is low or images are taken in low lighting. Or a speech-to-text system might not be used reliably to provide closed captions for online lectures because it fails to handle technical jargon.
        \item The authors should discuss the computational efficiency of the proposed algorithms and how they scale with dataset size.
        \item If applicable, the authors should discuss possible limitations of their approach to address problems of privacy and fairness.
        \item While the authors might fear that complete honesty about limitations might be used by reviewers as grounds for rejection, a worse outcome might be that reviewers discover limitations that aren't acknowledged in the paper. The authors should use their best judgment and recognize that individual actions in favor of transparency play an important role in developing norms that preserve the integrity of the community. Reviewers will be specifically instructed to not penalize honesty concerning limitations.
    \end{itemize}

\item {\bf Theory assumptions and proofs}
    \item[] Question: For each theoretical result, does the paper provide the full set of assumptions and a complete (and correct) proof?
    \item[] Answer: \answerYes{} 
    \item[] Justification: We present proofs in Appendix. 
    \item[] Guidelines:
    \begin{itemize}
        \item The answer NA means that the paper does not include theoretical results. 
        \item All the theorems, formulas, and proofs in the paper should be numbered and cross-referenced.
        \item All assumptions should be clearly stated or referenced in the statement of any theorems.
        \item The proofs can either appear in the main paper or the supplemental material, but if they appear in the supplemental material, the authors are encouraged to provide a short proof sketch to provide intuition. 
        \item Inversely, any informal proof provided in the core of the paper should be complemented by formal proofs provided in appendix or supplemental material.
        \item Theorems and Lemmas that the proof relies upon should be properly referenced. 
    \end{itemize}

    \item {\bf Experimental result reproducibility}
    \item[] Question: Does the paper fully disclose all the information needed to reproduce the main experimental results of the paper to the extent that it affects the main claims and/or conclusions of the paper (regardless of whether the code and data are provided or not)?
    \item[] Answer: \answerYes{} 
    \item[] Justification: We provide implementation details and experiment setups. 
    \item[] Guidelines:
    \begin{itemize}
        \item The answer NA means that the paper does not include experiments.
        \item If the paper includes experiments, a No answer to this question will not be perceived well by the reviewers: Making the paper reproducible is important, regardless of whether the code and data are provided or not.
        \item If the contribution is a dataset and/or model, the authors should describe the steps taken to make their results reproducible or verifiable. 
        \item Depending on the contribution, reproducibility can be accomplished in various ways. For example, if the contribution is a novel architecture, describing the architecture fully might suffice, or if the contribution is a specific model and empirical evaluation, it may be necessary to either make it possible for others to replicate the model with the same dataset, or provide access to the model. In general. releasing code and data is often one good way to accomplish this, but reproducibility can also be provided via detailed instructions for how to replicate the results, access to a hosted model (e.g., in the case of a large language model), releasing of a model checkpoint, or other means that are appropriate to the research performed.
        \item While NeurIPS does not require releasing code, the conference does require all submissions to provide some reasonable avenue for reproducibility, which may depend on the nature of the contribution. For example
        \begin{enumerate}
            \item If the contribution is primarily a new algorithm, the paper should make it clear how to reproduce that algorithm.
            \item If the contribution is primarily a new model architecture, the paper should describe the architecture clearly and fully.
            \item If the contribution is a new model (e.g., a large language model), then there should either be a way to access this model for reproducing the results or a way to reproduce the model (e.g., with an open-source dataset or instructions for how to construct the dataset).
            \item We recognize that reproducibility may be tricky in some cases, in which case authors are welcome to describe the particular way they provide for reproducibility. In the case of closed-source models, it may be that access to the model is limited in some way (e.g., to registered users), but it should be possible for other researchers to have some path to reproducing or verifying the results.
        \end{enumerate}
    \end{itemize}

\item {\bf Open access to data and code}
    \item[] Question: Does the paper provide open access to the data and code, with sufficient instructions to faithfully reproduce the main experimental results, as described in supplemental material?
    \item[] Answer: \answerYes{} 
    \item[] Justification: Code is publicly released.  
    \item[] Guidelines:
    \begin{itemize}
        \item The answer NA means that paper does not include experiments requiring code.
        \item Please see the NeurIPS code and data submission guidelines (\url{https://nips.cc/public/guides/CodeSubmissionPolicy}) for more details.
        \item While we encourage the release of code and data, we understand that this might not be possible, so “No” is an acceptable answer. Papers cannot be rejected simply for not including code, unless this is central to the contribution (e.g., for a new open-source benchmark).
        \item The instructions should contain the exact command and environment needed to run to reproduce the results. See the NeurIPS code and data submission guidelines (\url{https://nips.cc/public/guides/CodeSubmissionPolicy}) for more details.
        \item The authors should provide instructions on data access and preparation, including how to access the raw data, preprocessed data, intermediate data, and generated data, etc.
        \item The authors should provide scripts to reproduce all experimental results for the new proposed method and baselines. If only a subset of experiments are reproducible, they should state which ones are omitted from the script and why.
        \item At submission time, to preserve anonymity, the authors should release anonymized versions (if applicable).
        \item Providing as much information as possible in supplemental material (appended to the paper) is recommended, but including URLs to data and code is permitted.
    \end{itemize}

\item {\bf Experimental setting/details}
    \item[] Question: Does the paper specify all the training and test details (e.g., data splits, hyperparameters, how they were chosen, type of optimizer, etc.) necessary to understand the results?
    \item[] Answer: \answerYes{} 
    \item[] Justification: We specify hyperparameters in the paper. 
    \item[] Guidelines:
    \begin{itemize}
        \item The answer NA means that the paper does not include experiments.
        \item The experimental setting should be presented in the core of the paper to a level of detail that is necessary to appreciate the results and make sense of them.
        \item The full details can be provided either with the code, in appendix, or as supplemental material.
    \end{itemize}

\item {\bf Experiment statistical significance}
    \item[] Question: Does the paper report error bars suitably and correctly defined or other appropriate information about the statistical significance of the experiments?
    \item[] Answer: \answerNo{} 
    \item[] Justification: Due to computational constraints, we were unable to include error bars. 
    \item[] Guidelines:
    \begin{itemize}
        \item The answer NA means that the paper does not include experiments.
        \item The authors should answer "Yes" if the results are accompanied by error bars, confidence intervals, or statistical significance tests, at least for the experiments that support the main claims of the paper.
        \item The factors of variability that the error bars are capturing should be clearly stated (for example, train/test split, initialization, random drawing of some parameter, or overall run with given experimental conditions).
        \item The method for calculating the error bars should be explained (closed form formula, call to a library function, bootstrap, etc.)
        \item The assumptions made should be given (e.g., Normally distributed errors).
        \item It should be clear whether the error bar is the standard deviation or the standard error of the mean.
        \item It is OK to report 1-sigma error bars, but one should state it. The authors should preferably report a 2-sigma error bar than state that they have a 96\% CI, if the hypothesis of Normality of errors is not verified.
        \item For asymmetric distributions, the authors should be careful not to show in tables or figures symmetric error bars that would yield results that are out of range (e.g. negative error rates).
        \item If error bars are reported in tables or plots, The authors should explain in the text how they were calculated and reference the corresponding figures or tables in the text.
    \end{itemize}

\item {\bf Experiments compute resources}
    \item[] Question: For each experiment, does the paper provide sufficient information on the computer resources (type of compute workers, memory, time of execution) needed to reproduce the experiments?
    \item[] Answer: \answerYes{} 
    \item[] Justification: We specify memory usage and inference time in the paper.
    \item[] Guidelines:
    \begin{itemize}
        \item The answer NA means that the paper does not include experiments.
        \item The paper should indicate the type of compute workers CPU or GPU, internal cluster, or cloud provider, including relevant memory and storage.
        \item The paper should provide the amount of compute required for each of the individual experimental runs as well as estimate the total compute. 
        \item The paper should disclose whether the full research project required more compute than the experiments reported in the paper (e.g., preliminary or failed experiments that didn't make it into the paper). 
    \end{itemize}
    
\item {\bf Code of ethics}
    \item[] Question: Does the research conducted in the paper conform, in every respect, with the NeurIPS Code of Ethics \url{https://neurips.cc/public/EthicsGuidelines}?
    \item[] Answer: \answerYes{} 
    \item[] Justification: This work conforms with the NeurIPS Code of Ethics. 
    \item[] Guidelines:
    \begin{itemize}
        \item The answer NA means that the authors have not reviewed the NeurIPS Code of Ethics.
        \item If the authors answer No, they should explain the special circumstances that require a deviation from the Code of Ethics.
        \item The authors should make sure to preserve anonymity (e.g., if there is a special consideration due to laws or regulations in their jurisdiction).
    \end{itemize}

\item {\bf Broader impacts}
    \item[] Question: Does the paper discuss both potential positive societal impacts and negative societal impacts of the work performed?
    \item[] Answer: \answerYes{} 
    \item[] Justification: We discuss societal impacts in the paper. 
    \item[] Guidelines:
    \begin{itemize}
        \item The answer NA means that there is no societal impact of the work performed.
        \item If the authors answer NA or No, they should explain why their work has no societal impact or why the paper does not address societal impact.
        \item Examples of negative societal impacts include potential malicious or unintended uses (e.g., disinformation, generating fake profiles, surveillance), fairness considerations (e.g., deployment of technologies that could make decisions that unfairly impact specific groups), privacy considerations, and security considerations.
        \item The conference expects that many papers will be foundational research and not tied to particular applications, let alone deployments. However, if there is a direct path to any negative applications, the authors should point it out. For example, it is legitimate to point out that an improvement in the quality of generative models could be used to generate deepfakes for disinformation. On the other hand, it is not needed to point out that a generic algorithm for optimizing neural networks could enable people to train models that generate Deepfakes faster.
        \item The authors should consider possible harms that could arise when the technology is being used as intended and functioning correctly, harms that could arise when the technology is being used as intended but gives incorrect results, and harms following from (intentional or unintentional) misuse of the technology.
        \item If there are negative societal impacts, the authors could also discuss possible mitigation strategies (e.g., gated release of models, providing defenses in addition to attacks, mechanisms for monitoring misuse, mechanisms to monitor how a system learns from feedback over time, improving the efficiency and accessibility of ML).
    \end{itemize}
    
\item {\bf Safeguards}
    \item[] Question: Does the paper describe safeguards that have been put in place for responsible release of data or models that have a high risk for misuse (e.g., pretrained language models, image generators, or scraped datasets)?
    \item[] Answer: \answerYes{} 
    \item[] Justification: We do not provide or use prompts that may generate harmful or disruptive content. 
    \item[] Guidelines:
    \begin{itemize}
        \item The answer NA means that the paper poses no such risks.
        \item Released models that have a high risk for misuse or dual-use should be released with necessary safeguards to allow for controlled use of the model, for example by requiring that users adhere to usage guidelines or restrictions to access the model or implementing safety filters. 
        \item Datasets that have been scraped from the Internet could pose safety risks. The authors should describe how they avoided releasing unsafe images.
        \item We recognize that providing effective safeguards is challenging, and many papers do not require this, but we encourage authors to take this into account and make a best faith effort.
    \end{itemize}

\item {\bf Licenses for existing assets}
    \item[] Question: Are the creators or original owners of assets (e.g., code, data, models), used in the paper, properly credited and are the license and terms of use explicitly mentioned and properly respected?
    \item[] Answer: \answerYes{} 
    \item[] Justification: We have cited the relevant works.
    \item[] Guidelines:
    \begin{itemize}
        \item The answer NA means that the paper does not use existing assets.
        \item The authors should cite the original paper that produced the code package or dataset.
        \item The authors should state which version of the asset is used and, if possible, include a URL.
        \item The name of the license (e.g., CC-BY 4.0) should be included for each asset.
        \item For scraped data from a particular source (e.g., website), the copyright and terms of service of that source should be provided.
        \item If assets are released, the license, copyright information, and terms of use in the package should be provided. For popular datasets, \url{paperswithcode.com/datasets} has curated licenses for some datasets. Their licensing guide can help determine the license of a dataset.
        \item For existing datasets that are re-packaged, both the original license and the license of the derived asset (if it has changed) should be provided.
        \item If this information is not available online, the authors are encouraged to reach out to the asset's creators.
    \end{itemize}

\item {\bf New assets}
    \item[] Question: Are new assets introduced in the paper well documented and is the documentation provided alongside the assets?
    \item[] Answer: \answerNA{} 
    \item[] Justification: We do not release new assets.
    \item[] Guidelines:
    \begin{itemize}
        \item The answer NA means that the paper does not release new assets.
        \item Researchers should communicate the details of the dataset/code/model as part of their submissions via structured templates. This includes details about training, license, limitations, etc. 
        \item The paper should discuss whether and how consent was obtained from people whose asset is used.
        \item At submission time, remember to anonymize your assets (if applicable). You can either create an anonymized URL or include an anonymized zip file.
    \end{itemize}

\item {\bf Crowdsourcing and research with human subjects}
    \item[] Question: For crowdsourcing experiments and research with human subjects, does the paper include the full text of instructions given to participants and screenshots, if applicable, as well as details about compensation (if any)? 
    \item[] Answer: \answerNA{} 
    \item[] Justification: Does not involve research with human subjects. 
    \item[] Guidelines:
    \begin{itemize}
        \item The answer NA means that the paper does not involve crowdsourcing nor research with human subjects.
        \item Including this information in the supplemental material is fine, but if the main contribution of the paper involves human subjects, then as much detail as possible should be included in the main paper. 
        \item According to the NeurIPS Code of Ethics, workers involved in data collection, curation, or other labor should be paid at least the minimum wage in the country of the data collector. 
    \end{itemize}

\item {\bf Institutional review board (IRB) approvals or equivalent for research with human subjects}
    \item[] Question: Does the paper describe potential risks incurred by study participants, whether such risks were disclosed to the subjects, and whether Institutional Review Board (IRB) approvals (or an equivalent approval/review based on the requirements of your country or institution) were obtained?
    \item[] Answer: \answerNA{} 
    \item[] Justification: Does not involve research with human subjects. 
    \item[] Guidelines:
    \begin{itemize}
        \item The answer NA means that the paper does not involve crowdsourcing nor research with human subjects.
        \item Depending on the country in which research is conducted, IRB approval (or equivalent) may be required for any human subjects research. If you obtained IRB approval, you should clearly state this in the paper. 
        \item We recognize that the procedures for this may vary significantly between institutions and locations, and we expect authors to adhere to the NeurIPS Code of Ethics and the guidelines for their institution. 
        \item For initial submissions, do not include any information that would break anonymity (if applicable), such as the institution conducting the review.
    \end{itemize}

\item {\bf Declaration of LLM usage}
    \item[] Question: Does the paper describe the usage of LLMs if it is an important, original, or non-standard component of the core methods in this research? Note that if the LLM is used only for writing, editing, or formatting purposes and does not impact the core methodology, scientific rigorousness, or originality of the research, declaration is not required.
    \item[] Answer: \answerNA{} 
    \item[] Justification: Core method development of our work did not involve usage of LLMs. 
    \item[] Guidelines:
    \begin{itemize}
        \item The answer NA means that the core method development in this research does not involve LLMs as any important, original, or non-standard components.
        \item Please refer to our LLM policy (\url{https://neurips.cc/Conferences/2025/LLM}) for what should or should not be described.
    \end{itemize}

\end{enumerate}


\ifpaper
\else
    \clearpage
    \newpage
    \onecolumn
    \appendix
    \setcounter{section}{0}
    \def\thesection{\Alph{section}}
    \section*{Appendix}

  



\ifpaper
\else


\section{Proofs}
\label{sec:appendix_policy_derivations}
\subsection{Derivation of the Target Distribution}
\label{sec:appendix_optimal_marginal_target}
From Eq.~\ref{eq:reward_max_obj}, we obtain the target distribution $p^{*}_0$, which maximizes the reward while maintaining proximity to the distribution of the pretrained model $p_0$:
\begin{align}
    \nonumber
    p^{*}_0(\B{x}_0) &= \argmax_{q} \  \mathbb{E}_{\B{x}_0 \sim q} \left[ r(\B{x}_0) \right] -\beta \mathcal{D}_{\text{KL}} \left[ q \| p_0 \right], \\
    \nonumber
    &= \argmax_{q} \mathbb{E}_{\B{x}_0 \sim q} \left[r(\B{x}_0) - \beta \log \frac{q(\B{x}_0)}{p_0(\B{x}_0)} \right] \\
    \nonumber
    &= \argmin_{q} \mathbb{E}_{\B{x}_0 \sim q} \left[\log \frac{q (\B{x}_0)}{p_0(\B{x}_0)} - \frac{1}{\beta} r(\B{x}_0) \right] \\
    \nonumber
    &= \argmin_{q} \int q(\B{x}_0) \log \frac{q(\B{x}_0)}{p_0 (\B{x}_0)} \mathrm{d} \B{x}_0 -\frac{1}{\beta} \int q(\B{x}_0) r(\B{x}_0) \mathrm{d} \B{x}_0.
\end{align}

This can be solved via calculus of variation where the functional $\mathcal{J}$ is given as follows:
\begin{align}
    \nonumber
    \mathcal{J} \left[ q(\B{x}_0) \right] \coloneqq \int q(\B{x}_0) \left( \log \frac{q(\B{x}_0)}{p_0 (\B{x}_0)}-\frac{1}{\beta}r(\B{x}_0) \right)\mathrm{d} \B{x}_0.
\end{align}

Substituting $\tilde{q}(\B{x}_0, \epsilon) \coloneqq q(\B{x}_0) + \epsilon \eta(\B{x}_0)$ gives:
\begin{align}
    \nonumber
    \mathcal{J} \left[ \tilde{q}(\B{x}_0, \epsilon) \right] = \int \tilde{q}(\B{x}_0, \epsilon) \left( \log \frac{\tilde{q}(\B{x}_0, \epsilon)}{p_0 (\B{x}_0)}-\frac{1}{\beta} r(\B{x}_0) \right)\mathrm{d} \B{x}_0, 
\end{align}
where $\eta(\B{x}_0)$ is an arbitrary smooth function, and $\epsilon$ is a scalar parameter. 


Introducing a Lagrange multiplier $\lambda$ to constraint $\int q(\B{x}_0) \mathrm{d} \B{x}_0 = 1$ gives:
\begin{align}
    \nonumber
    \mathcal{J} \left[ \tilde{q}(\B{x}_0, \epsilon) \right] &= \int \tilde{q}(\B{x}_0, \epsilon) \left( \log \frac{\tilde{q}(\B{x}_0, \epsilon)}{p_0 (\B{x}_0)}-\frac{1}{\beta}r(\B{x}_0) \right) +\lambda \tilde{q}(\B{x}_0, \epsilon) \mathrm{d} \B{x}_0\\
    \nonumber
    & :=\int f\{ \tilde{q};\B{x}_0\} \mathrm{d} \B{x}_0
\end{align}
Then the problem boils down to finding a function $\tilde{q}(\B{x}_0, \epsilon)$ satisfying:
\begin{align}
    \nonumber
    \left. \frac{\partial \mathcal{J}}{\partial \epsilon} \right|_{\epsilon=0} = 0
\end{align}

This can be solved using the Euler-Lagrange equation:
\begin{align}
    \nonumber
    \frac{\partial f}{\partial q}-\frac{\mathrm{d}}{\mathrm{d} \B{x}_0} \frac{\partial f}{\partial q'}=0,
\end{align}
where $q'$ is a derivative of $q$ with respect to $\B{x}_0$ and tilde notation is dropped since the condition is to be satisfied at $\epsilon=0$.

Note that $q'$ does not appear in $f$, so the Euler-Lagrange equation simplifies to:
\begin{align}
    \nonumber 
    \frac{\partial f}{\partial q}&=\frac{\partial}{\partial q}\left( q(\B{x}_0) \left( \log \frac{q(\B{x}_0)}{p_0 (\B{x}_0)}-\frac{1}{\beta} r(\B{x}_0) \right) +\lambda q(\B{x}_0) \right)=0\\
    \label{eq:euler_lagrange}
    &=\log \frac{q(\B{x}_0)}{p_0 (\B{x}_0)}-\frac{1}{\beta} r(\B{x}_0)+1+\lambda=0.
\end{align}

Solving Eq.~\ref{eq:euler_lagrange} gives the target distribution $p_0^{*}$, which minimizes the objective function in Eq.~\ref{eq:reward_max_obj}:
\begin{align}
    \label{eq:target_dist_unnormalized}
    p^{*}_0(\B{x}_0)=p_0 (\B{x}_0) \exp  \left(\frac{r(\B{x}_0)}{\beta}-1-\lambda \right)
\end{align}

Lastly, the Lagrangian multiplier $\lambda$ is obtained from the normalization constraint, $\exp (\lambda)=\int p_0 (\B{x}_0) \exp \left(\frac{r(\B{x}_0)}{\beta}-1\right) \mathrm{d} \B{x}_0$. Plugging this into Eq.~\ref{eq:target_dist_unnormalized} gives the target distribution presented in Eq.~\ref{eq:target_distribution}:
\begin{align}
    \label{eq:appendix_target_distribution}
    p^{*}_0(\B{x}_0) = \frac{p_0(\B{x}_0) \exp \left( \frac{r(\B{x}_0)}{\beta} \right)}{\int p_0 (\B{x}_0) \exp \left( \frac{r(\B{x}_0)}{\beta} \right) \mathrm{d} \B{x}_0 },
\end{align}

\subsection{Derivation of the Optimal Policy}
\label{sec:appendix_optimal_soft_policy}
Here, we provide the derivations of the optimal policy given in Eq.~\ref{eq:optimal_policy} for completeness, which is proposed in previous works~\cite{Uehara:2024Finetuning, Uehara:2024Bridging}. 

To sample from the target distribution defined in Eq.~\ref{eq:appendix_target_distribution}, previous studies utilize an optimal policy $p^{*}_\theta(\B{x}_{t - \Delta t} | \B{x}_{t})$. 
The optimal value function $v(\B{x}_{t})$ is defined as the expected future reward at current timestep $t$:
\begin{align}
    \label{eq:appendix_value_func}
    v(\B{x}_t) = \beta \log \mathbb{E}_{\B{x}_0 \sim p_{\theta} (\B{x}_0 | \B{x}_t)} \left[ \exp \left(  \frac{r(\B{x}_0)}{\beta} \right) \right]
\end{align}

The optimal policy is the policy that maximizes the objective function:
\begin{align}
    \nonumber
    p_{\theta}^{*}(\B{x}_{t-\Delta t} | \B{x}_{t}) &= \argmax_{q(\cdot|\B{x}_t)} \mathbb{E}_{\B{x}_{t-\Delta t}\sim q(\cdot|\B{x}_t)} \left[ v(\B{x}_{t-\Delta t}) \right] -\beta \mathcal{D}_{\text{KL}} \left[ q(\cdot | \B{x}_{t}) \| p_\theta(\cdot | \B{x}_{t}) \right] \\
    \label{eq:soft_optimal_policy}
    &= \frac{p_\theta(\B{x}_{t-\Delta t} | \B{x}_{t}) \exp \left( \frac{1}{\beta} v(\B{x}_{t-\Delta t}) \right)}{\int p_\theta (\B{x}_{t-\Delta t} | \B{x}_t) \exp \left( \frac{1}{\beta} v(\B{x}_{t-\Delta t}) \right) d\B{x}_{t-\Delta t} }\\
    \label{eq:a.2.optmal_policy}
    &=\frac{p_\theta(\B{x}_{t-\Delta t} | \B{x}_{t}) \exp \left( \frac{1}{\beta} v(\B{x}_{t-\Delta t}) \right)}{\exp \left( \frac{1}{\beta} v(\B{x}_{t}) \right)}
\end{align}
where the last equality follows from the soft-Bellman equations~\cite{Uehara:2024Bridging}.  
For completeness, we present the theorem. 
\begin{theorem}
    \textnormal{(Theorem 1 of Uehara~\etal~\cite{Uehara:2024Bridging})}. The induced distribution of the optimal policy in Eq.~\ref{eq:soft_optimal_policy} is the target distribution in Eq.~\ref{eq:appendix_target_distribution}.
    \begin{align}
        \nonumber
        p^{*}_0(\B{x}_0) &= \int \left\{ p_1(\B{x}_1) \prod_{s=T}^{1} p_{\theta}^{*}(\B{x}_{\frac{s}{T} - \frac{1}{T}} | \B{x}_{\frac{s}{T}}) \right \} d\B{x}_{\frac{1}{T}:1}.
    \end{align}
\end{theorem}

However, computing the optimal value function in Eq.~\ref{eq:appendix_value_func} is non-trivial. 
Hence, we follow the previous works~\cite{Kim:2025DAS,Li2024:SVDD} and approximate it using the posterior mean 
$\B{x}_{0|t} := \mathbb{E}_{\B{x}_0\sim p_\theta (\B{x}_0|\B{x}_t)}\left[ \B{x}_0\right]$:
\begin{align}
    \nonumber
    v(\B{x}_t) &= \beta \log \left( \int \exp (\frac{r(\B{x}_0)}{\beta}) p_\theta(\B{x}_0 | \B{x}_t) \mathrm{d} \B{x}_0 \right) \\
    \label{eq:appendix_soft_value_func}
    &\approx \beta \log (\exp (\frac{r(\B{x}_{0|t})}{\beta})) = r(\B{x}_{0|t}). 
\end{align}


\section{Choice of Diffusion Coefficient}
\label{sec:appendix_diffusion_coefficient_choice}
Ma~\etal~\cite{Ma2024:sit} have shown that the diffusion coefficient can be chosen freely within the stochastic interpolant framework~\cite{Albergo:2023Interpolant}. 
Here, we present a more comprehensive proof. 
We use $\B{w}$ interchangeably to denote the standard Wiener process for both forward and reverse time flows. 

\begin{proposition}


For a linear stochastic process $ \B{x}_t = \alpha_t \B{x}_0 + \sigma_t \B{x}_1 $ and the Probability-Flow ODE $ \mathrm{d} \B{x}_t = u_t(\B{x}_t) \mathrm{d} t $ that yields the marginal density $p_t(\B{x}_t)$, the following forward and reverse SDEs with an arbitrary diffusion coefficient $g_t \geq 0$ share the same marginal density:
\begin{align}
    \label{eq:appx_forward_sde}
    \text{Forward SDE: }\mathrm{d} \B{x}_t&=\left[u_t(\B{x}_t) + \frac{g_t^2}{2}\nabla\log p_t(\B{x}_t)\right] \mathrm{d} t+g_t \mathrm{d} \B{w}\\
    \label{eq:appx_reverse_sde}
    \text{Reverse SDE: }\mathrm{d} \B{x}_t&=\left[u_t(\B{x}_t) - \frac{g_t^2}{2}\nabla\log p_t(\B{x}_t)\right] \mathrm{d} t+g_t \mathrm{d} \B{w}.
\end{align}

\end{proposition}

\begin{proof}
\label{proof:prop1}
When velocity field $u_t$ generates a probability density path $p_t$, it satisfies the continuity equation:
\begin{align}
\label{eq:continuity equation}
    \frac{\partial}{\partial t}p_t(\B{x}_t)=-\nabla\cdot\left(p_t(\B{x}_t)u_t(\B{x}_t)\right).
\end{align}
Similarly, for the SDE $\mathrm{d} \B{x}_t=\mathbf{f}_t(\B{x}_t) \mathrm{d} t + g_t \mathrm{d} \B{w}$, the Fokker-Planck equation describes the time evolution of $\tilde{p}_t$:
\begin{align}
\label{eq:fokker planck}
    \frac{\partial}{\partial t}\tilde{p}_t(\B{x}_t)=-\nabla\cdot (\tilde{p}_t(\B{x}_t)\mathbf{f}_t(\B{x}_t))+\frac{1}{2}g_t^2\nabla^2\tilde{p}_t(\B{x}_t)
\end{align}
where $\nabla^2$ denotes the Laplace operator. 

To find an SDE that yields the same marginal probability density as the ODE, we equate the probability density functions in Eq.~\ref{eq:fokker planck} and Eq.~\ref{eq:continuity equation}, resulting in the following equation:
\begin{align}
    \nonumber -&\nabla\cdot (p_t(\B{x}_t)\mathbf{f}_t(\B{x}_t))+\frac{1}{2}g_t^2\nabla^2p_t(\B{x}_t)=-\nabla\cdot(p_t(\B{x}_t)u_t(\B{x}_t))\\
    \label{eq:continuity_equals_fokker_planck}
    &\nabla\cdot (p_t(\B{x}_t)(\mathbf{f}_t(\B{x}_t)-u_t(\B{x}_t)))=\frac{1}{2}g_t^2\nabla^2p_t(\B{x}_t)
\end{align}
This implies that any SDE with drift coefficient $\mathbf{f}_t(\B{x}_t)$ and diffusion coefficient $g_t$ that satisfies Eq.~\ref{eq:continuity_equals_fokker_planck} will generate $p_t$.
One particular choice is to set $p_t(\B{x}_t)(\mathbf{f}_t(\B{x}_t)-u_t(\B{x}_t))$ proportional to $\nabla p_t(\B{x}_t)$, \ie $p_t(\B{x}_t)(\mathbf{f}_t(\B{x}_t)-u_t(\B{x}_t))=A_t \nabla p_t(\B{x}_t)$. 
Then Eq.~\ref{eq:continuity_equals_fokker_planck} can be rewritten as:
\begin{align}
    \nonumber
    A_t \nabla^2 p_t(\B{x}_t)=\frac{1}{2}g_t^2\nabla^2p_t(\B{x}_t),
\end{align}
which leads to the relation $A_t=\frac{1}{2}g_t^2$. 
Similarly, the drift coefficient is given by:
\begin{align}
    \nonumber
    \mathbf{f}_t(\B{x}_t)&=u_t(\B{x}_t)+\frac{1}{2}g_t^2\frac{\nabla p_t(\B{x}_t)}{p_t(\B{x}_t)} \\
    \nonumber
    &=u_t(\B{x}_t)+\frac{1}{2}g^2_t \nabla\log p_t(\B{x}_t)
\end{align}
Thus, a family of SDEs that generate $p_t$ takes the following form:
\begin{align}
    \nonumber
    \mathrm{d} \B{x}_t = \left[u_t(\B{x}_t) + \frac{1}{2}g^2_t\nabla\log p_t(\B{x}_t) \right] \mathrm{d}t + g_t \mathrm{d} \B{w},
\end{align}
which is the forward SDE presented in Eq.~\ref{eq:appx_forward_sde}. 
Similarly, the reverse SDE in Eq.~\ref{eq:appx_reverse_sde} can be derived by applying the time reversal formula, following Anderson~\etal~\cite{Anderson:1982SDE}. 
\end{proof}

\begin{corollary}
If diffusion coefficient is chosen as $g_t=\sqrt{2\left( \sigma_t \dot{\sigma}_t-\sigma_t^2\frac{\dot{\alpha}_t} {\alpha_t} \right)}$ then the score function $\nabla\log p_t(\B{x}_t)$ inside the forward SDE vanish and it can be written as:
\begin{align}
    \label{eq:appx corollary eq}
    \mathrm{d} \B{x}_t=\frac{\dot{\alpha}_t} {\alpha_t}\B{x}_t \mathrm{d} t + \sqrt{2\left( \sigma_t \dot{\sigma}_t-\sigma_t^2\frac{\dot{\alpha}_t} {\alpha_t} \right)} \mathrm{d} \B{w}
\end{align}
\end{corollary}
\begin{proof}
    Velocity field $u_t(\B{x}_t)$ for linear stochastic process $\B{X}_t=\alpha_t \B{X}_0+\sigma_t\B{X}_1$ is given as:
    \begin{align}
        u_t(\B{x}_t)=\frac{\dot{\alpha}_t} {\alpha_t}\B{x}_t-\left( \sigma_t \dot{\sigma}_t-\sigma_t^2\frac{\dot{\alpha}_t} {\alpha_t} \right)\nabla\log p_t(\B{x}_t)
    \end{align}
    Plugging this equation into forward SDE Eq.~\ref{eq:appx_forward_sde}, we can immediately see that when $g_t=\sqrt{2\left( \sigma_t \dot{\sigma}_t-\sigma_t^2\frac{\dot{\alpha}_t} {\alpha_t} \right)}$ the score function term vanishes and the remaining terms constitute Eq.~\ref{eq:appx corollary eq}.
\end{proof}

\section{Adaptive Time Scheduling and Rollover Strategy}
\label{sec:appendix_adaptive_time_and_rollover}
In this section, we provide details of adaptive time scheduling and NFE analysis result which inspired rollover strategy. 
\paragraph{Adaptive Time Scheduling.}
As discussed in Sec.~\ref{sec:scheduler_conversion}, to maximize the exploration space in VP-SDE sampling, we design the time scheduler to take smaller steps during the initial phase—when variance is high—and gradually increase the step size in later stages. Specifically, we define the time scheduler as $t_\text{new} = \sqrt{1-(1-t)^2}$. 
While this approach can be problematic when the number of steps is too low—resulting in excessively large discretization steps in later iterations—we find that using a reasonable number of steps (\eg $10$) works well in practice, benefiting from the few-step generation capability of flow models. This setup effectively balances a broad exploration space with fast inference time, highlighting one of the key advantages of flow models over diffusion models. 

\paragraph{NFE Analysis.}
As discussed in Sec.~\ref{sec:roll_over}, we analyze the number of function evaluations (NFEs) required to obtain a sample with a higher reward than the current one. In Fig.~\ref{fig:temporal_axis_scaling}, we visualize the variance band of the required NFEs across timesteps, with the blue-dotted line representing the uniform allocation used in previous particle sampling methods~\cite{Li2024:SVDD, Singh:2025CoDE}. Notably, uniform compute allocation may constrain exploration and fail to identify high-reward samples, as evidenced by crossings within the variance band. This observation motivates the use of a rollover strategy to optimize compute utilization efficiently. As demonstrated in Sec.~\ref{sec:applications}, our experiments confirm that \Ours{} provides additional improvements over previous particle sampling methods~\cite{Li2024:SVDD, Singh:2025CoDE}.

\begin{figure*}[t!]
    \centering
    \begin{minipage}{0.63\linewidth} 
        \centering
        \includegraphics[width=0.9\linewidth]{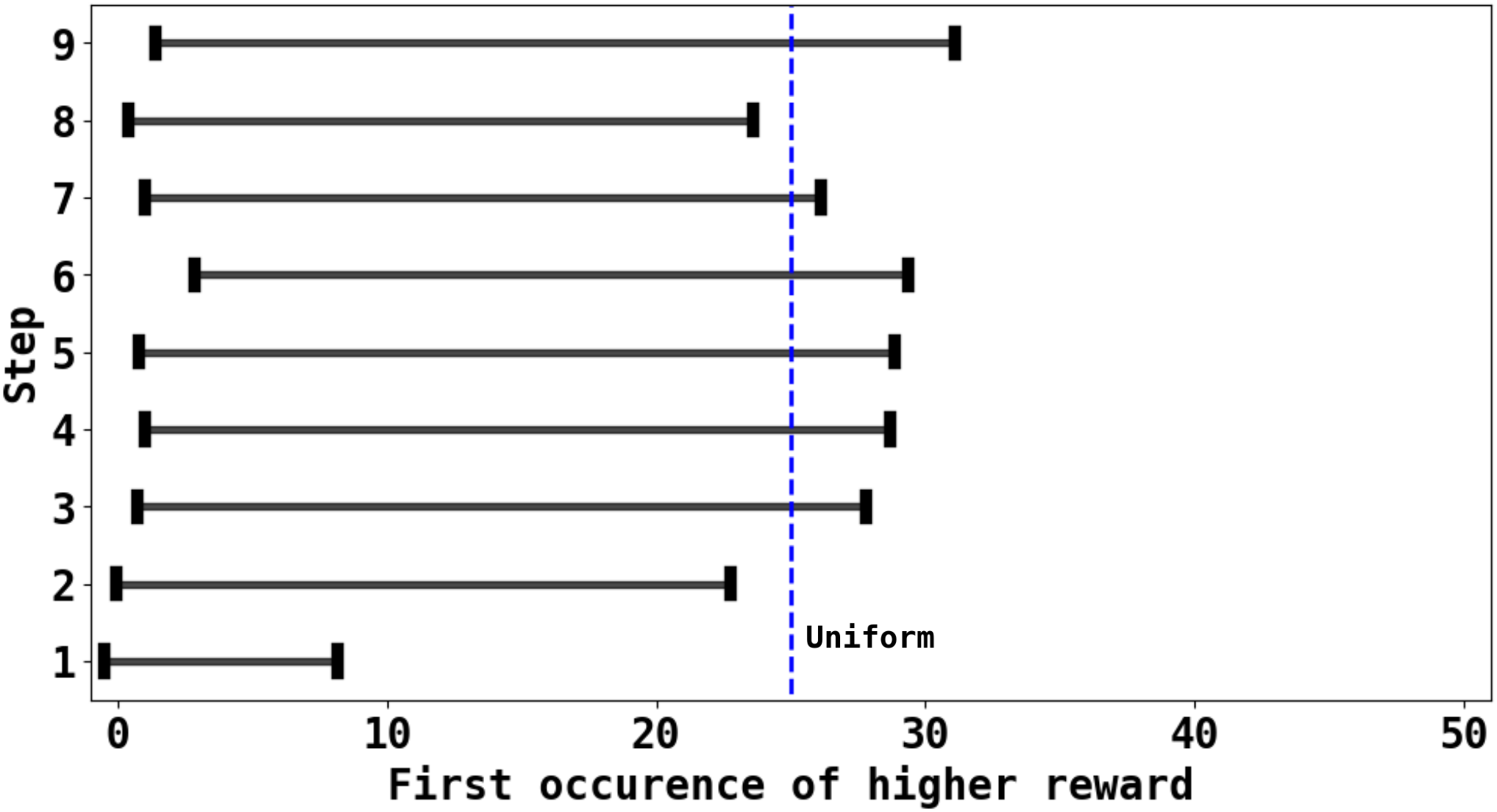}
    \end{minipage}%
    \hfill
    \begin{minipage}{0.35\linewidth} 
        \scriptsize
        \captionof{figure}{\textbf{Analysis of number of function evaluations (NFEs) across timesteps.}  
        The NFEs required to achieve a higher reward for each timestep. The plot illustrates the $\pm 1$ sigma variation band. The blue-dotted line represents the uniform allocation of compute (NFEs) across timesteps. We observe that the NFEs required to identify a higher-reward sample may exceed the uniformly allocated budget (blue dotted line).}
        \label{fig:temporal_axis_scaling}
    \end{minipage}
\end{figure*}



\begin{figure*}[ht!]
    \centering
    \includegraphics[width=\linewidth]{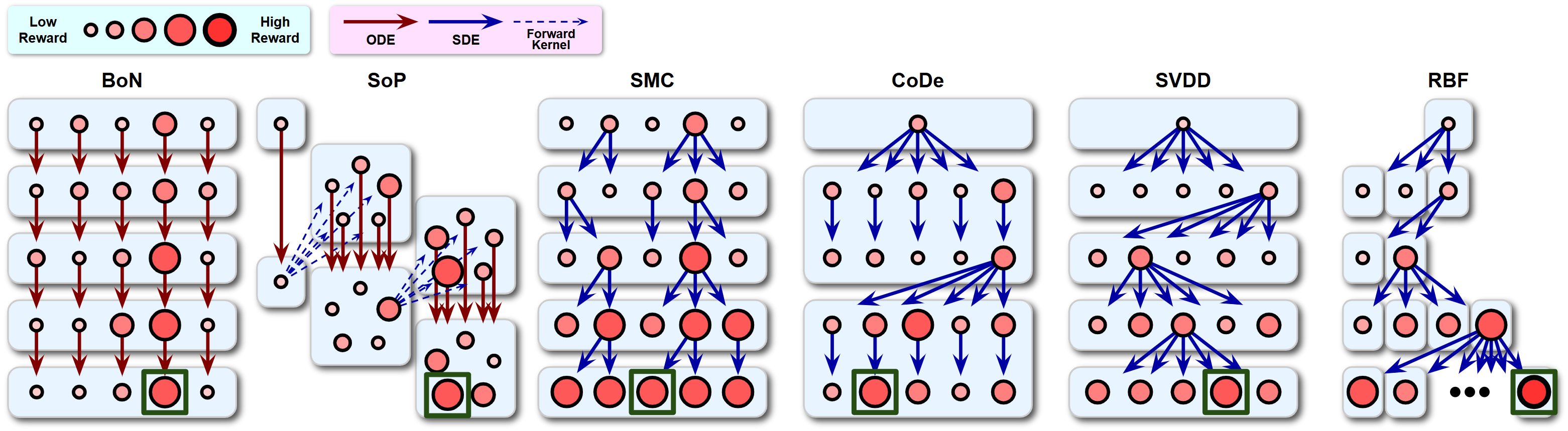}
    \caption{
    \textbf{Schematics of inference-time search algorithms.} Linear-ODE-based methods, BoN and SoP use a deterministic sampling process, whereas particle-sampling-based methods follow a stochastic process. Note that \Ours{} adaptively allocates NFEs across denoising timesteps.
    }
    \label{fig:search_schematic}
\end{figure*}

\section{Search Algorithms}
\label{sec:appendix_inference_time_scaling}
In this section, we introduce the inference-time search algorithms discussed in Sec.~\ref{sec:related} along with their implementation details. An illustrative figure of the algorithms is provided in Fig.~\ref{fig:search_schematic}. 
Here, we define the batch size $(N)$ as the number of initial latent samples and the particle size $(K)$ as the number of samples drawn from the proposal distribution $p_\theta(\B{x}_{t-\Delta t} | \B{x}_t)$ at each denoising step. 
\\
\\
\textbf{Best-of-N} (BoN)~\cite{Stiennon:2020BoN,Tang:2024Realfill} is a form of rejection sampling. Given $N$ generated samples $\{ \B{x}^{(i)}_{0} \}_{i=1}^N$, BoN selects the sample with the highest reward. 
\begin{align}
    \nonumber
    \B{x}_{0} &= \argmax_{\{ \B{x}^{(i)}_{0} \}_{i=1}^N } r(\B{x}_{0}^{(i)}).
\end{align}
As presented in Sec.~\ref{sec:applications}, we fixed the total compute budget to $500$ NFEs and the number of denoising steps to $10$, which sets the batch size of BoN to $N=50$.
\\
\\
\textbf{Search over Paths} (SoP)~\cite{Ma2025:SoP} begins by sampling $N$ initial noises and running the ODE solver up to a predefined timestep $t_0$. Then the following two operations iterate until reaching $t=0$:

\begin{enumerate}[leftmargin=*]
    \item Applying the forward kernel: For each sample in the batch at time $t$, $K$ particles are sampled using the forward kernel, which propagates them from $t$ to $t + \Delta_f$.
    \item Solving the ODE: The resulting $N \cdot K$ particles are then evolved from $t + \Delta_f$ to $t + \Delta_f - \Delta_b$ by solving the ODE. The top $N$ candidates with the highest rewards are selected.
\end{enumerate}
We followed the original implementations~\cite{Ma2025:SoP} for $\Delta_f$ and $\Delta_b$. We used $N=2$ and $K=5$.
\\
\vspace{0.5\baselineskip}
\\
%
%
\textbf{Sequential Monte Carlo} (SMC)~\cite{Kim:2025DAS, Doucet2001:SMC} 
extends the idea of importance sampling to a time-sequential setting by maintaining $N$ samples and updating their importance weights over time:
\begin{align}
    \nonumber
    w_{t-\Delta t}^{(i)} =\frac{p^*_\theta(\B{x}_{t-\Delta t}|\B{x}_t)}{q(\B{x}_{t-\Delta t}|\B{x}_t)}w_t^{(i)} =\frac{p_\theta(\B{x}_{t-\Delta t}|\B{x}_t) \exp( v(\B{x}^{(i)}_{t-\Delta t})/\beta)}{q(\B{x}_{t-\Delta t}|\B{x}_t) \exp (v(\B{x}^{(i)}_{t})/\beta)}w_t^{(i)},
\end{align}
where $q(\B{x}_{t-\Delta t}|\B{x}_t)$ is a proposal distribution and the last equality follows from the optimal policy Eq.~\ref{eq:a.2.optmal_policy}. We used the reverse process of the pretrained model as the proposal distribution, which leads to the following importance weight equation:
\begin{align}
    \label{eq:smc_importance_weight}
    w_{t-\Delta t}^{(i)}=\frac{\exp( v(\B{x}^{(i)}_{t-\Delta t})/\beta)}{\exp( v(\B{x}^{(i)}_{t})/\beta)}w_t^{(i)}.
\end{align}
At each step when effective sample size $\left(\sum_{j=1}^N w_t^{(j)} \right)^2/\sum_{i=1}^N (w_t^{(i)})^2$ is below the threshold, we perform resampling, \ie, indices $\{ a_t^{(i)} \}_{i=1}^{N} $ are first sampled from a multinomial distribution based on the normalized importance weights:
\begin{align}
    \nonumber
    \{ a_t^{(i)} \}_{i=1}^{N} \sim \text{Multinomial} \left(N,\quad \left\{ \frac{w_t^{(i)}}{\sum_{j=1}^{N}w_t^{(j)}} \right\}_{i=1}^{N} \right).
\end{align}
These ancestor indices $a_t^{(i)}$ are then used to replicate high-weight particles and discard low-weight ones, yielding the resampled set $\{ \B{x}_{t}^{ ( a_t^{(i)} ) } \}_{i=1}^N$. If resampling is not performed, the indices are simply set as $a_t^{(i)} = i$.
Lastly, one-step denoised samples are obtained from $\{\B{x}_{t}^{ ( a_t^{(i)} ) } \}_{i=1}^N$:
\begin{align}
    \nonumber
    \B{x}^{(i)}_{t-\Delta t}  \sim p_\theta(\B{x}_{t-\Delta t} | \B{x}_{t}^{ ( a_t^{(i)} ) } ).
\end{align}
When resampling is performed, the importance weights are reinitialized to one, \ie, $w_t=\B{1}$. The importance weights for the next step, $w_{t-\Delta t}$ are subsequently computed according to Eq.~\ref{eq:smc_importance_weight}, regardless of whether resampling was applied.
\\ 
We used $N=50$ for all applications. 
\\
\\
\textbf{Controlled Denoising} (CoDe)~\cite{Singh:2025CoDE} extends BoN by incorporating an interleaved selection step after every $L$ denoising steps. 
\begin{align}
    \nonumber
    \B{x}_{t-L\Delta t} &= \argmax_{\{ \B{x}^{(i)}_{t-L\Delta t} \}_{i=1}^{K} } \exp\left(v(\B{x}^{(i)}_{t-L\Delta t})/\beta\right) 
\end{align}

We used $N=2$, $K=25$, and $L=2$ for all applications. 
\\
\\
%
%
\textbf{SVDD}~\cite{Li2024:SVDD} approximates the optimal policy in Eq.~\ref{eq:optimal_policy} by leveraging weighted $K$ particles:
\begin{align}
    \label{supp:SVDD_eq}
    p^{*}_\theta(\B{x}_{t-\Delta t} | \B{x}_{t}) &\approx \sum_{i=1}^K  \frac{w^{(i)}_{t-\Delta t}}{\sum_{j=1}^K w^{(j)}_{t-\Delta t}} \delta_{\B{x}^{(i)}_{t-\Delta t}} \\
    \nonumber
    \{ \B{x}^{(i)}_{t-\Delta t} \}_{i=1}^{K} & \sim p_\theta(\B{x}_{t-\Delta t} | \B{x}_{t})\\
    \nonumber
    w^{(i)}_{t-\Delta t} &= \exp( v ( \B{x}^{(i)}_{t-\Delta t} ) / \beta ).
\end{align}
At each timestep, the approximate optimal policy in Eq.~\ref{supp:SVDD_eq} is sampled by first drawing an index $a_{t-\Delta t}$ from a categorical distribution:
\begin{align}
    a_{t-\Delta t} \sim \text{Categorical}\left(\left\{ \frac{w^{(i)}_{t-\Delta t}}{\sum_{j=1}^K w^{(j)}_{t-\Delta t}} \right\}_{i=1}^{K}\right)
\end{align}
This index is then used to select the sample from $\{ \B{x}^{(i)}_{t-\Delta t} \}_{i=1}^{K}$, \ie, $\B{x}_{t-\Delta t}\leftarrow \B{x}_{t-\Delta t}^{(a_{t-\Delta t})}$.
In practice, SVDD uses $\beta=0$, replacing sampling from the categorical distribution with a direct $\argmax$ operation, \ie, selecting the particle with the largest importance weight.
Following the original implementation~\cite{Li2024:SVDD}, we used $N=2$ and $K=25$ for all applications. 
\\
\\
\textbf{Rollover Budget Forcing} (\Ours{}) adaptively allocates compute across denoising timesteps. 
At each timestep, when a particle with a higher reward than the previous one is discovered, it immediately takes a denoising step, and the remaining NFEs are rolled over to the next timestep, ensuring efficient utilization of the available compute. To maintain consistency with SVDD~\cite{Li2024:SVDD}, we set $N=2$, with the compute initially allocated uniformly across all timesteps. 
We present the pseudocode for sampling from the stochastic proposal distribution with interpolant conversion in Alg.~\ref{alg:one-step_stochastic_denoising}. Specifically, the pseudocode for \Ours{} with SDE conversion and interpolant conversion is provided in Alg.~\ref{alg:ours}. 
Here, we denote $\{S^{(i)}\}_{i=1}^{M}$ as a sequence of timesteps in descending order, where $S^{(1)} = 1$ and $S^{(M)} = 0$, and $M$ is the total number of denoising steps. 

\SetAlCapHSkip{0pt}
\setlength{\algomargin}{0em}
\SetKwInput{KwInput}{Inputs}
\SetKwInput{KwOutput}{Outputs}
\SetCommentSty{mycommfont}

\newcommand\mycommfont[1]{\normalsize\rmfamily\textcolor{gray}{#1}}
\SetCommentSty{mycommfont}

\begin{figure*}[ht!]
    \centering
    \begin{minipage}{0.43\linewidth}
        \centering
        \begin{algorithm}[H]
            \setstretch{1.67}
            \SetAlgoLined
            \DontPrintSemicolon
            \caption{\texttt{\footnotesize{stoch\_denoise}}{\small{: 1-step stochastic denoising}}}
            \label{alg:one-step_stochastic_denoising}
            {
            \KwInput{
                original velocity field $u$, original interpolant $(\alpha, \sigma)$, new interpolant $(\Bar{\alpha}, \Bar{\sigma})$, diffusion coefficient $g$, current sample $\Bar{\B{x}}_s$, current timestep $s$, denoising step size $\Delta s$
            }
            \KwOutput{
                Stochastically denoised sample $\Bar{\B{x}}_{s - \Delta s}$
            }
            $t_{s} \leftarrow \rho^{-1} (\Bar{\rho}(s)) \quad c_{s} \leftarrow \Bar{\sigma}_{s} / \sigma_{t_{s}}$
            
            $ \Bar{\B{u}}_s \leftarrow \frac{\dot{c}_s}{c_s} \Bar{\B{x}}_s + c_s \dot{t}_s u_{t_s} \left(\frac{\Bar{\B{x}}_{s}}{c_s} \right) $ \tcp*{Eq.~\ref{eq:velocity_transform}}
        
            $ \B{s}_s \leftarrow \frac{1}{\Bar{\sigma}_s} \frac{\Bar{\alpha}_s \Bar{\B{u}}_s - \dot{\Bar{\alpha}}_s \Bar{\B{x}}_s}
            {\dot{\Bar{\alpha}}_s \Bar{\sigma}_s - \Bar{\alpha}_s \dot{\Bar{\sigma}}_s}$ \tcp*{Eq.~\ref{eq:score_velocity}}
        
            $ \B{f}_s = \Bar{\B{u}}_s - \frac{g_s^2}{2}\B{s}_s $ \tcp*{Eq.~\ref{eq:reverse_sde}}
        
            $\B{z} \sim \mathcal{N}(\B{0}, \textbf{\textit{I}})$
        
            $ \Bar{\B{x}}_{s-\Delta s} \leftarrow \Bar{\B{x}}_s - \B{f}_s \Delta s + g_s\sqrt{\Delta s} \ \B{z} $ 
            }
        \end{algorithm}
    \end{minipage}
    \hfill
    \begin{minipage}{0.53\linewidth}
        \centering
        \begin{algorithm}[H]
            \SetAlgoNoEnd
            \setstretch{1.25}
            \SetAlgoLined
            \DontPrintSemicolon
            \caption{Rollover Budget Forcing (RBF)}
            \label{alg:ours}
            {
            \KwInput{
                Number of denoising steps $M$, 
                timesteps $\{ S^{(i)} \}_{i=1}^{M} $,
                NFE quota $ \{ Q^{(i)} \}_{i=1}^{M}$
            }
            \KwOutput{
                Aligned sample $\Bar{\B{x}}_0$
            }
            $\Bar{\B{x}}_1\sim \mathcal{N}(0, \textbf{\textit{I}})\quad  r^* \leftarrow r(\Bar{\B{x}}_{0|1})$\\
            \For{$i \in \{ 1,\dots, M \}$} {
                $s \leftarrow S^{(i)} \quad \Delta s \leftarrow S^{(i)} - S^{(i + 1)} \quad q \leftarrow Q^{(i)}$ \\
                \For{ $j \in \{ 1, \dots, q \}$ } {
                    $ \Bar{\B{x}}_{s-\Delta s}^{(j)} \leftarrow \texttt{\footnotesize{stoch\_denoise}}\left(\Bar{\B{x}}_{s}, s, \Delta s \right)$ \tcp*{Alg.~\ref{alg:one-step_stochastic_denoising}}
                    \If{$r^{*} < r(\Bar{\B{x}}^{(j)}_{0|s-\Delta s})$} {
                        $ Q^{(i+1)} \leftarrow Q^{(i+1)} + Q^{(i)} - j$ \tcp*{Sec.~\ref{sec:roll_over}}
                        $ r^{*} \leftarrow r(\Bar{\B{x}}^{(j)}_{0|s-\Delta s}) \quad \Bar{\B{x}}_{s-\Delta s} \leftarrow \Bar{\B{x}}_{s-\Delta s}^{(j)} $ \\
                        \textbf{break} \\
                    }
                    \If {$j = q$} {
                        $k^{*} \leftarrow \argmax_{k\in \{1,\dots,q\}} r(\Bar{\B{x}}^{(k)}_{0|s-\Delta s}) $ \\
                        $\Bar{\B{x}}_{s-\Delta s} \leftarrow \Bar{\B{x}}_{s-\Delta s}^{(k^{*})}$ \\ 
                    }
                }
            }
            }
        \end{algorithm}
    \end{minipage}
\end{figure*}


\begin{figure*}[t!]
    \centering
    \begin{minipage}{0.44\linewidth}
        \centering
        \scriptsize
        \setlength{\tabcolsep}{2.9pt}
        \newcommand{\good}[1]{\tiny{\textcolor{blue}{+#1\%}}}
        \newcommand{\bad}[1]{\tiny{\textcolor{red}{-#1\%}}}
        \captionof{table}
        {\textbf{Quantitative results of aesthetic image generation.} \textsuperscript{\textdagger} denotes the given reward used in inference time. The best result in each row is highlighted in \textbf{bold}.}
        \label{tab:aesthetic}
        \newcolumntype{Z}{>{\centering\arraybackslash}m{0.4\linewidth}}
        \newcolumntype{X}{>{\centering\arraybackslash}m{0.27\linewidth}}
        \begin{tabular}{Z | X | X }
        \toprule
        Model & \makecell{Aesthetic\\Score\textsuperscript{\textdagger}~\cite{Schuhmann:aesthetics}} & \makecell{ImageReward\\~\cite{Xu2023:ImageReward} (held-out)} \\
        \midrule
        \makecell{FLUX~\cite{BlackForestLabs:2024Flux}} 
        & 5.795 
        & 0.991 \\
        \makecell{DPS~\cite{Chung:2023DPS}} 
        & \makecell{6.438} 
        & \makecell{0.605} \\
        \makecell{SVDD~\cite{Li2024:SVDD}+DPS~\cite{Chung:2023DPS}} 
        & \makecell{6.887} 
        & \makecell{1.077} \\
        \midrule
        \makecell{{\Oursbf{}} (Ours)+DPS~\cite{Chung:2023DPS}} 
        & \makecell{\textbf{7.170}} 
        & \makecell{\textbf{1.152}} \\
        \bottomrule
    \end{tabular}
    \end{minipage}
    \hfill
    \begin{minipage}{0.52\linewidth}
        \centering
        \scriptsize
        \setlength{\tabcolsep}{0.0em}
        \def\arraystretch{1.2}
        \newcolumntype{Z}{>{\centering\arraybackslash}m{0.24\textwidth}}
        \begin{tabularx}{0.96\textwidth}{Z Z Z Z}
            \toprule
             FLUX~\cite{BlackForestLabs:2024Flux} & DPS~\cite{Chung:2023DPS} & SVDD~\cite{Li2024:SVDD} + DPS~\cite{Chung:2023DPS} & \makecell{{\Oursbf{}} (Ours) \\ + DPS~\cite{Chung:2023DPS}} \\
            \midrule
            \multicolumn{4}{c}{\textit{``Bird''}} \\
            \includegraphics[width=0.24\textwidth]{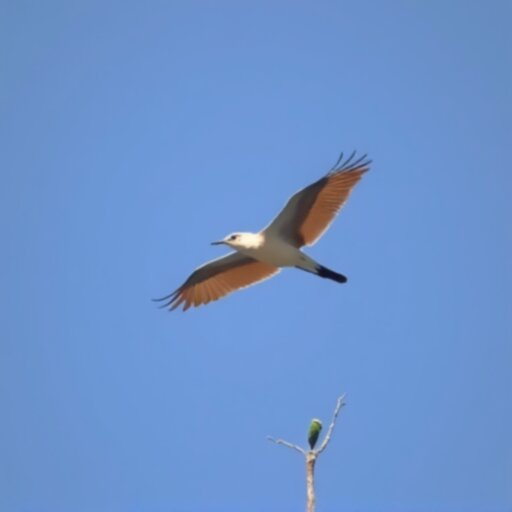} &
            \includegraphics[width=0.24\textwidth]{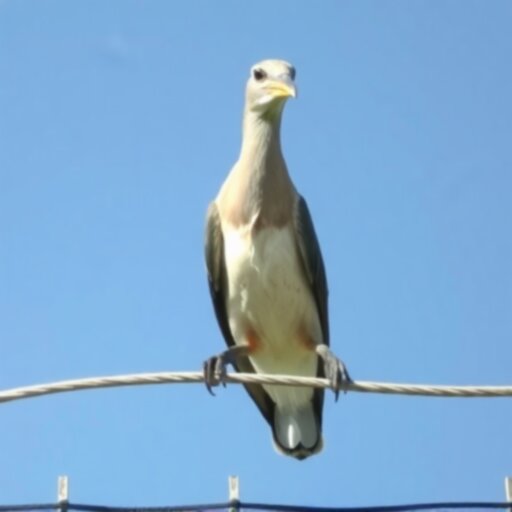} &
            \includegraphics[width=0.24\textwidth]{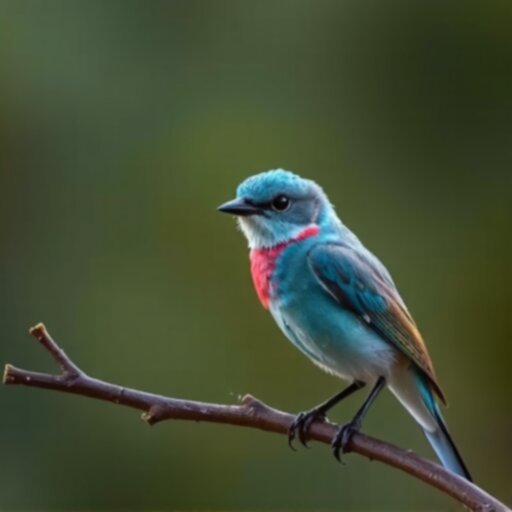} &
            \includegraphics[width=0.24\textwidth]{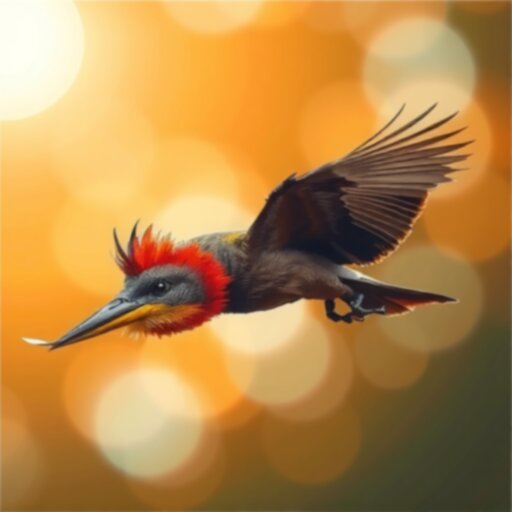} \\
            \midrule
            \multicolumn{4}{c}{\textit{``Bat''}} \\ 
            \includegraphics[width=0.24\textwidth]{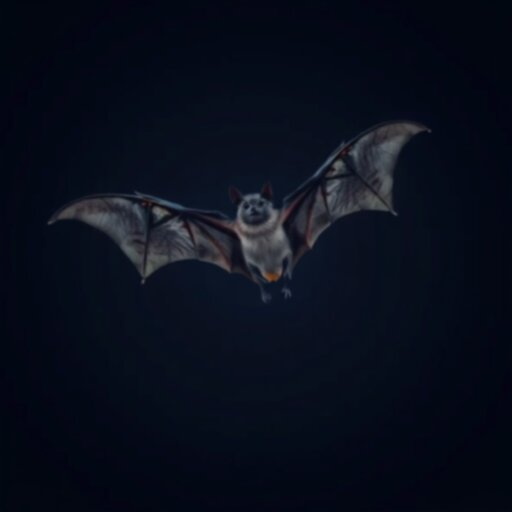} &
            \includegraphics[width=0.24\textwidth]{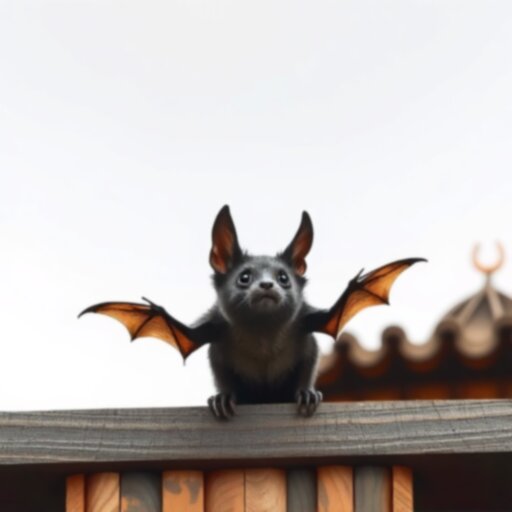} &
            \includegraphics[width=0.24\textwidth]{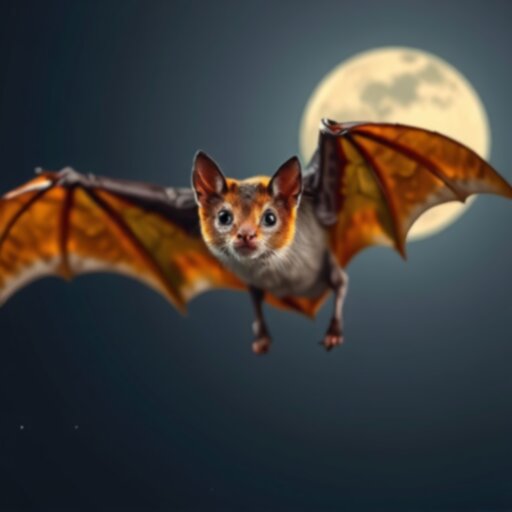} &
            \includegraphics[width=0.24\textwidth]{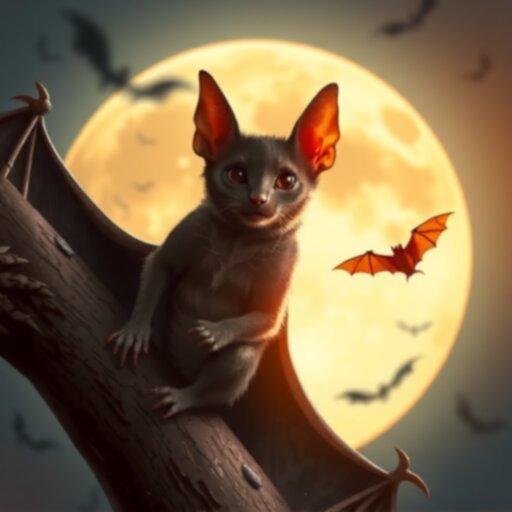} \\
            \bottomrule
        \end{tabularx}
        \caption{\textbf{Qualitative results of aesthetic image generation.} At inference-time, we guide generate using the aesthetic score~\cite{Schuhmann:aesthetics} as the given reward, which assesses \textbf{visual appeal}.}
        \label{fig:aesthetic}
    \end{minipage}
\end{figure*}

%
\section{Additional Results}
\subsection{Aesthetic Image Generation}
\label{sec:appendix_diff_reward}
In this section, we demonstrate that inference-time scaling can also be applied to gradient-based methods, such as DPS~\cite{Chung:2023DPS}, for differentiable rewards. Specifically, we consider aesthetic image generation and show that \Ours{} leads to synergistic performance improvements. We first derive the formulation of the proposal distribution for differentiable rewards and then present qualitative and quantitative results. 

\subsubsection{Gradient-Based Guidance}
Uehara~\etal~\cite{Uehara:2024Bridging} have shown that the marginal distribution $p^{*}_t(\B{x}_t)$ is computed as follows:
\begin{align}
    \nonumber
    p^{*}_t(\B{x}_t) &\propto \exp \left( \frac{v(\B{x}_t)}{\beta} \right) p_t(\B{x}_t) \approx \exp \left( \frac{r(\B{x}_{0|t})}{\beta} \right) p_t(\B{x}_t),
\end{align}
where the approximation follows from Eq.~\ref{eq:appendix_soft_value_func}. 
When the reward is differentiable (\eg, aesthetic score~\cite{Schuhmann:aesthetics}), one can simulate samples from $p^{*}_t(\B{x}_t)$ by computing its score function:
\begin{align}
    \nonumber
    \label{eq:differentiable_score}
    \nabla \log p_t^{*}(\B{x}_t) &= \nabla \log \left[ \exp (\frac{r(\B{x}_{0|t})}{\beta}) p_t(\B{x}_t) \right] \\
    &= \frac{1}{\beta} \underbrace{\nabla r (\B{x}_{0|t})}_{\text{Guidance}} + \underbrace{\nabla \log p_t(\B{x}_t)}_{\text{Pretrained Score}}.
\end{align}
For differentiable rewards, we incorporate the gradient-based guidance defined in Eq.~\ref{eq:differentiable_score} into the SDE sampling process described in Eq.~\ref{eq:reverse_sde}. 
Notably, this approach is orthogonal to inference-time scaling, and \Ours{} can be additionally utilized to further enhance performance. 
In the next section, we experimentally demonstrate that \Ours{} can be effectively integrated with gradient-based guidance. 

\subsubsection{Aesthetic Image Generation Results}
The aesthetic image generation task aims to sample images that best capture human preferences, such as visual appeal. We use $45$ animal prompts from previous work, DDPO~\cite{Black2024:DDPO}. The aesthetic score~\cite{Schuhmann:aesthetics} serves as the given reward, while ImageReward~\cite{Xu2023:ImageReward} is used as the held-out reward. 

We present quantitative and qualitative results of aesthetic image generation in Tab.~\ref{tab:aesthetic} and Fig.~\ref{fig:aesthetic}. Notably, \Ours{}, implemented with DPS~\cite{Chung:2023DPS}, achieves significant improvements on both the given and held-out rewards, even surpassing SVDD~\cite{Li2024:SVDD}. Qualitatively, \Ours{} effectively adapts the pretrained flow model to better align with human preferences, particularly in terms of visual appeal.  

\begin{table}[!t]
\setlength{\tabcolsep}{2.9pt}
\scriptsize
\centering
\caption{\textbf{Comparison of diffusion and flow models}.} 
\label{tab:diff_flow_comp}
\renewcommand{\arraystretch}{1.4}

\newcolumntype{Z}{>{\centering\arraybackslash}m{0.13\linewidth}}
\newcolumntype{X}{>{\centering\arraybackslash}m{0.08\linewidth}}

\begin{tabular}{X | Z | Z Z Z Z | X}
\toprule
Type & Model & ImageReward~\cite{Xu2023:ImageReward} & HPS~\cite{Wu2023:hps} & PickScore~\cite{Kirstain2023:pickapic} & CLIP Score~\cite{radford2021:CLIP} & Steps \\
\hline
\multirow{2}{*}{Diffusion} 
& SD2~\cite{Rombach:2022LDM}     & 0.429 & 0.280 & 0.218 & 0.269 & 50 \\
& SANA-1.5~\cite{xie2025sana}    & 0.894 & 0.284 & 0.222 & 0.270 & 20 \\
\hline
\multirow{2}{*}{Flow} 
& SD3~\cite{Esser2403:scaling}   & 1.154 & 0.294 & 0.226 & 0.277 & 28 \\
& FLUX~\cite{BlackForestLabs:2024Flux} & 1.054 & 0.290 & 0.226 & 0.275 & 5 \\
\bottomrule
\end{tabular}
\label{tab:appendix_diff_flow}
\end{table}

\subsection{Comparison of Diffusion and Flow Models}
\label{sec:appendix_diff_flow}
We present quantitative comparisons between text-to-image diffusion and flow models in Tab.~\ref{tab:appendix_diff_flow}, using compositional text prompts from GenAI-Bench~\cite{Jiang:2024GenAI}.
As shown, flow-based models outperform diffusion models across all evaluation metrics assessing image quality~\cite{Xu2023:ImageReward, Wu2023:hps, Kirstain2023:pickapic} and text alignment~\cite{radford2021:CLIP, Xu2023:ImageReward}. In the flow-based models, FLUX~\cite{BlackForestLabs:2024Flux} achieves competitive performance while requiring fewer steps compared to Stable Diffusion 3~\cite{Esser2403:scaling}. 

\subsection{Scaling Behavior Comparison}
As discussed in Sec.~\ref{sec:method}, expanding the exploration space and applying budget forcing significantly enhance the efficiency of \Ours{}, leading to superior performance improvements over BoN. 
Here, we compare the scaling behavior of BoN, a representative Linear-ODE-based method, with \Ours{} across different numbers of function evaluations (NFEs). 

We report qualitative and quantitative scaling results for quantity-aware image generation (Fig.~\ref{fig:appendix_nfe_scaling}, Tab.~\ref{tab:appendx_counting_nfe}) and for compositional text-to-image generation (Fig.~\ref{fig:composition_nfe_scaling_plot}, Tab.~\ref{tab:composition_nfe_scaling}), respectively.
Our results indicate that allocating more compute leads to performance improvements for both BoN and \Ours{}. 
However, the performance of BoN plateaus after $300$ NFEs, whereas \Ours{} continues to scale and achieves the highest reward in both tasks.  
Notably, \Ours{} shows similar trend in the held-out reward, outperforming BoN and demonstrating its efficiency. 

\begin{figure*}[t!]
    \centering
    \begin{minipage}{0.48\linewidth}
        \centering
        \scriptsize
        \setlength{\tabcolsep}{2.9pt}
        \renewcommand{\arraystretch}{1.25} 
        \captionof{table}{\textbf{Quantitative results of quantity-aware image generation in NFE scaling expriment.} We use the same $100$ prompts from T2I-CompBench~\cite{Huang:2025T2ICompBench++}. \textsuperscript{\textdagger} denotes the given reward.}
        \label{tab:appendx_counting_nfe}
        \newcolumntype{Z}{>{\centering\arraybackslash}m{0.01\textwidth}}
        \newcolumntype{Y}{>{\centering\arraybackslash}m{0.2\textwidth}}
        \newcolumntype{W}{>{\centering\arraybackslash}m{0.15\textwidth}}
        \newcolumntype{X}{>{\centering\arraybackslash}m{0.11\textwidth}}
        \begin{tabular}{Z X | W W Y Y}
            \toprule
            & NFEs & \makecell{RSS\textsuperscript{\textdagger}\\\cite{Liu:2024GroundingDINO} $\downarrow$} & Acc. $\uparrow$ & \makecell{VQAScore~\cite{Lin:2024CLIPFlanT5}\\(held-out) $\uparrow$ } & \makecell{Aesthetic\\Score~\cite{Schuhmann:aesthetics} $\uparrow$} \\
            \midrule
            \multirow{5}{*}{\rotatebox{90}{BoN}} & 50 & 4.360 & 0.400  & 0.758  & 5.408 \\
            & 100 & 3.280 & 0.510  & 0.750  & 5.522 \\
            & 300 & 2.190 & 0.570  & 0.755  & 5.463 \\
            & 500 & 1.760 & 0.580  & 0.756  & 5.420 \\
            & 1000 & 1.340 & 0.590  & 0.759  & 5.466 \\
            \midrule
            \multirow{5}{*}{\rotatebox{90}{{\Oursbf{}} (Ours)}} & 50 & 3.250 & 0.410  & 0.756  & 5.560 \\
            & 100 & 1.860 & 0.590  & 0.760  & 5.627 \\
            & 300 & 0.690 & 0.720  & 0.779  & 5.503 \\
            & 500 & 0.540 & 0.800  & 0.769  & 5.581 \\
            & 1000 & 0.290 & 0.880  & 0.777  & 5.526 \\
            \bottomrule
        \end{tabular}
    \end{minipage}
    \hfill
    \begin{minipage}{0.48\linewidth}
        \centering
        {\scriptsize
        \setlength{\tabcolsep}{0.0em}
        \def\arraystretch{0.0}
        \newcolumntype{Z}{>{\centering\arraybackslash}m{\linewidth}}
        \begin{tabularx}{\linewidth}{Z}
            \includegraphics[width=\linewidth]{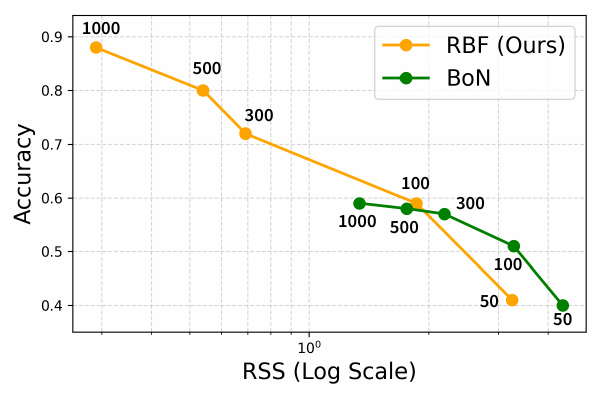}
        \end{tabularx}
        }
        \caption{
        \textbf{Quantity-aware image generation scaling behavior comparison of BoN and \Ours{}.} 
        We plot the known reward (RSS)~\cite{Liu:2024GroundingDINO} against accuracy for different numbers of function evaluations: $\{ 50, 100, 300, 500, 1,000 \}$. 
        Note that the horizontal axis is displayed on a logarithmic scale. 
        }
        \label{fig:appendix_nfe_scaling}
    \end{minipage}
\end{figure*}

\begin{figure*}[t!]
    \centering
    \begin{minipage}[t]{0.48\linewidth}
        \centering
        \scriptsize
        \setlength{\tabcolsep}{2.9pt}
        \renewcommand{\arraystretch}{1.30} 
        \captionof{table}{\textbf{Quantitative results of compositional text-to-image generation in NFE scaling expriment.}  We use the $121$ prompts from GenAI-Bench~\cite{Jiang:2024GenAI}. 
        \textsuperscript{\textdagger} denotes the given reward.}
        \label{tab:composition_nfe_scaling}
        \newcolumntype{Z}{>{\centering\arraybackslash}m{0.01\textwidth}}
        \newcolumntype{W}{>{\centering\arraybackslash}m{0.26\linewidth}}
        \newcolumntype{X}{>{\centering\arraybackslash}m{0.11\textwidth}}
        \newcolumntype{Y}{>{\centering\arraybackslash}m{0.23\linewidth}}
        \begin{tabular}{Z X | W Y Y}
            \toprule
            & NFEs & \makecell{VQAScore \textsuperscript{\textdagger} \\ \cite{Lin:2024CLIPFlanT5} $\uparrow$} & \makecell{Inst.BLIP~\cite{Dai:2023InstructBLIP}\\(held-out) $\uparrow$} & \makecell{Aesthetic \\ \cite{Schuhmann:aesthetics} $\uparrow$} \\
            \midrule
            \multirow{5}{*}{\rotatebox{90}{BoN}} 
            & 50   & 0.8310 & 0.8011 & 5.2246 \\
            & 100  & 0.8459 & 0.7959 & 5.2594 \\
            & 300  & 0.8775 & 0.8250 & 5.1414 \\
            & 500  & 0.8790 & 0.8200 & 5.1620 \\
            & 1000 & 0.8886 & 0.8269 & 5.2055 \\
            \midrule
            \multirow{5}{*}{\rotatebox{90}{{\Oursbf{}} (Ours)}} 
            & 50   & 0.8577 & 0.8253 & 5.2704 \\
            & 100  & 0.8824 & 0.8212 & 5.3213 \\
            & 300  & 0.9146 & 0.8387 & 5.2837 \\
            & 500  & 0.9250 & 0.8430 & 5.2370 \\
            & 1000 & 0.9283 & 0.8369 & 5.2593 \\
            \bottomrule
        \end{tabular}
    \end{minipage}
    \hfill
    \begin{minipage}[t]{0.48\linewidth}\vspace{-\topsep}
        \centering
        \includegraphics[width=\linewidth]{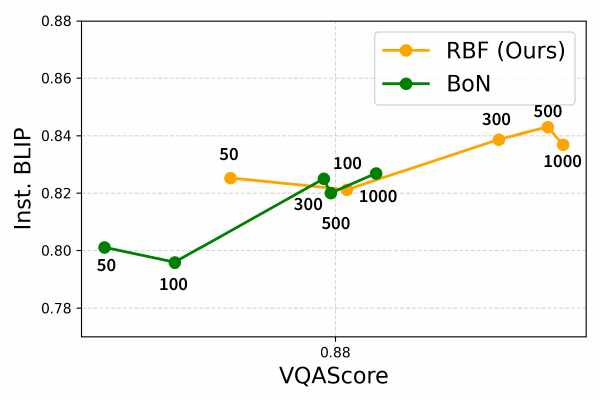}
        \caption{
        \textbf{Compositional text-to-image generation scaling behavior comparison of BoN and \Ours{}.} 
        We plot the known reward (VQAScore)~\cite{Lin:2024CLIPFlanT5} against the held-out reward~\cite{Dai:2023InstructBLIP} for different numbers of function evaluations: $\{ 50, 100, 300, 500, 1,000 \}$. 
        }
        \label{fig:composition_nfe_scaling_plot}
    \end{minipage}
\end{figure*}

\begin{figure*}[t!]
    \centering
    \begin{minipage}[t]{0.49\textwidth}
        \centering
        \scriptsize
        \setlength{\tabcolsep}{3pt}
        \captionof{table}{
            \textbf{Time complexity of scaling methods.}
        }
        \label{tab:inference_cost}
        \newcolumntype{X}{>{\centering\arraybackslash}m{0.23\linewidth}}
        \newcolumntype{Z}{>{\centering\arraybackslash}m{0.24\linewidth}}
        \newcolumntype{Y}{>{\centering\arraybackslash}m{0.09\linewidth}}
        \renewcommand{\arraystretch}{1.15}
        \begin{tabular}{Y Z X X}
            \toprule
            Base & BoN & SMC~\cite{Kim:2025DAS} & SVDD~\cite{Li2024:SVDD}, \Ours{} \\
            \midrule
            $S \cdot c_d$ &
            $N \cdot c_d + \frac{N}{S} \cdot c_v$ &
            $N \cdot c_d + N \cdot c_v$ &
            $N \cdot c_d + N \cdot c_v$ \\
            \bottomrule
        \end{tabular}
    \end{minipage}
    \hfill
    \begin{minipage}[t]{0.45\textwidth}
        \centering
        \scriptsize
        \setlength{\tabcolsep}{3pt}
        \captionof{table}{
            \textbf{Runtime of \Ours{}.}
        }
        \label{tab:runtime_rbf}
        \newcolumntype{X}{>{\centering\arraybackslash}m{0.04\linewidth}}
        \newcolumntype{Z}{>{\centering\arraybackslash}m{0.10\linewidth}}
        \renewcommand{\arraystretch}{1.15}
        \begin{tabularx}{\linewidth}{X | Z Z Z Z Z}
            \toprule
            & 50 & 100 & 300 & 500 & 1000 \\
            \midrule
            Runtime (sec) & 84.00 & 140.11 & 383.89 & 635.01 & 1243.68 \\
            VQAScore~\cite{Lin:2024CLIPFlanT5} & 0.858 & 0.882 & 0.915 & 0.925 & 0.928 \\
            \bottomrule
        \end{tabularx}
    \end{minipage}
\end{figure*}

\paragraph{Time Complexity and Compute Analysis.}
We present time complexity of scaling methods in Tab.~\ref{tab:inference_cost}. 
Let $S$ as the number of denoising steps, $N$ as the NFE budget, and $c_s$ and $c_v$ as the costs of the denoising and verification, respectively. 
Since all methods share the same NFE budget $N$, the total denoising cost is fixed at $N \cdot c_d$. 
For the verification cost, although BoN has the lowest cost, \Ours{} consistently outperforms BoN across all NFE budget regimes in both compositional text-to-image generation and quantity-aware image generation tasks (Fig.~\ref{fig:appendix_nfe_scaling} and Fig.~\ref{fig:composition_nfe_scaling_plot}) while incurring only a marginal increase in verification overhead.

Additionally, at inference time, a user can specify the compute budget (NFEs), which determines the total runtime of our method. We report the runtime of \Ours{} in Tab.~\ref{tab:runtime_rbf}. Under a $500$-NFE budget, scaling for compositional text-to-image generation (VQAScore~\cite{Lin:2024CLIPFlanT5}) requires $635.01$ seconds per image. 
Runtime can be reduced by lowering the NFE budget—at the cost of reward performance—and further accelerated by decreasing output resolution or increasing batch size.

For all experiments, we use FLUX~\cite{BlackForestLabs:2024Flux}, which requires approximately $32$GB of GPU memory, accounting for the majority of overall memory usage. 
All evaluations are performed on an NVIDIA RTX~A6000 GPU.


\section{Implementation Details}
\label{sec:appendix_impl_details}

\begin{figure*}[t!]
    \centering
    \scriptsize
    \setlength{\tabcolsep}{3pt}
    \renewcommand{\arraystretch}{1.15}
    \captionof{table}{
    \textbf{Choice of hyperparameters.}
    Evaluation of the images generated with different (a) number of denoising steps and (b) diffusion coefficient.}
    \label{tab:hyperparams}
    \begin{subtable}[t]{0.4\textwidth}
        \centering
        \newcolumntype{X}{>{\centering\arraybackslash}m{0.18\linewidth}}
        \newcolumntype{Y}{>{\centering\arraybackslash}m{0.3\linewidth}}
        \begin{tabularx}{\linewidth}{Y | Y Y}
            \toprule
            Steps & Aesthetic \cite{Schuhmann:aesthetics} & \makecell{Diversity~\cite{Kim:2025DAS}} \\
            \midrule
            10 & \textbf{5.635} & \textbf{0.084} \\
            20 & 5.680 & 0.103 \\
            \bottomrule
        \end{tabularx}
        \caption{Number of denoising steps}
        \label{tab:r1_steps}
    \end{subtable}
    \hfill
    \begin{subtable}[t]{0.58\textwidth}
        \centering
        \newcolumntype{Z}{>{\centering\arraybackslash}m{0.1\linewidth}}
        \newcolumntype{Y}{>{\centering\arraybackslash}m{0.2\linewidth}}
        \begin{tabularx}{\linewidth}{Z | Y Y | Y Y}
            \toprule
            Norm & \makecell{Aesthetic~\cite{Schuhmann:aesthetics}\\$g(t)=t$} & \makecell{Diversity~\cite{Kim:2025DAS}\\$g(t)=t$} & \makecell{Aesthetic~\cite{Schuhmann:aesthetics}\\$g(t)=t^2$} & \makecell{Diversity~\cite{Kim:2025DAS}\\$g(t)=t^2$} \\
            \midrule
            1 & 5.635 & 0.084 & 5.652 & 0.083 \\
            3 & 5.168 & 0.153 & \textbf{5.436} & \textbf{0.158} \\
            5 & 4.608 & 0.223 & 4.838 & 0.187 \\
            \bottomrule
        \end{tabularx}
        \caption{Diffusion coefficient}
        \label{tab:r2_diffusion}
    \end{subtable}
\end{figure*}

\subsection{Choice of Hyperparameters}
We report quantitative results on aesthetic score~\cite{Schuhmann:aesthetics} and diversity~\cite{Kim:2025DAS} for images generated under different settings of the number of denoising steps and the diffusion coefficient. 
As shown in Tab.~\ref{tab:hyperparams}(a), the number of denoising steps beyond $10$ gives marginal gains. 
Hence, we fixed the number of denoising steps to $10$ to ensure fair and efficient evaluation across all methods. Note that once the number of denoising steps is fixed, the total particle count per step is automatically determined by dividing the total NFE budget by the number of steps. 
Additionally, Tab.~\ref{tab:hyperparams}(b) reports results obtained under varying diffusion coefficients scaled by different norms. 
We found that using $g(t) = 3t^2$ consistently offered the best trade-off between sample diversity and output fidelity, so we adopt this setting for all SDE sampling.

\subsection{Compositional Text-to-Image Generation}
In the compositional text-to-image generation task, we use the VQAScore as the reward, which evaluates image-text alignment using a visual question-answering (VQA) model (CLIP-FlanT5~\cite{Lin:2024CLIPFlanT5} and InstructBLIP~\cite{Dai:2023InstructBLIP}). Specifically, VQAScore measures the probability that a given attribute or object is present in the generated image. To compute the reward, we scale the probability value by setting $\beta=0.1$ in Eq.~\ref{eq:optimal_policy}. 

\subsection{Quantity-Aware Image Generation}
In quantity-aware image generation, text prompts specify objects along with their respective quantities. To generate images that accurately match the specified object counts, we use the negation of the Residual Sum of Squares (RSS) as the given reward. 
Here, RSS is computed to measure the discrepancy between the detected object count $\hat{C}_i$ and the target object count $C_i$ in the text prompt: $\text{RSS} = \sum_{i=1}^{n} \left( C_i - \hat{C}_i \right)^2$,
where $n$ is the total number of object categories in the prompt. 
We additionally report accuracy, which is defined as $1$ when $\text{RSS} = 0$ and $0$ otherwise. 
For the held-out reward, we report VQAScore measured with CLIP-FlanT5~\cite{Lin:2024CLIPFlanT5} model. 

\paragraph{Object Detection Implementation Details.} 
To compute the given reward, RSS, it is necessary to detect the number of objects per category, $\hat{C}_i$. Here, we leverage the state-of-the-art object detection model, GroundingDINO~\cite{Liu:2024GroundingDINO} and the object segmentation model SAM~\cite{Kirillov2023:SAM}, which is specifically used to filter out duplicate detections. 

We observe that na\"ively using the detection model~\cite{Liu:2024GroundingDINO} to compute RSS leads to poor detection accuracy due to two key issues: inner-class duplication and cross-class duplication.
Inner-class duplication occurs when multiple detections are assigned to the same object within a category, leading to overcounting. This often happens when an object is detected both individually and as part of a larger group.
Cross-class duplication arises when an object is assigned to multiple categories due to shared characteristics (\eg, a toy airplane being classified as both a toy and an airplane), making it difficult to assign it to a single category. 

To address inner-class duplication, we refine the object bounding boxes detected by GroundingDINO~\cite{Liu:2024GroundingDINO} using SAM~\cite{Kirillov2023:SAM} and filter out overlapping detections. Smaller bounding boxes are prioritized, and larger ones that significantly overlap with existing detections are discarded. This ensures that each object is counted only once within its category.
To resolve cross-class duplication, we assign each object to the category with the highest GroundingDINO~\cite{Liu:2024GroundingDINO} confidence score which prevents duplicate counting across multiple classes.

\begin{center}
    \hfill \break
    \textbf{More qualitative results are presented in the following pages.}
\end{center}

\clearpage
\newpage

\section{Additional Qualitative Results}
\subsection{Comparisons of Inference-Time SDE Conversion and Interpolant Conversion}
\label{sec:appendix_quali_ode_sde_vp}
\begin{figure}[ht!]
{\tiny
\setlength{\tabcolsep}{0.0em}
\def\arraystretch{1.0}
\newcolumntype{X}{>{\centering\arraybackslash}m{0.005\textwidth}}
\newcolumntype{Y}{>{\centering\arraybackslash}m{0.05\textwidth}}
\newcolumntype{Z}{>{\centering\arraybackslash}m{0.155\textwidth}}
    \begin{tabularx}{\textwidth}{Y | Z Z Z X | X Z Z Z}
        \toprule
        & \scriptsize{Linear-ODE} & \scriptsize{Linear-SDE} & \scriptsize{VP-SDE} & & & \scriptsize{Linear-ODE} & \scriptsize{Linear-SDE} & \scriptsize{VP-SDE} \\
        \midrule
        & \multicolumn{3}{c}{\textit{``Six people gathered for a picnic.''}} & & & \multicolumn{3}{c}{\textit{``Four candles, two balloons, one dog, two tomatoes and three helicopters.''}} \\
        \rotatebox{90}{\makecell{\normalsize{SMC}~\cite{Kim:2025DAS}\\Quantity}} & \includegraphics[width=0.15\textwidth]{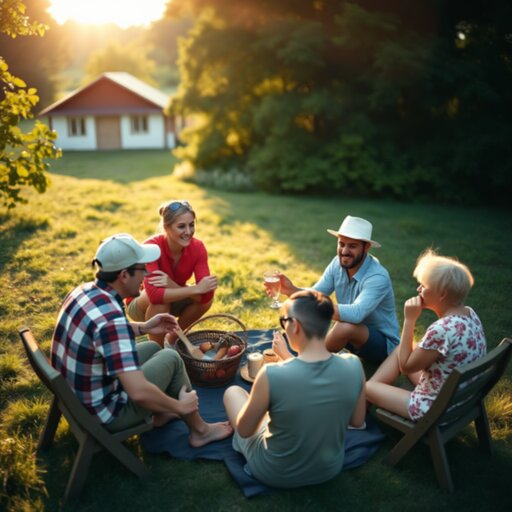} &
        \includegraphics[width=0.15\textwidth]{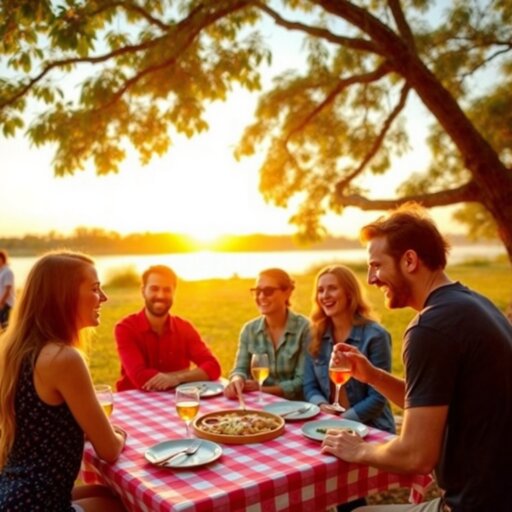} &
        \includegraphics[width=0.15\textwidth]{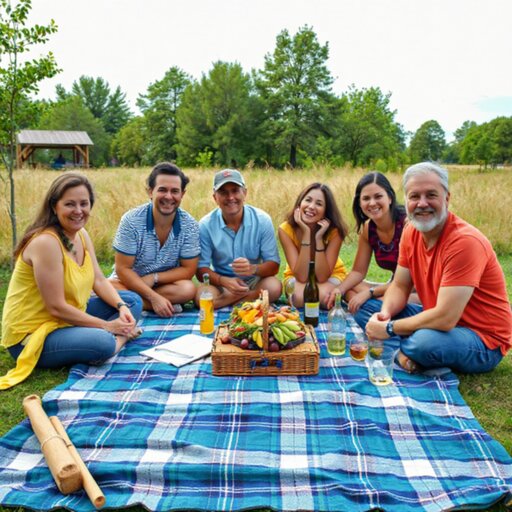} & & &
        \includegraphics[width=0.15\textwidth]{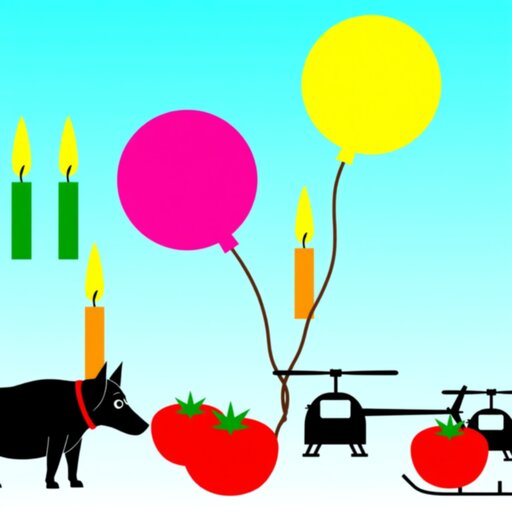} &
        \includegraphics[width=0.15\textwidth]{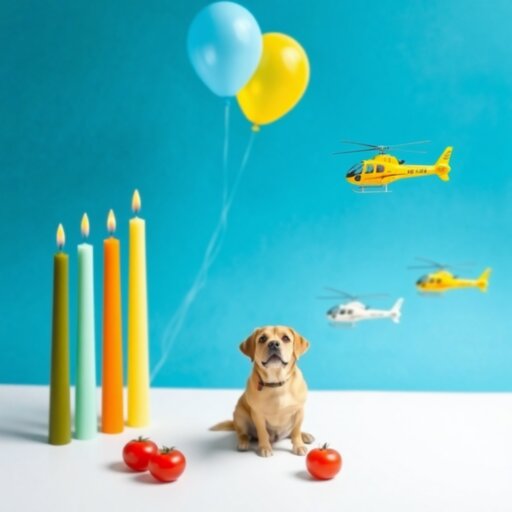} &
        \includegraphics[width=0.15\textwidth]{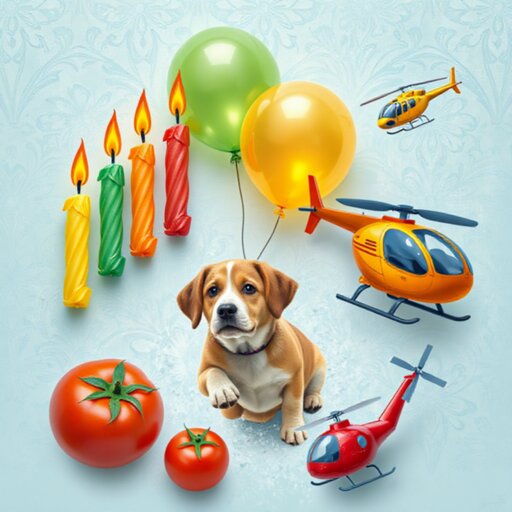} \\
        \midrule
        & \multicolumn{3}{c}{\textit{``Four rabbits, three apples, two mice and four televisions.''}} & & & \multicolumn{3}{c}{\textit{``Seven pigs snorted and played in the mud.''}} \\
        \rotatebox{90}{\makecell{\normalsize{SMC}~\cite{Kim:2025DAS}\\Quantity}} & \includegraphics[width=0.15\textwidth]{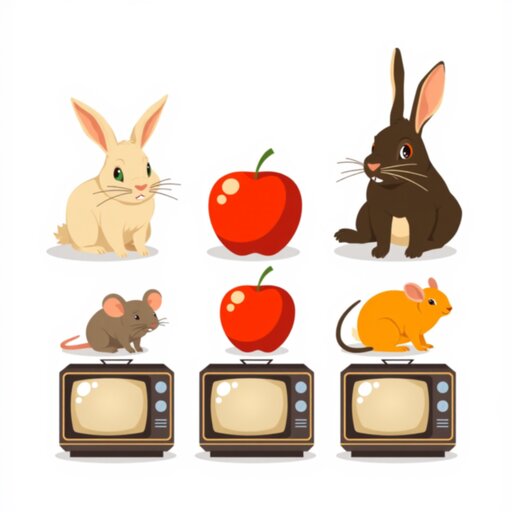} &
        \includegraphics[width=0.15\textwidth]{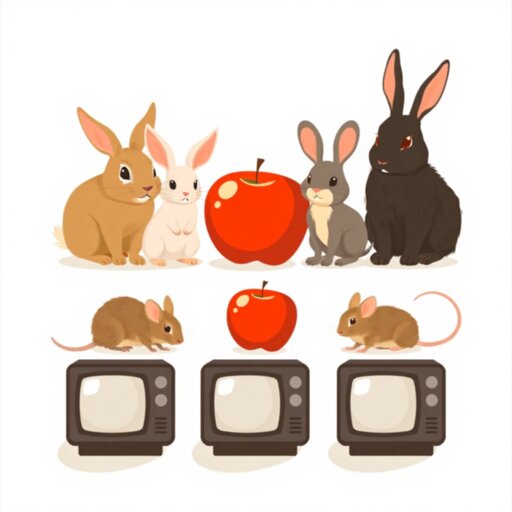} &
        \includegraphics[width=0.15\textwidth]{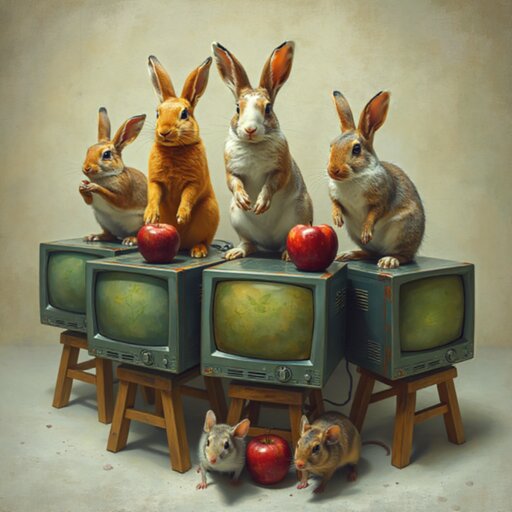} & & &
        \includegraphics[width=0.15\textwidth]{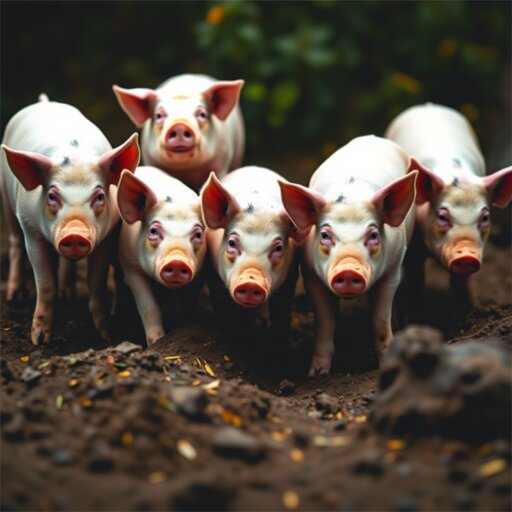} &
        \includegraphics[width=0.15\textwidth]{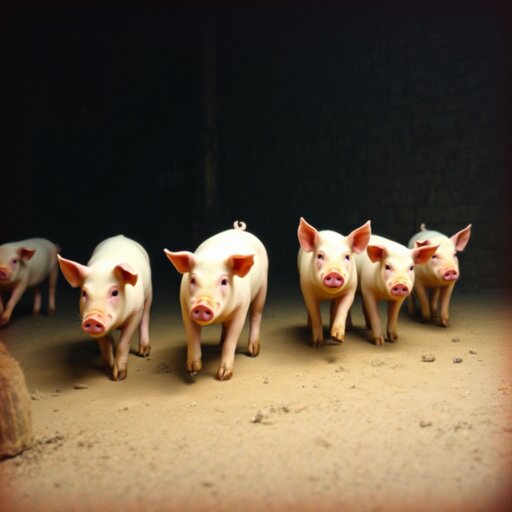} &
        \includegraphics[width=0.15\textwidth]{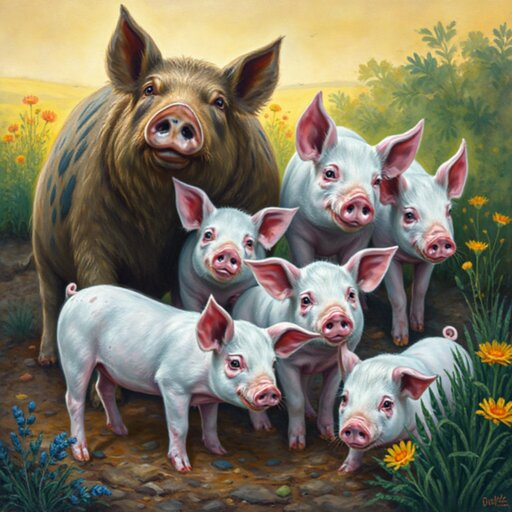} \\
        \midrule
        & \multicolumn{3}{c}{\makecell{\textit{``Two frogs in tracksuits, competing in a high jump.}\\ \textit{The frog in blue tracksuit jumps higher than the frog not in blue tracksuit.''}}} & & & \multicolumn{3}{c}{\textit{\makecell{``Three purple gemstones and one pink gemstone,\\with the pink gemstone having the smoothest looking surface.''}}} \\
        \rotatebox{90}{\makecell{\normalsize{SMC}~\cite{Kim:2025DAS}\\Composition}} &
        \includegraphics[width=0.15\textwidth]{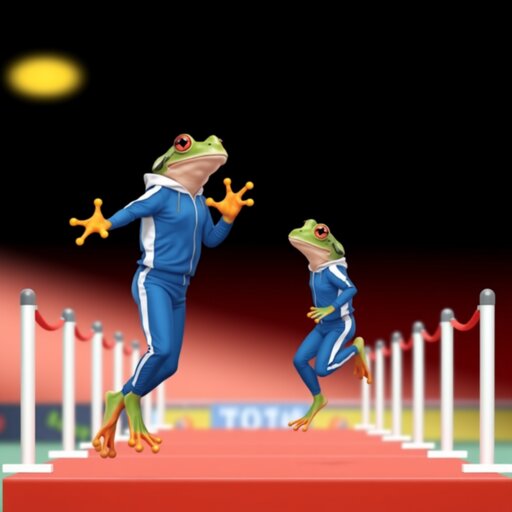} &
        \includegraphics[width=0.15\textwidth]{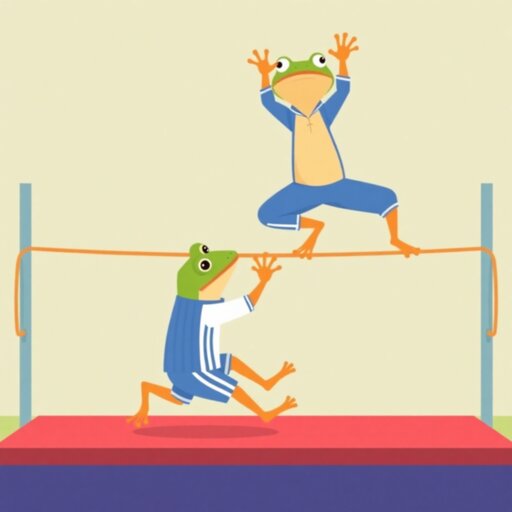} &
        \includegraphics[width=0.15\textwidth]{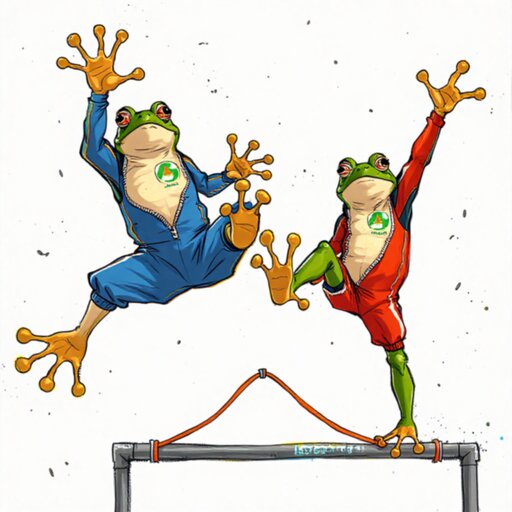} & & &
        \includegraphics[width=0.15\textwidth]{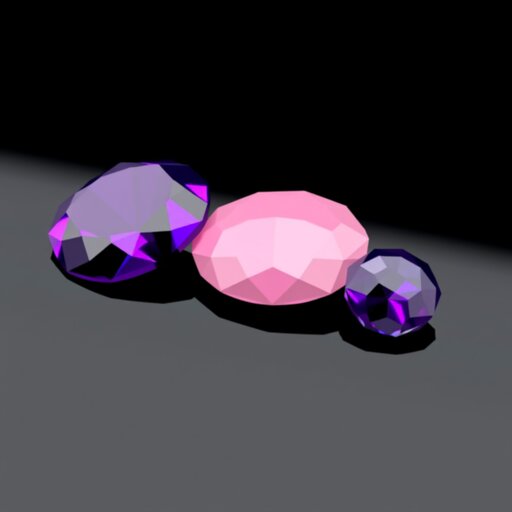} &
        \includegraphics[width=0.15\textwidth]{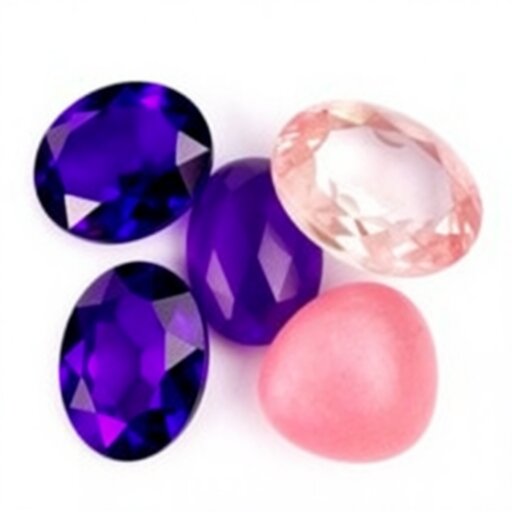} &
        \includegraphics[width=0.15\textwidth]{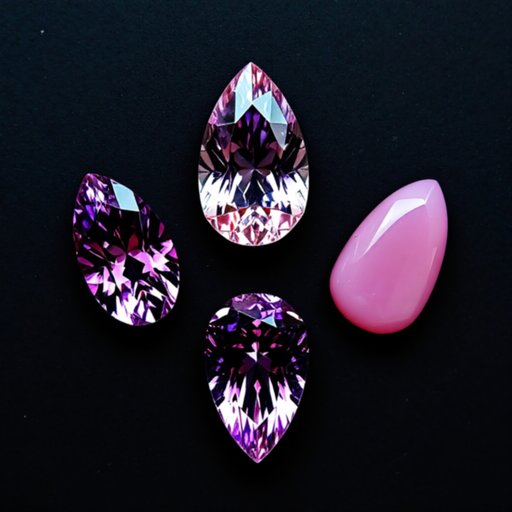} \\
        %
        \toprule
        & \scriptsize{Linear-ODE} & \scriptsize{Linear-SDE} & \scriptsize{VP-SDE} & & & \scriptsize{Linear-ODE} & \scriptsize{Linear-SDE} & \scriptsize{VP-SDE} \\
        \midrule
        & \multicolumn{3}{c}{\textit{``Seven helmets''}} & & & \multicolumn{3}{c}{\textit{``Four couches, three candles, two fish, one frog and three plates.''}} \\
        \rotatebox{90}{\makecell{\normalsize{CoDe}~\cite{Singh:2025CoDE}\\Quantity}} & \includegraphics[width=0.15\textwidth]{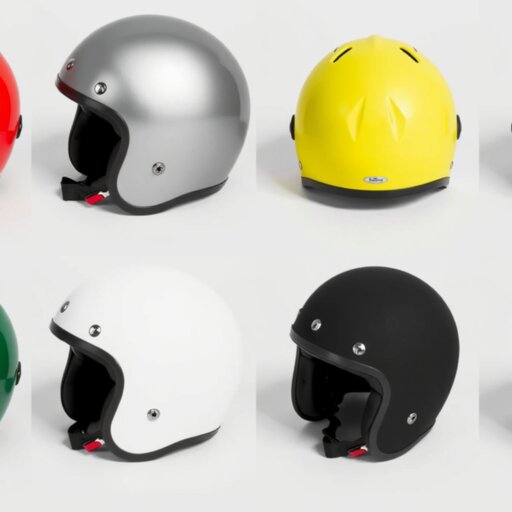} &
        \includegraphics[width=0.15\textwidth]{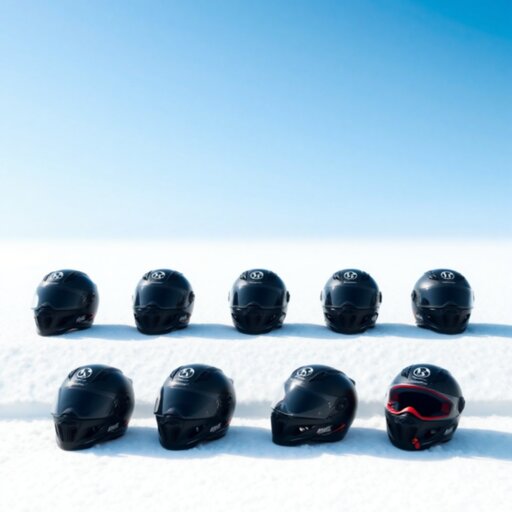} &
        \includegraphics[width=0.15\textwidth]{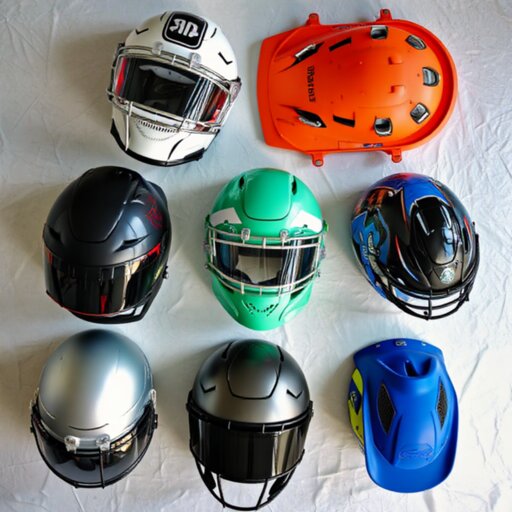} & & &
        \includegraphics[width=0.15\textwidth]{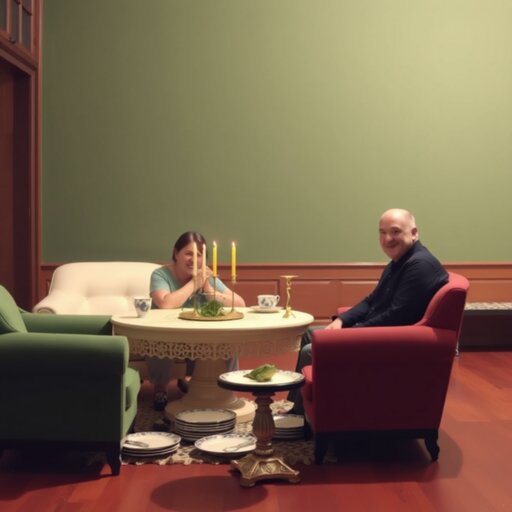} &
        \includegraphics[width=0.15\textwidth]{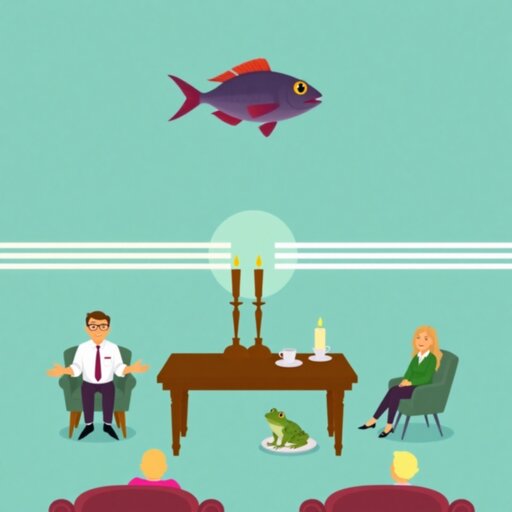} &
        \includegraphics[width=0.15\textwidth]{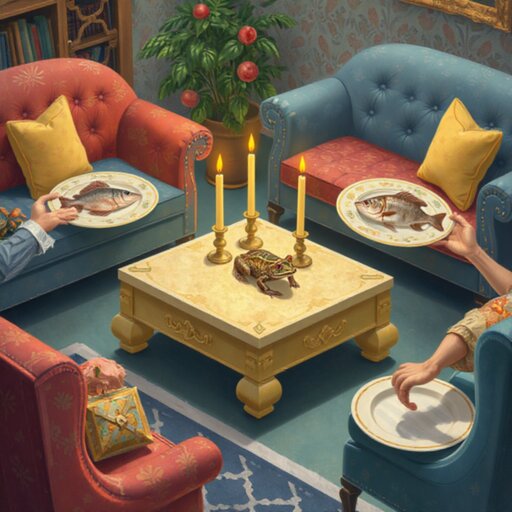} \\
        \midrule
        & \multicolumn{3}{c}{\textit{``Seven lamps.''}} & & & \multicolumn{3}{c}{\textit{``Seven desks.''}} \\
        \rotatebox{90}{\makecell{\normalsize{CoDe}~\cite{Singh:2025CoDE}\\Quantity}} & \includegraphics[width=0.15\textwidth]{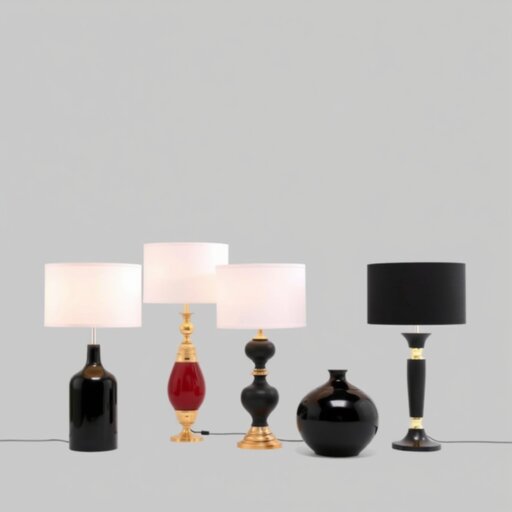} &
        \includegraphics[width=0.15\textwidth]{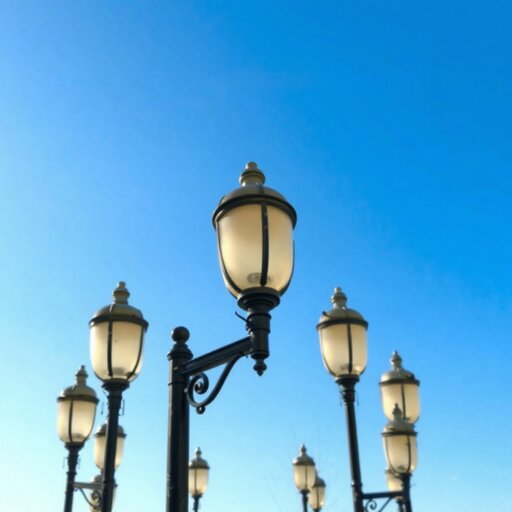} &
        \includegraphics[width=0.15\textwidth]{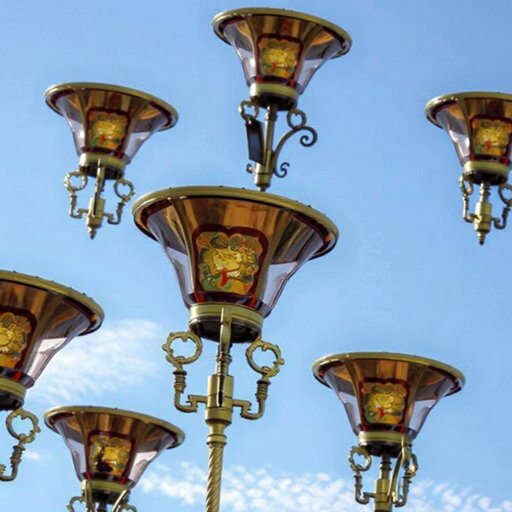} & & &
        \includegraphics[width=0.15\textwidth]{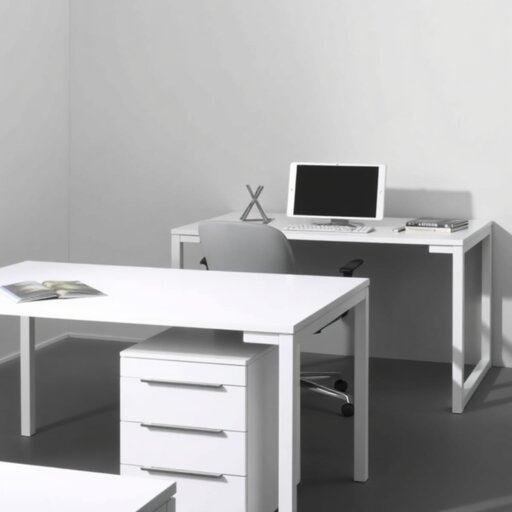} &
        \includegraphics[width=0.15\textwidth]{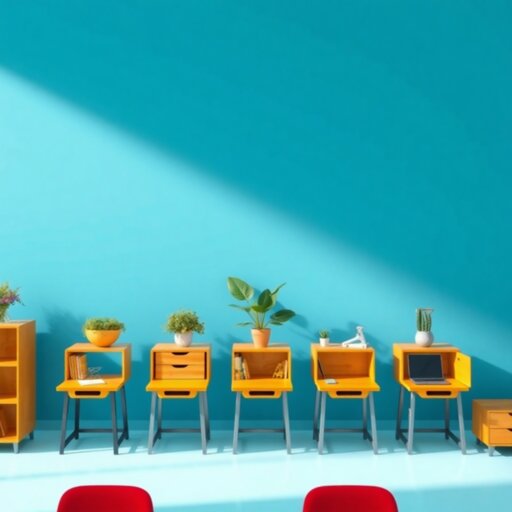} &
        \includegraphics[width=0.15\textwidth]{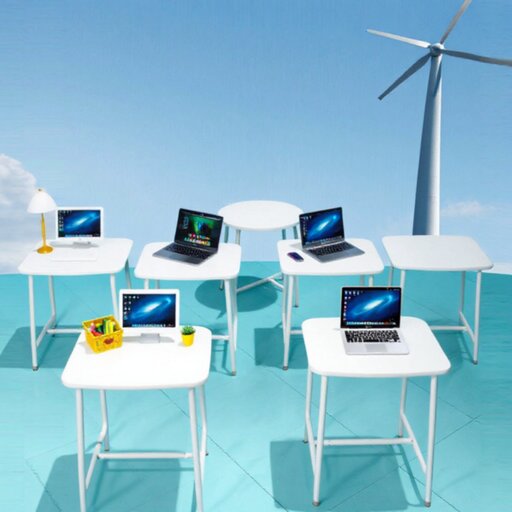} \\
        \midrule
        & \multicolumn{3}{c}{\textit{``In a collection of hats, each one is plain, but one is adorned with feathers.''}} & & & \multicolumn{3}{c}{\makecell{\textit{``A frog with a baseball cap is crouching on a lotus leaf,}\\ \textit{and another frog without a cap is crouching on a bigger lotus leaf.''}}} \\
        \rotatebox{90}{\makecell{\normalsize{CoDe}~\cite{Singh:2025CoDE}\\Composition}} & \includegraphics[width=0.15\textwidth]{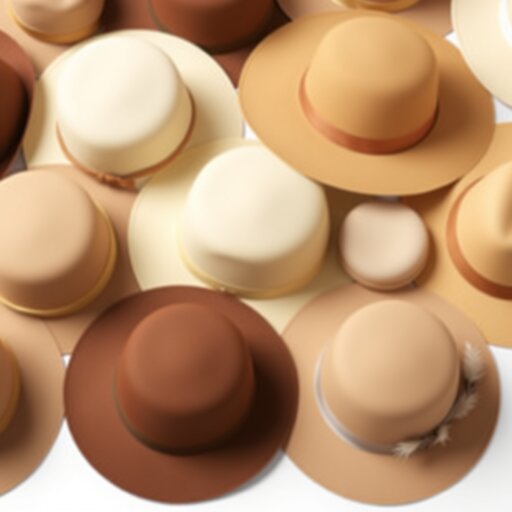} &
        \includegraphics[width=0.15\textwidth]{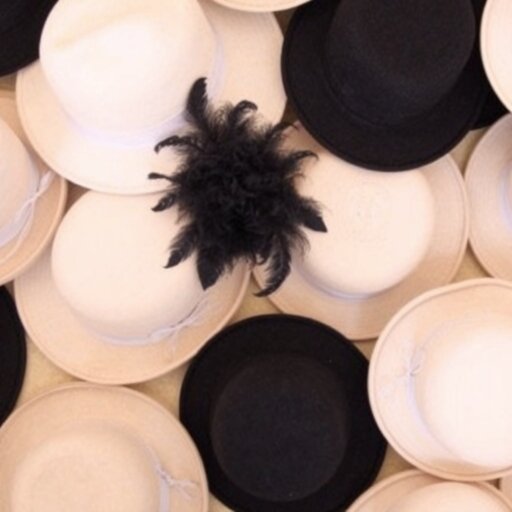} &
        \includegraphics[width=0.15\textwidth]{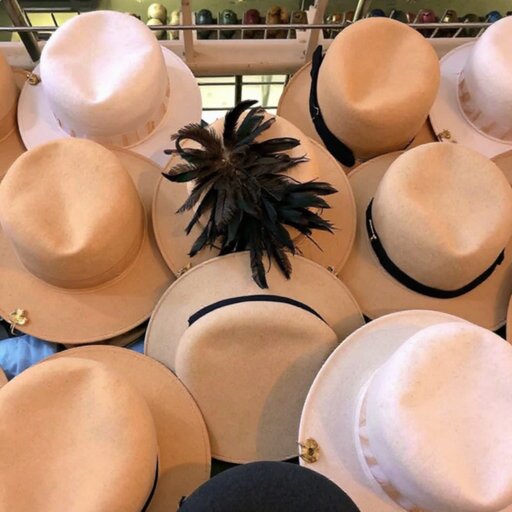} & & &
        \includegraphics[width=0.15\textwidth]{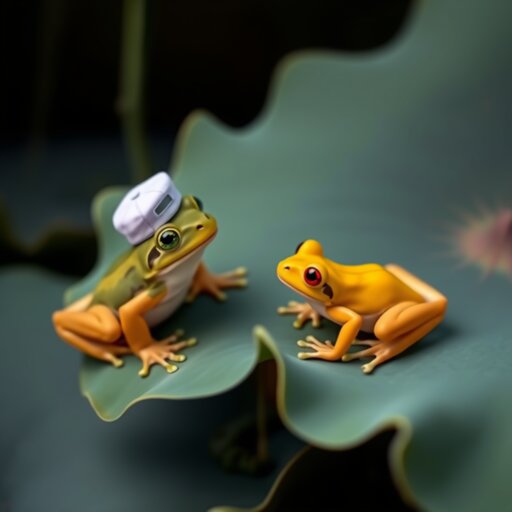} &
        \includegraphics[width=0.15\textwidth]{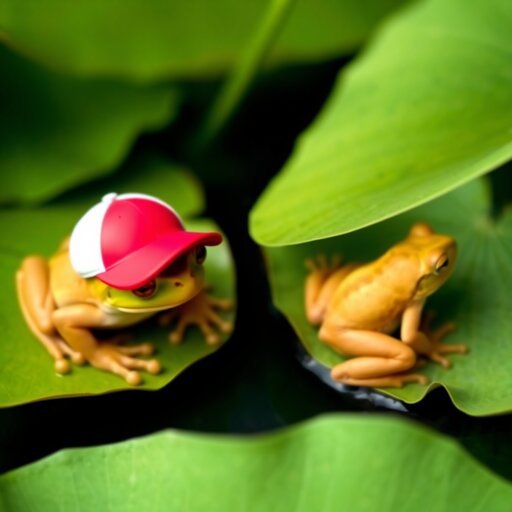} &
        \includegraphics[width=0.15\textwidth]{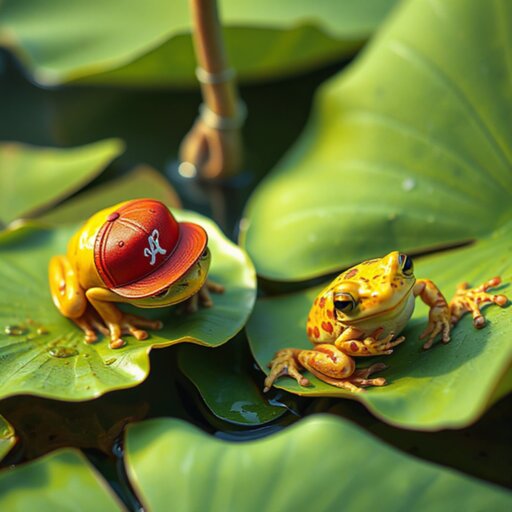} \\
        
        \bottomrule
        
    \end{tabularx}
  \caption{\textbf{Additional qualitative results of inference-time SDE conversion and interpolant conversion.} }
  \label{fig:ode-sde-vpsde-appendix1}
}
\end{figure}

\begin{figure}[ht!]
{\tiny
\setlength{\tabcolsep}{0.0em}
\def\arraystretch{0.0}
\newcolumntype{X}{>{\centering\arraybackslash}m{0.005\textwidth}}
\newcolumntype{Y}{>{\centering\arraybackslash}m{0.05\textwidth}}
\newcolumntype{Z}{>{\centering\arraybackslash}m{0.155\textwidth}}
    \begin{tabularx}{\textwidth}{Y | Z Z Z X | X Z Z Z}
        \toprule
        & \scriptsize{Linear-ODE} & \scriptsize{Linear-SDE} & \scriptsize{VP-SDE} & & & \scriptsize{Linear-ODE} & \scriptsize{Linear-SDE} & \scriptsize{VP-SDE} \\
        \midrule
        & \multicolumn{3}{c}{\textit{``Six bottles.''}} & & & \multicolumn{3}{c}{\textit{``Five hamburgers sizzled on the grill.''}} \\
        \rotatebox{90}{\makecell{\normalsize{SVDD}~\cite{Li2024:SVDD}\\Quantity}} & \includegraphics[width=0.15\textwidth]{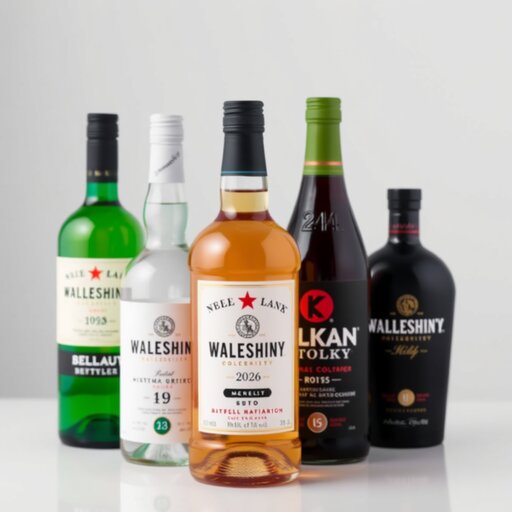} &
        \includegraphics[width=0.15\textwidth]{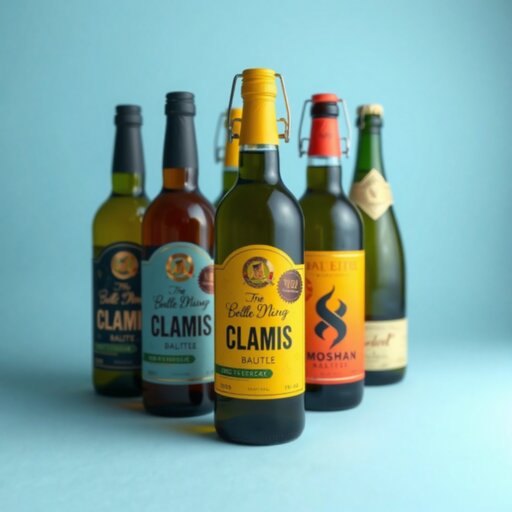} &
        \includegraphics[width=0.15\textwidth]{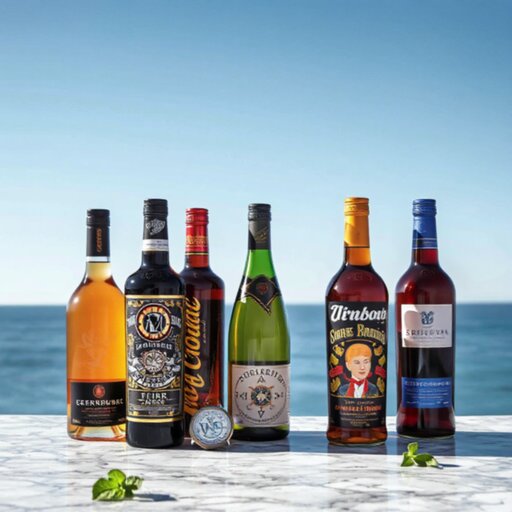} & & &
        \includegraphics[width=0.15\textwidth]{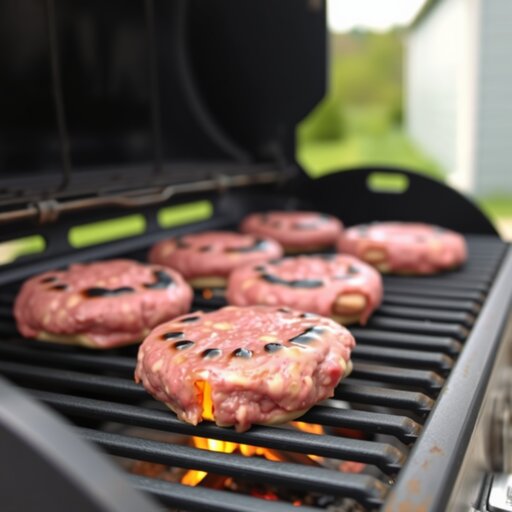} &
        \includegraphics[width=0.15\textwidth]{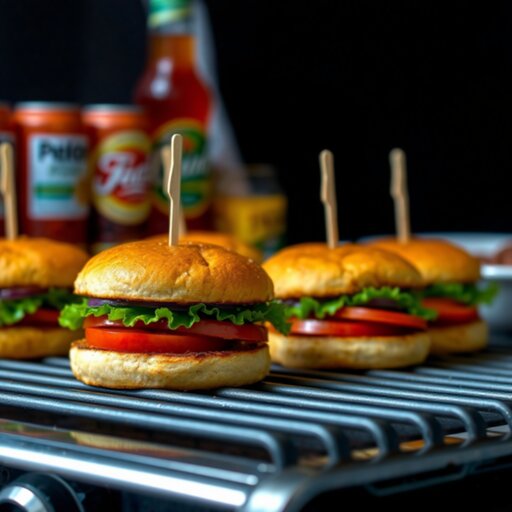} &
        \includegraphics[width=0.15\textwidth]{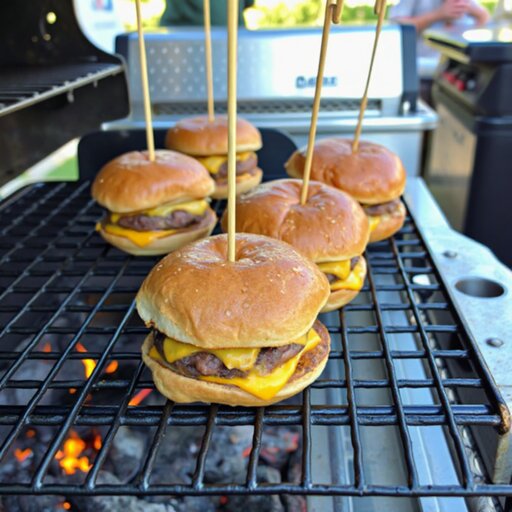} \\
        \midrule
        & \multicolumn{3}{c}{\textit{``Two men, four vases, four chickens and four ships.''}} & & & \multicolumn{3}{c}{\textit{``Six bicycles.''}} \\
        \rotatebox{90}{\makecell{\normalsize{SVDD}~\cite{Li2024:SVDD}\\Quantity}} & \includegraphics[width=0.15\textwidth]{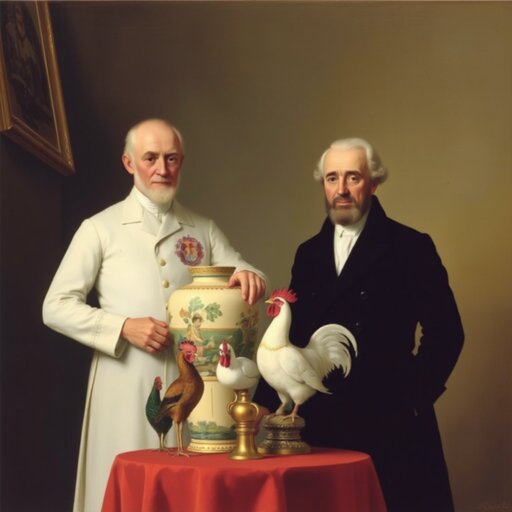} &
        \includegraphics[width=0.15\textwidth]{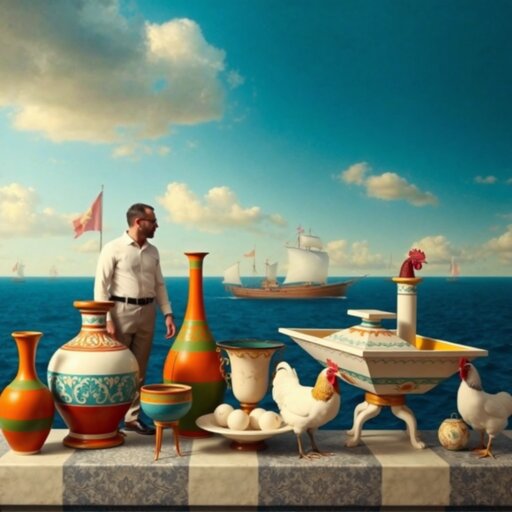} &
        \includegraphics[width=0.15\textwidth]{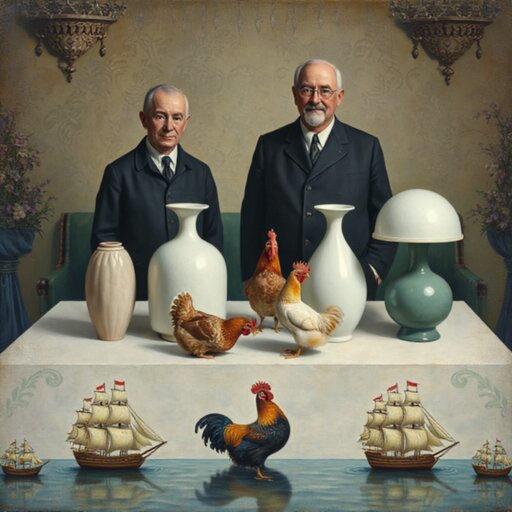} & & &
        \includegraphics[width=0.15\textwidth]{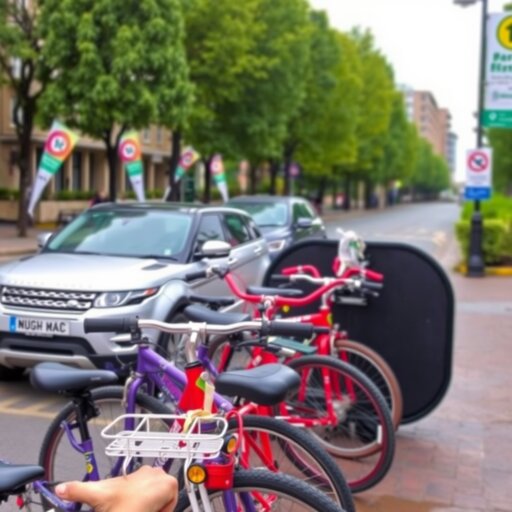} &
        \includegraphics[width=0.15\textwidth]{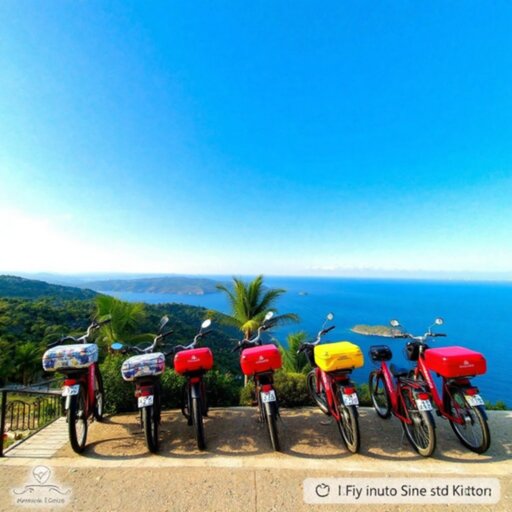} &
        \includegraphics[width=0.15\textwidth]{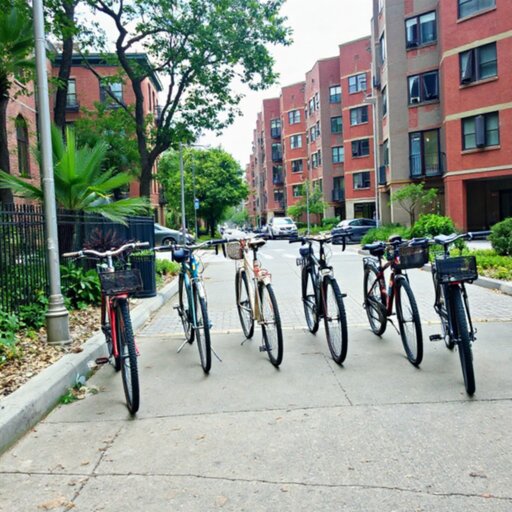} \\
        \midrule
        & \multicolumn{3}{c}{\makecell{\textit{``Two people and two bicycles in the street,}\\ \textit{the bicycle with the larger wheels belongs to the taller person.''}}} & & & \multicolumn{3}{c}{\makecell{\textit{``There are two cups on the table, the cup without coffee}\\ \textit{is on the left of the other filled with coffee.''}}} \\
        \rotatebox{90}{\makecell{\normalsize{SVDD}~\cite{Li2024:SVDD}\\Composition}} & \includegraphics[width=0.15\textwidth]{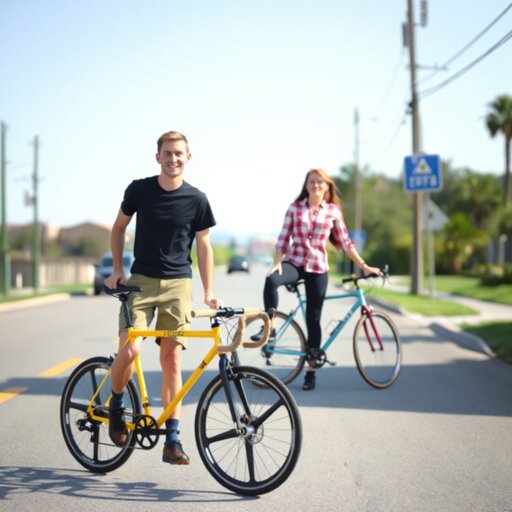} &
        \includegraphics[width=0.15\textwidth]{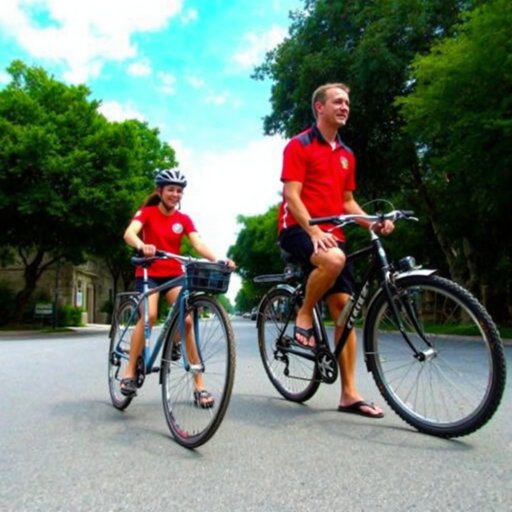} &
        \includegraphics[width=0.15\textwidth]{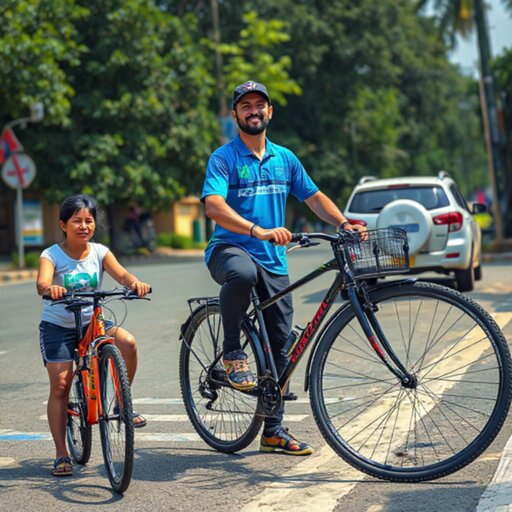} & & &
        \includegraphics[width=0.15\textwidth]{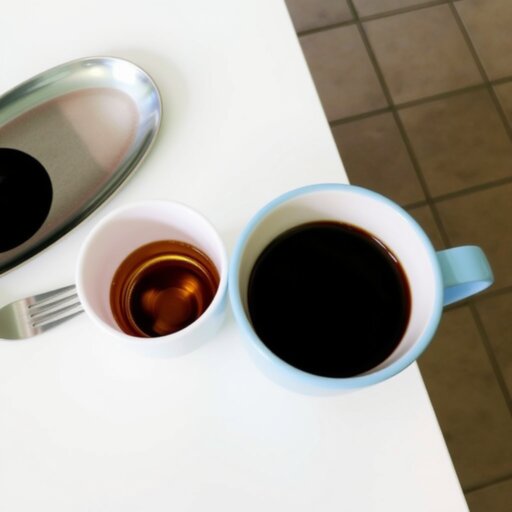} &
        \includegraphics[width=0.15\textwidth]{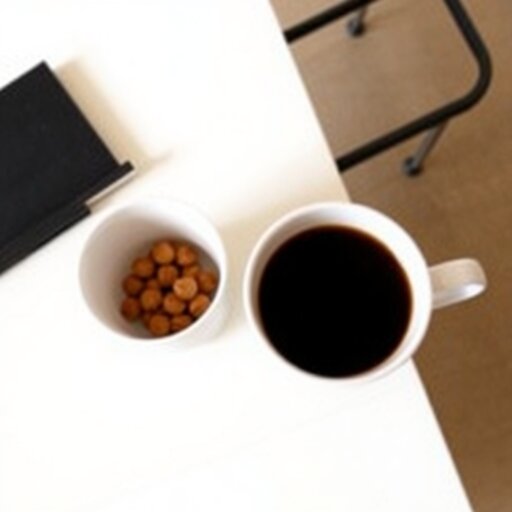} &
        \includegraphics[width=0.15\textwidth]{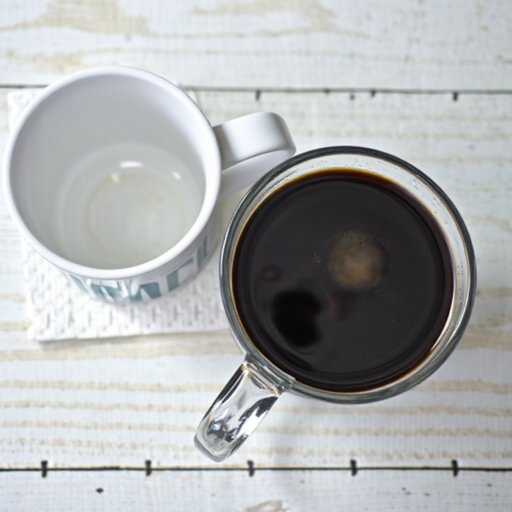} \\
        
        \toprule
        & \scriptsize{Linear-ODE} & \scriptsize{Linear-SDE} & \scriptsize{VP-SDE} & & & \scriptsize{Linear-ODE} & \scriptsize{Linear-SDE} & \scriptsize{VP-SDE} \\
        \midrule
        
        & \multicolumn{3}{c}{\makecell{\textit{``Two giraffes, three eggs, two breads, }\\\textit{three microwaves and four strawberries.''}}} & & & \multicolumn{3}{c}{\textit{``Four pears, four desks, three paddles and two rabbits.''}} \\
        \rotatebox{90}{\makecell{\normalsize{ {\Oursbf{}} (Ours) }\\Quantity}} &
        \includegraphics[width=0.15\textwidth]{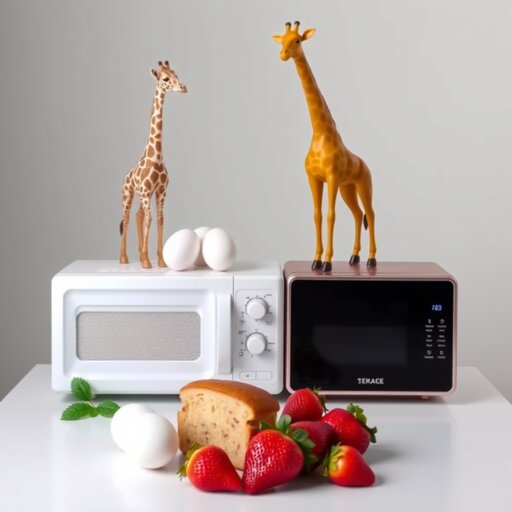}  &
        \includegraphics[width=0.15\textwidth]{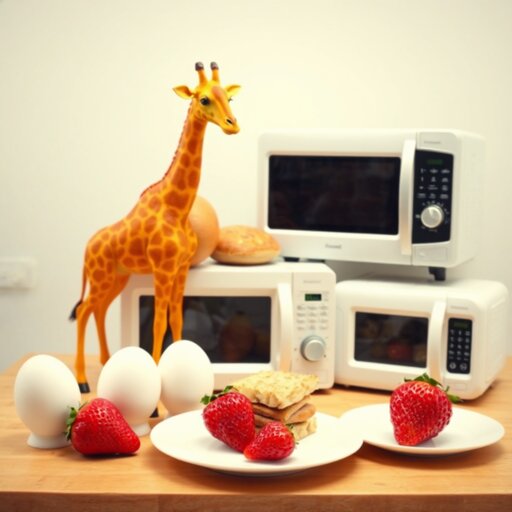} &
        \includegraphics[width=0.15\textwidth]{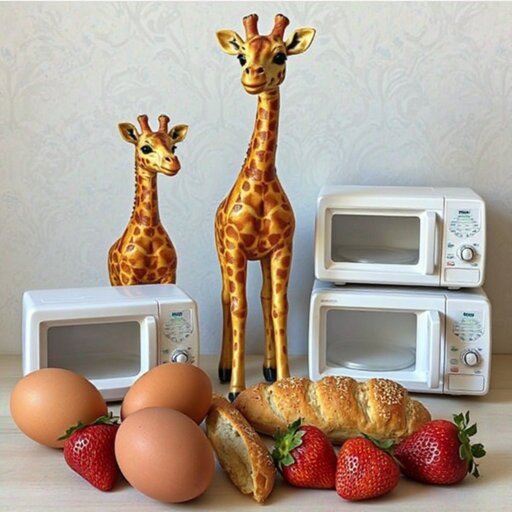} & & &
        \includegraphics[width=0.15\textwidth]{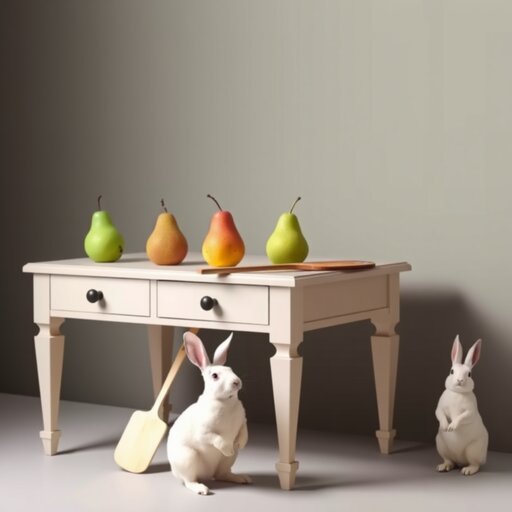}  &
        \includegraphics[width=0.15\textwidth]{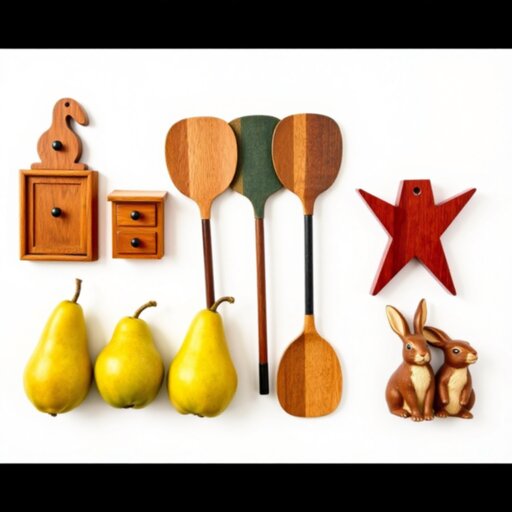} &
        \includegraphics[width=0.15\textwidth]{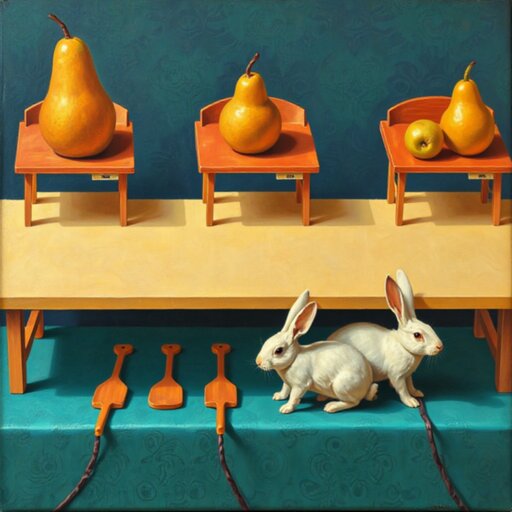} \\
        \midrule
        & \multicolumn{3}{c}{\textit{``Four balloons, one cup, four desks, two dogs and four microwaves.''}} & & & \multicolumn{3}{c}{\textit{``Seven women.''}} \\
        \rotatebox{90}{\makecell{\normalsize{ {\Oursbf{}} (Ours) }\\Quantity}} &
        \includegraphics[width=0.15\textwidth]{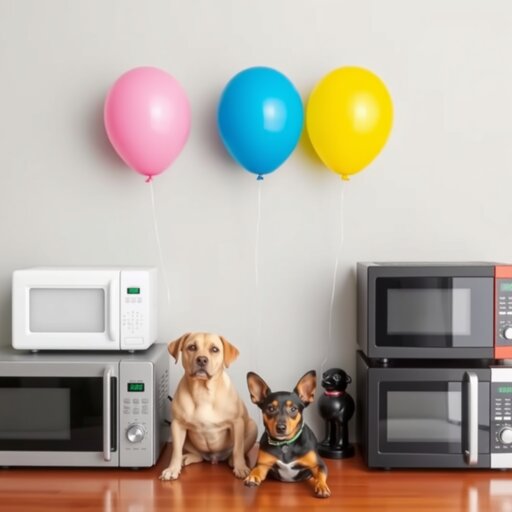} &
        \includegraphics[width=0.15\textwidth]{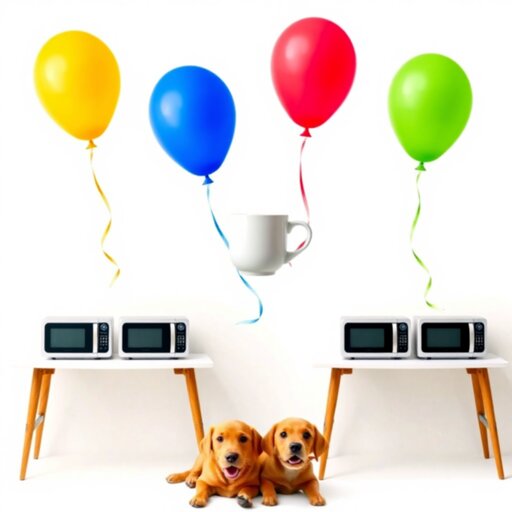} &
        \includegraphics[width=0.15\textwidth]{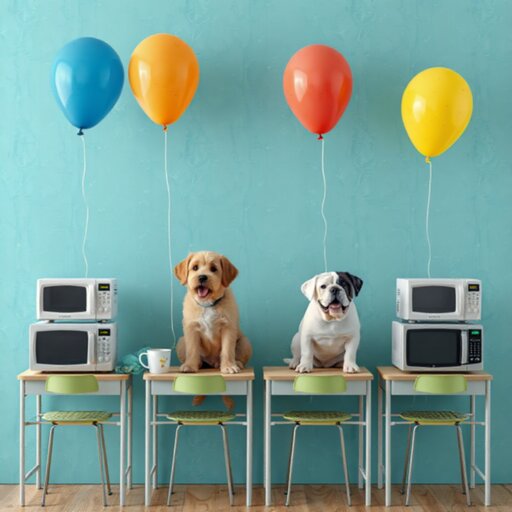} & & &
        \includegraphics[width=0.15\textwidth]{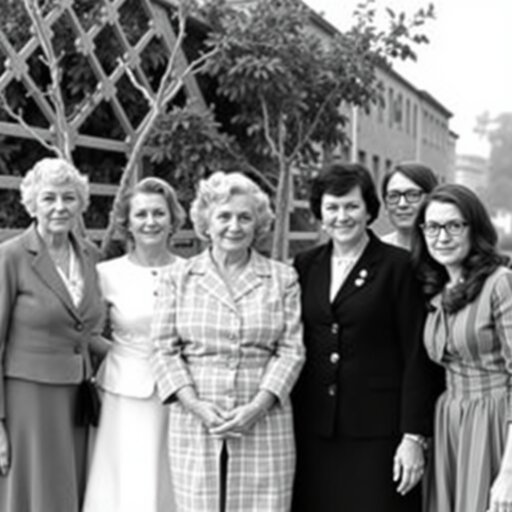}  &
        \includegraphics[width=0.15\textwidth]{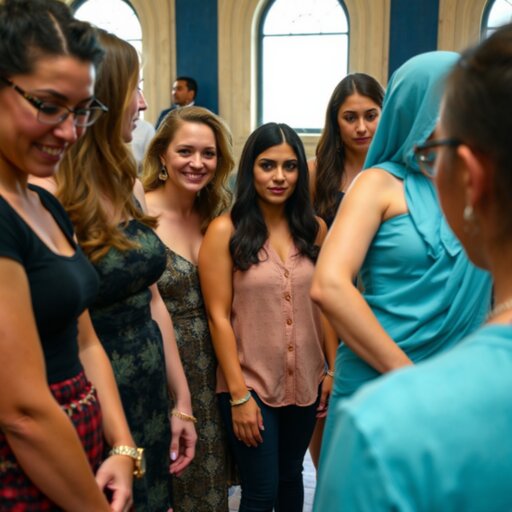} &
        \includegraphics[width=0.15\textwidth]{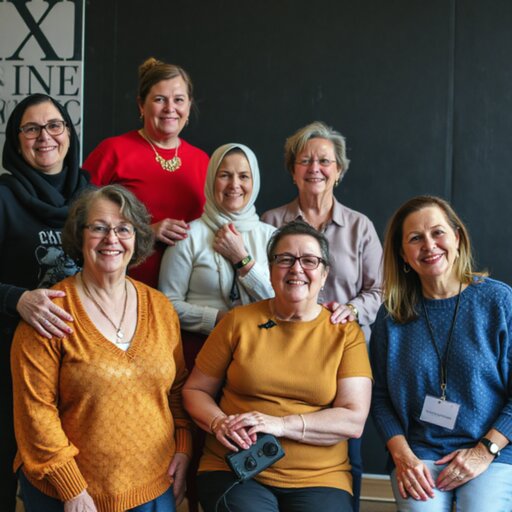} \\
        \midrule
        & \multicolumn{3}{c}{\textit{``Three sailboats on the water, each with sails of a different color.''}} & & & \multicolumn{3}{c}{\makecell{\textit{``Two birds are chasing each other in the air, with the one}\\ \textit{flying higher having a long tail and the other bird having a short tail.''}}} \\
        \rotatebox{90}{\makecell{\normalsize{ {\Oursbf{}} (Ours) }\\Composition}} &
        \includegraphics[width=0.15\textwidth]{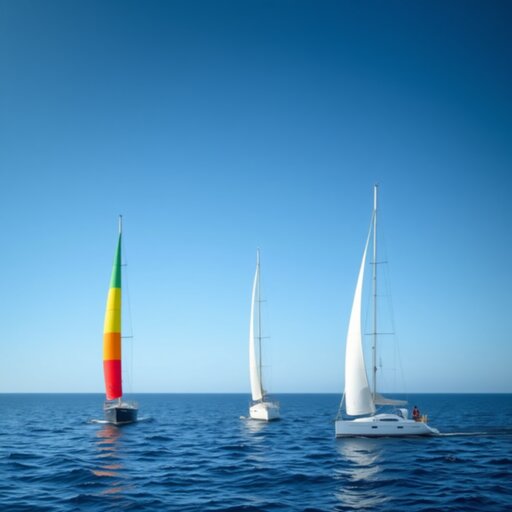} & 
        \includegraphics[width=0.15\textwidth]{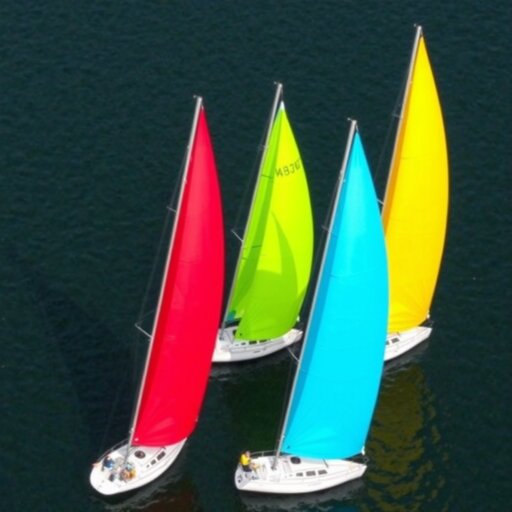} &
        \includegraphics[width=0.15\textwidth]{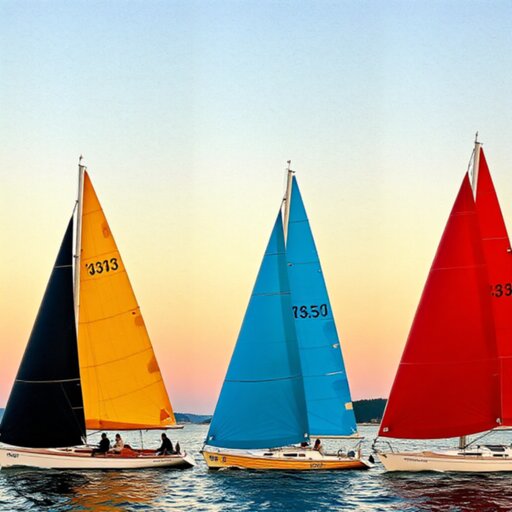} & & &
        \includegraphics[width=0.15\textwidth]{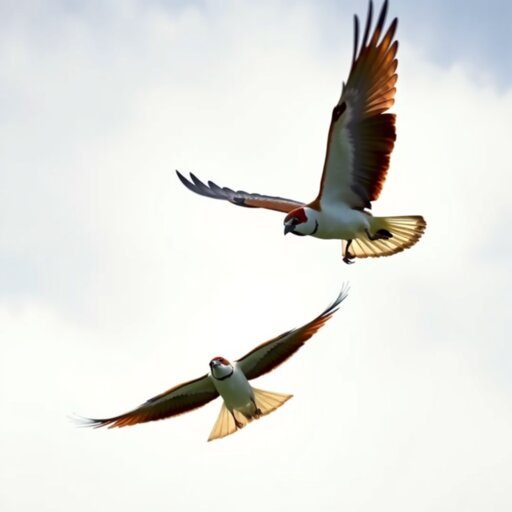} & 
        \includegraphics[width=0.15\textwidth]{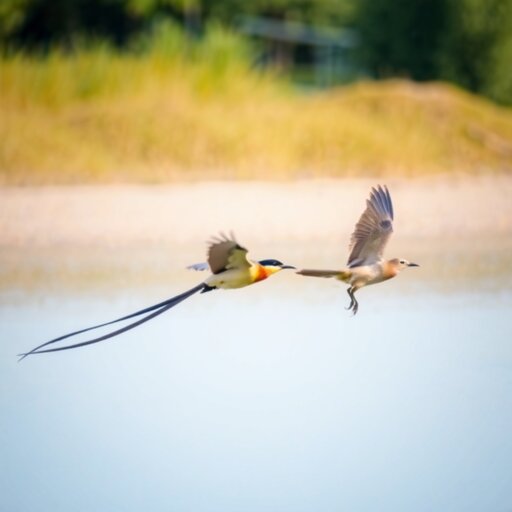} &
        \includegraphics[width=0.15\textwidth]{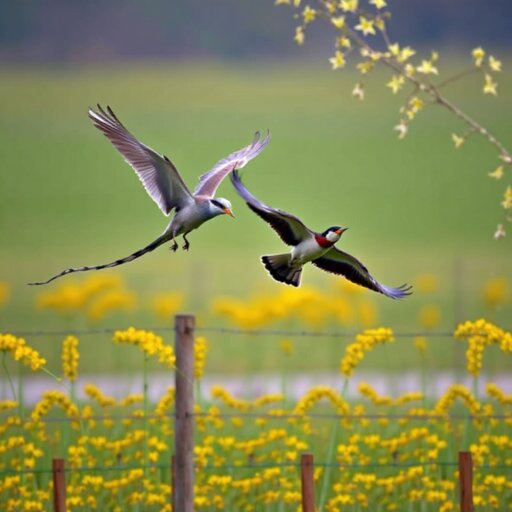} \\
        
        \bottomrule
        
    \end{tabularx}
  \caption{\textbf{Additional qualitative results of inference-time SDE conversion and interpolant conversion.} }
  \label{fig:ode-sde-vpsde-appendix2}
}
\end{figure}

\clearpage
\newpage
\subsection{Comparisons of Inference-Time Scaling}
\label{sec:appendix_quali_search_alg}
\begin{figure*}[h!]
{\small
\setlength{\tabcolsep}{0.0em}
\def\arraystretch{0.0}
\newcolumntype{Z}{>{\centering\arraybackslash}m{0.166\textwidth}}
    \begin{tabularx}{\textwidth}{Z Z Z Z Z Z}
        \toprule
        BoN & SoP~\cite{Ma2025:SoP} & SMC~\cite{Kim:2025DAS} & CoDe~\cite{Singh:2025CoDE} & SVDD~\cite{Li2024:SVDD} & {\Oursbf{}} (Ours) \\
        \midrule
        \multicolumn{6}{c}{\textit{``In a room, all the chairs are occupied except one.''}}\\
        \includegraphics[width=0.16\textwidth]{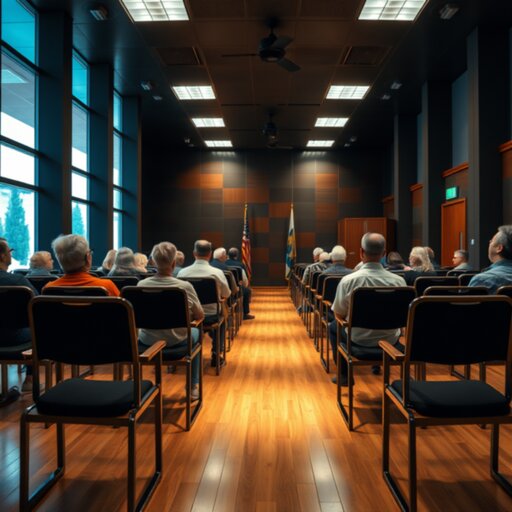} &
        \includegraphics[width=0.16\textwidth]{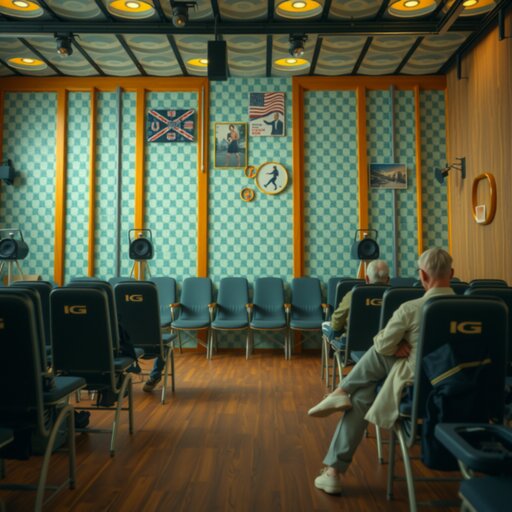} &
        \includegraphics[width=0.16\textwidth]{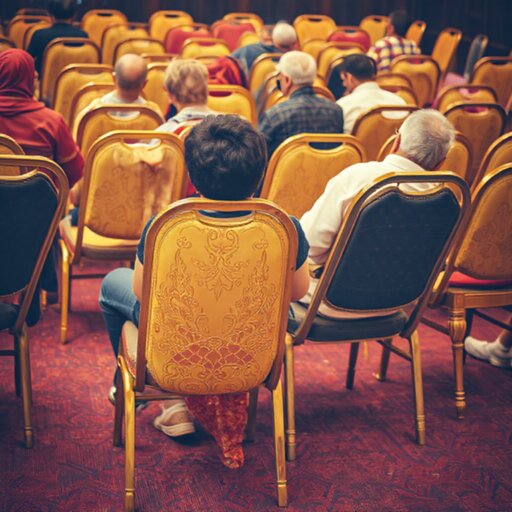} &
        \includegraphics[width=0.16\textwidth]{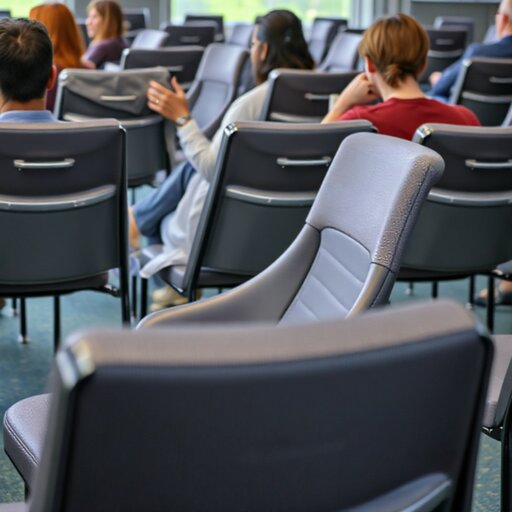} &
        \includegraphics[width=0.16\textwidth]{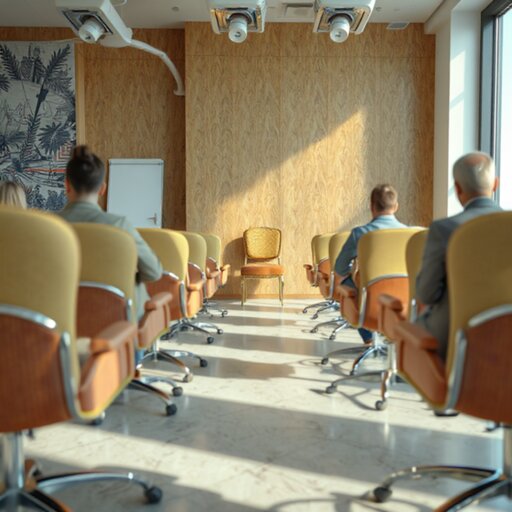} &
        \includegraphics[width=0.16\textwidth]{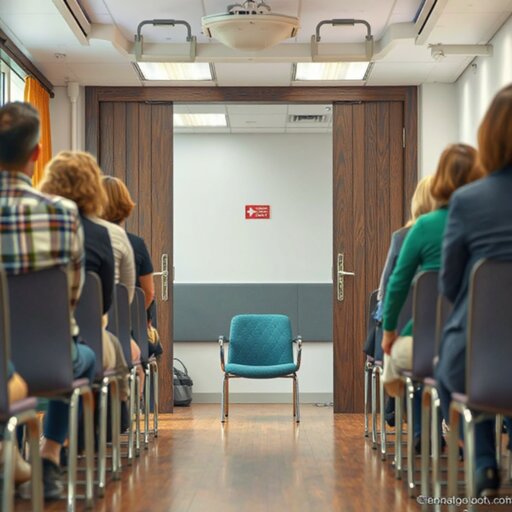} \\
        \midrule
        \multicolumn{6}{c}{\textit{\makecell{``Three mugs are placed side by side;\\the two closest to the faucet each contain a toothbrush, while the one furthest away is empty.''}}} \\
        \includegraphics[width=0.16\textwidth]{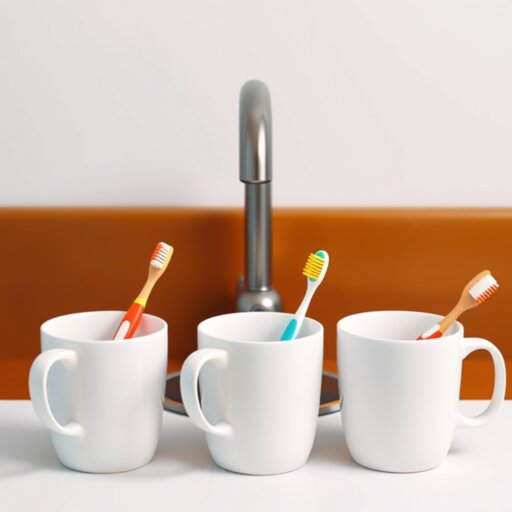} &
        \includegraphics[width=0.16\textwidth]{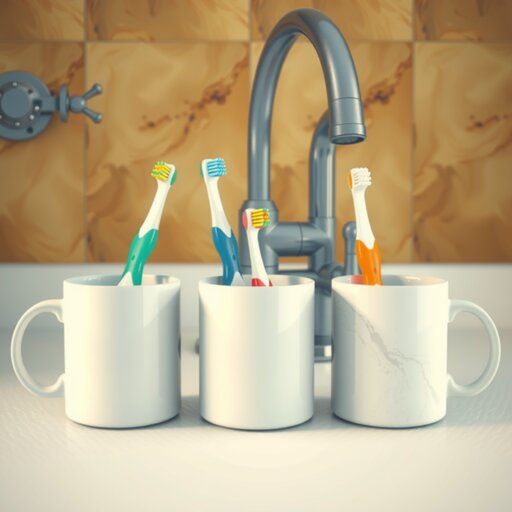} &
        \includegraphics[width=0.16\textwidth]{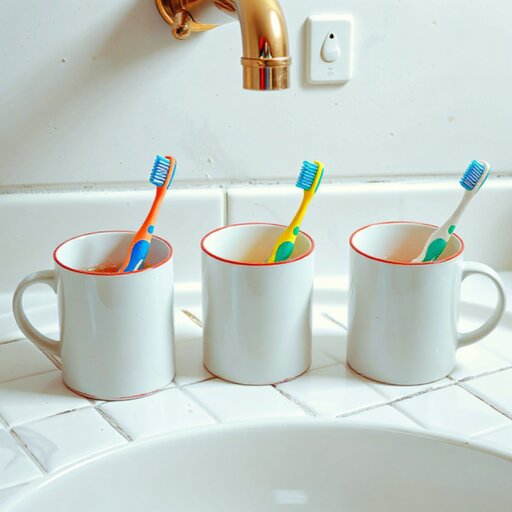} &
        \includegraphics[width=0.16\textwidth]{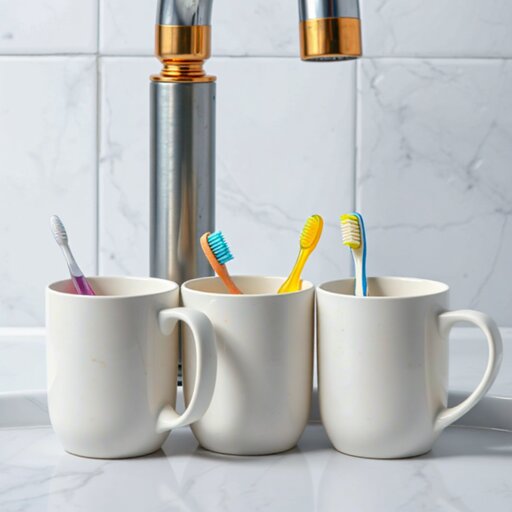} &
        \includegraphics[width=0.16\textwidth]{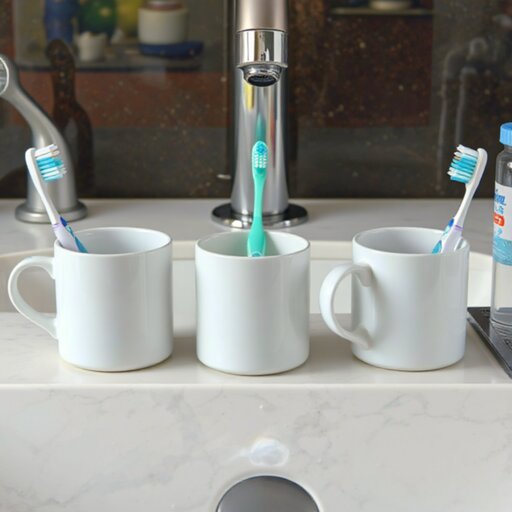} &
        \includegraphics[width=0.16\textwidth]{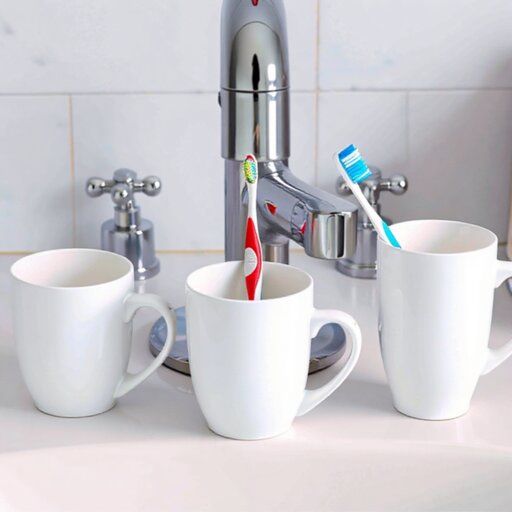} \\
        \midrule
        \multicolumn{6}{c}{\textit{``Three flowers in the meadow, with only the red rose blooming; the others are not open.''}}\\
        \includegraphics[width=0.16\textwidth]{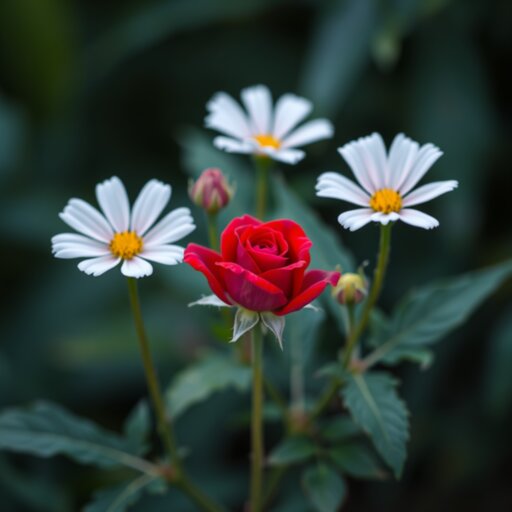} &
        \includegraphics[width=0.16\textwidth]{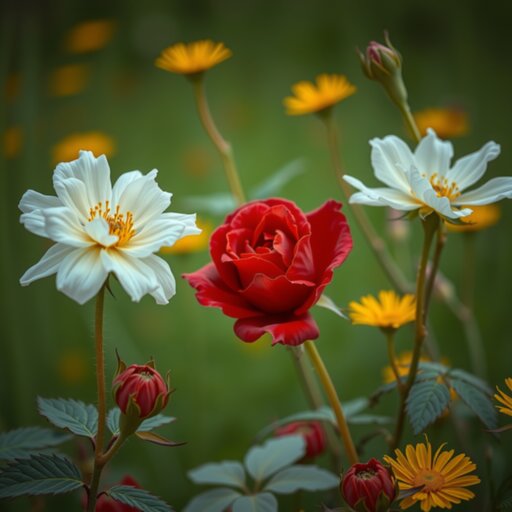} &
        \includegraphics[width=0.16\textwidth]{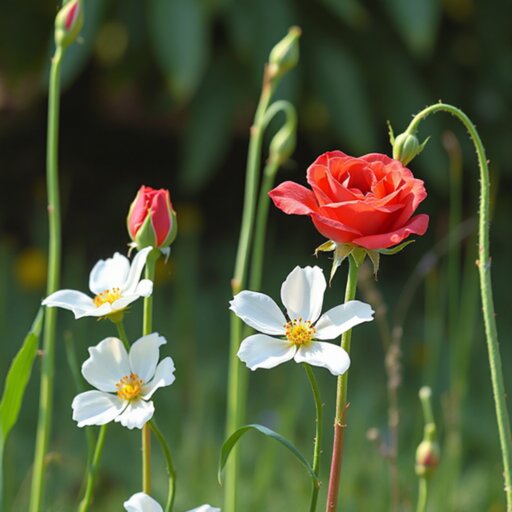} &
        \includegraphics[width=0.16\textwidth]{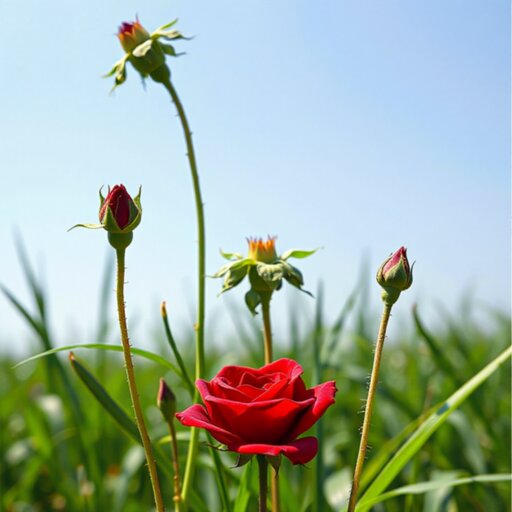} &
        \includegraphics[width=0.16\textwidth]{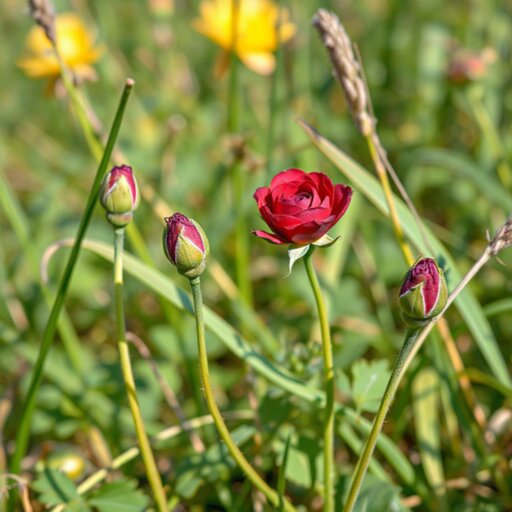} &
        \includegraphics[width=0.16\textwidth]{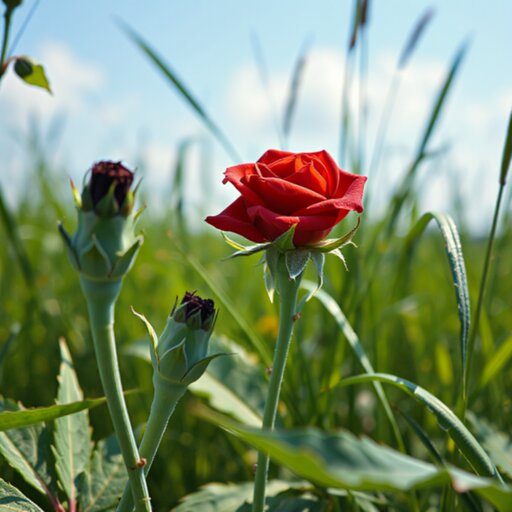} \\
        \midrule
        \multicolumn{6}{c}{\textit{``In a pack of wolves, each one howls at the moon, but one remains silent.''}}\\
         \includegraphics[width=0.16\textwidth]{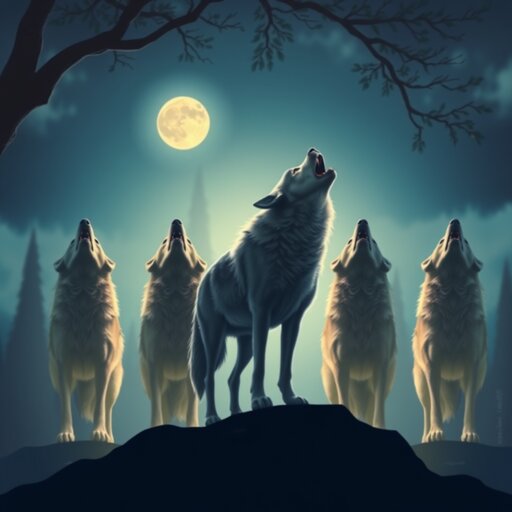} &
        \includegraphics[width=0.16\textwidth]{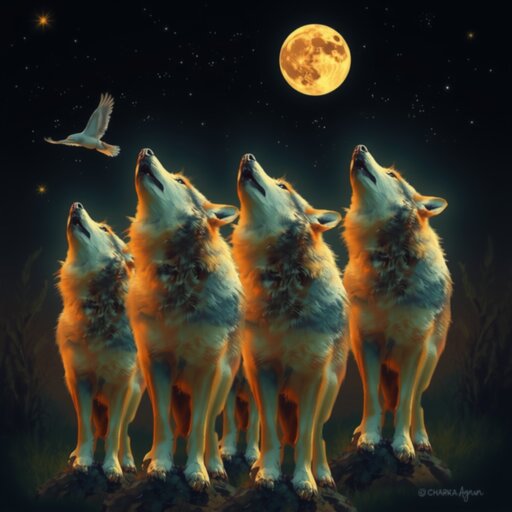} &
        \includegraphics[width=0.16\textwidth]{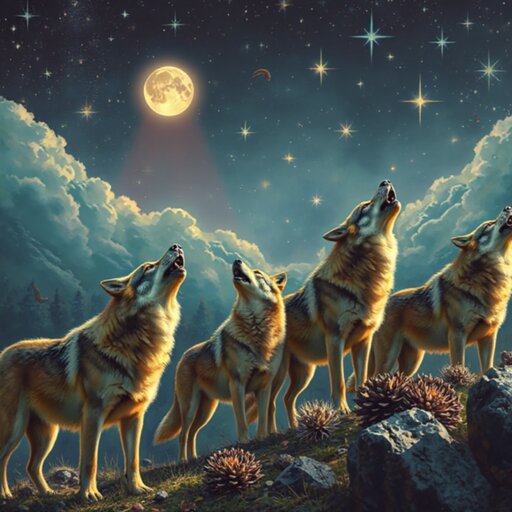} &
        \includegraphics[width=0.16\textwidth]{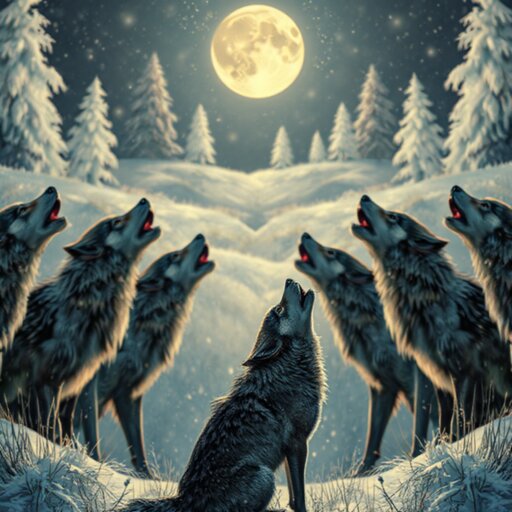} &
        \includegraphics[width=0.16\textwidth]{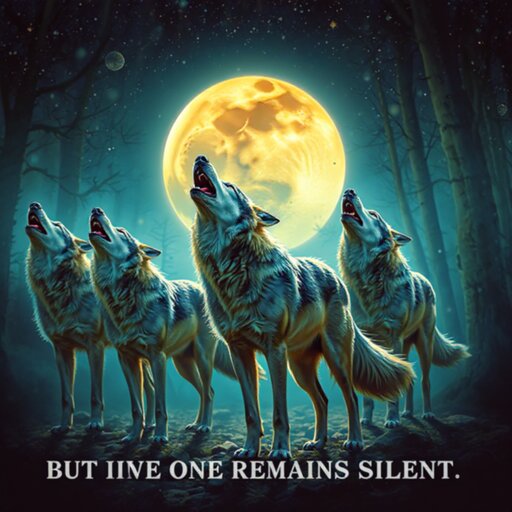} &
        \includegraphics[width=0.16\textwidth]{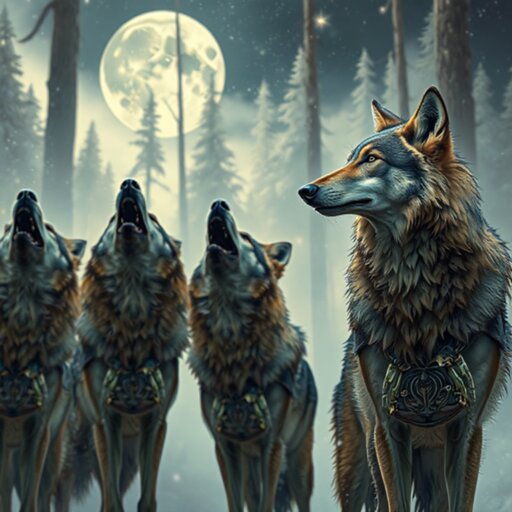} \\
        \midrule
        \multicolumn{6}{c}{\makecell{\textit{``An open biscuit tin contains three biscuits,}\\\textit{one without sultanas is square-shaped and the other two are round-shaped.''}}}\\
        \includegraphics[width=0.16\textwidth]{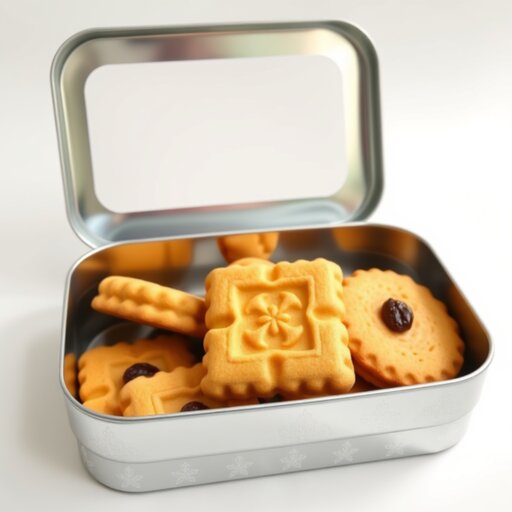} &
        \includegraphics[width=0.16\textwidth]{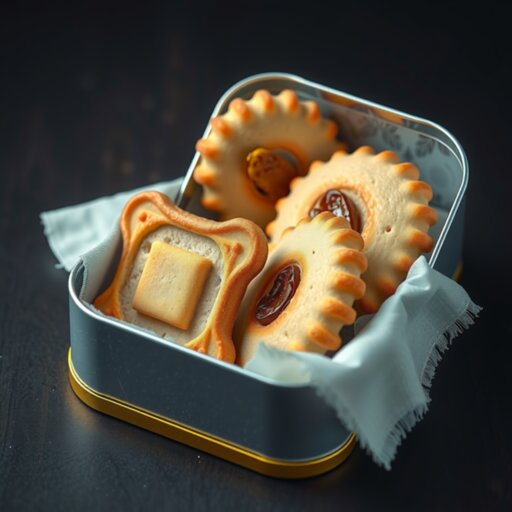} &
        \includegraphics[width=0.16\textwidth]{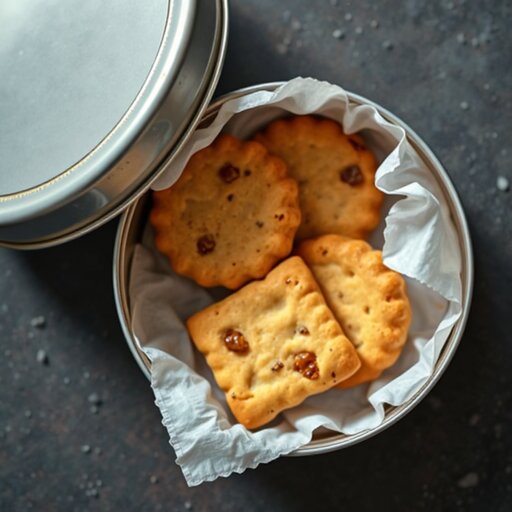} &
        \includegraphics[width=0.16\textwidth]{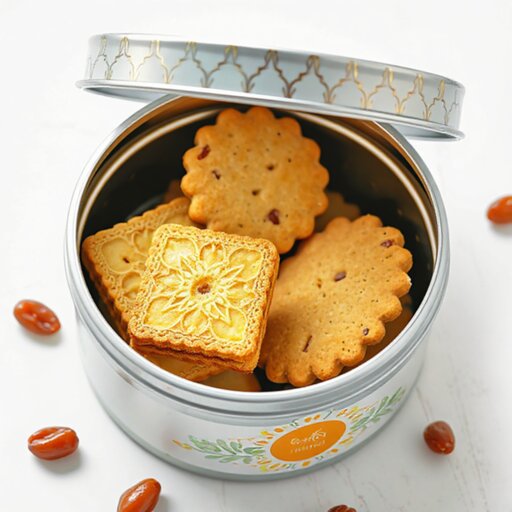} &
        \includegraphics[width=0.16\textwidth]{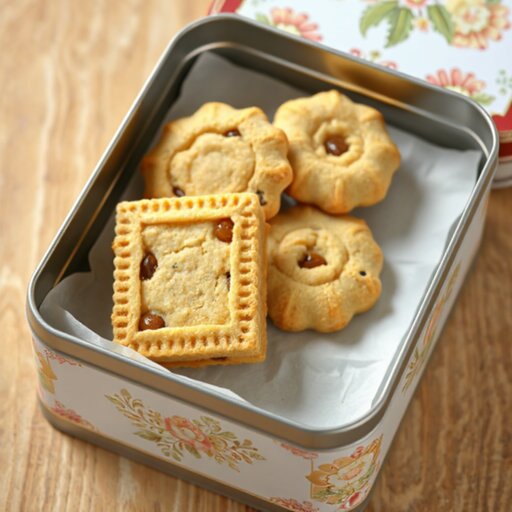} &
        \includegraphics[width=0.16\textwidth]{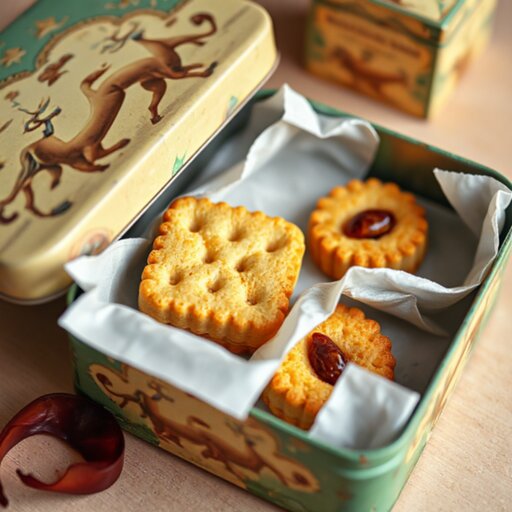} \\
        \midrule
        \multicolumn{6}{c}{\textit{``A rose that is not fully bloomed is higher than a rose that is already in bloom.''}}\\
        \includegraphics[width=0.16\textwidth]{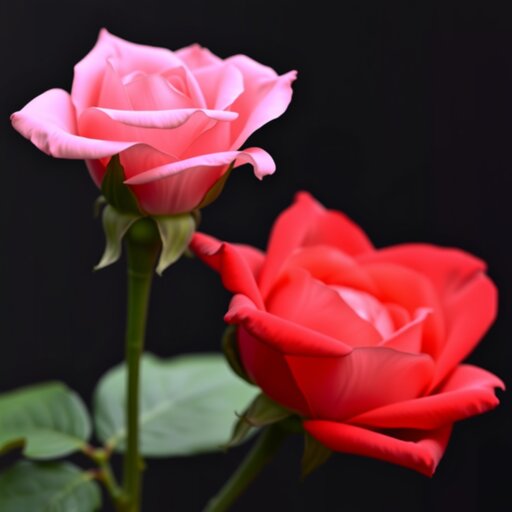} &
        \includegraphics[width=0.16\textwidth]{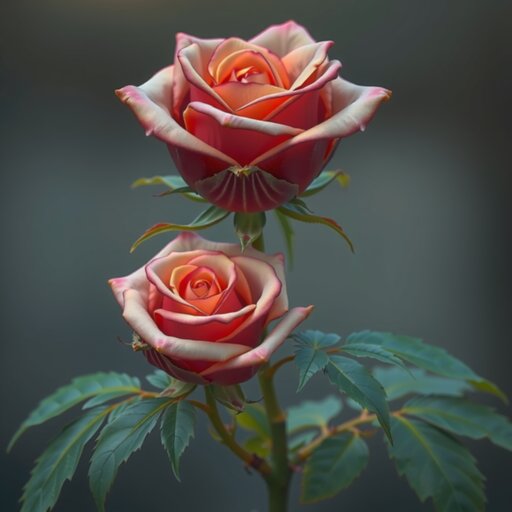} &
        \includegraphics[width=0.16\textwidth]{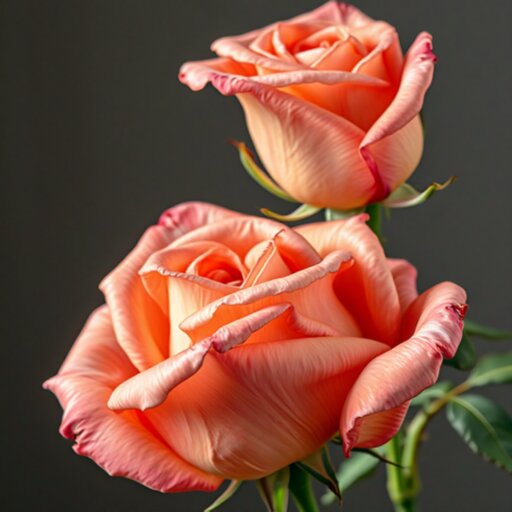} &
        \includegraphics[width=0.16\textwidth]{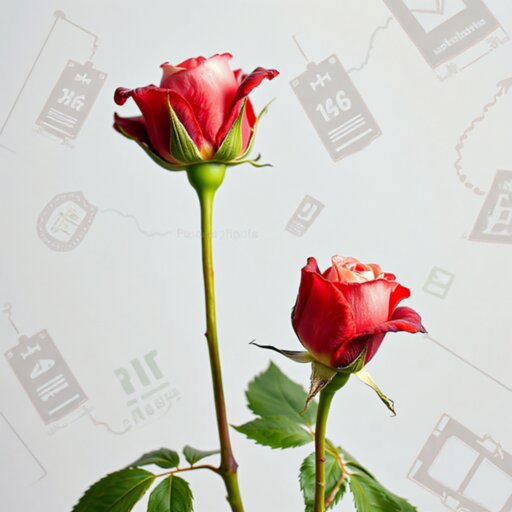} &
        \includegraphics[width=0.16\textwidth]{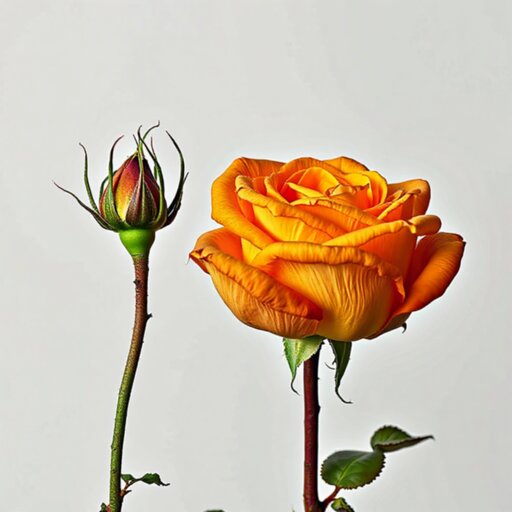} &
        \includegraphics[width=0.16\textwidth]{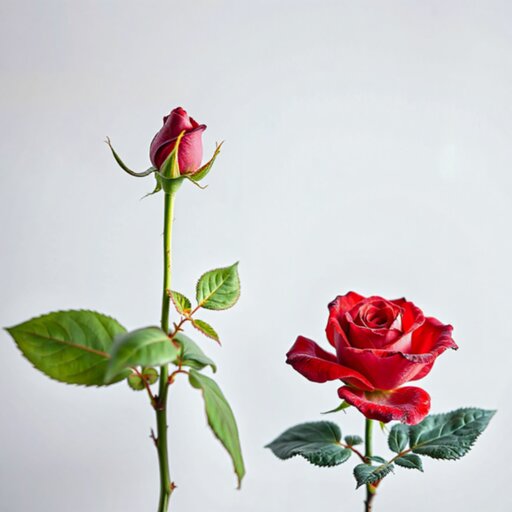} \\
        \midrule
        \multicolumn{6}{c}{\makecell{\textit{``There are two colors of pots in the flower garden;}\\\textit{all green pots have tulips in them and all yellow pots have no flowers in them.''}}}\\
        \includegraphics[width=0.16\textwidth]{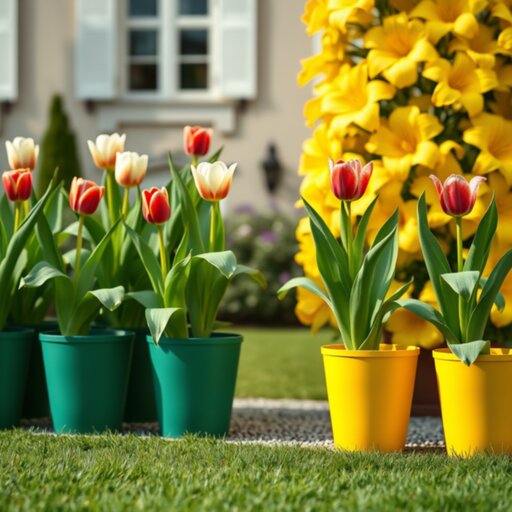} &
        \includegraphics[width=0.16\textwidth]{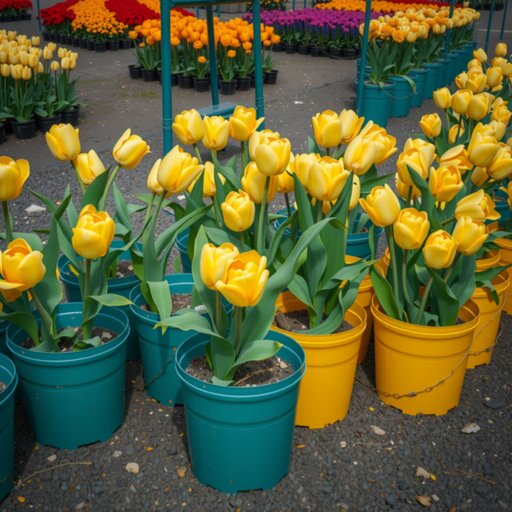} &
        \includegraphics[width=0.16\textwidth]{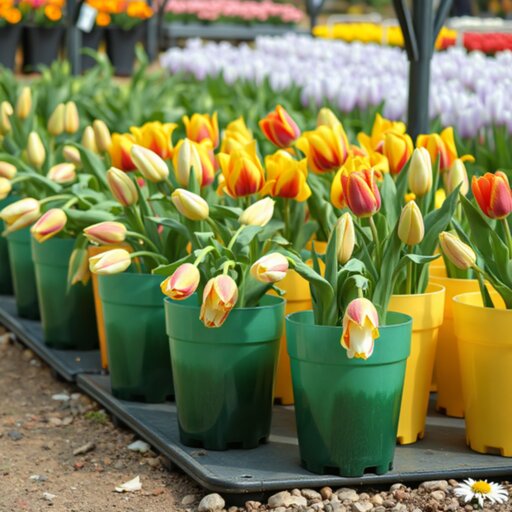} &
        \includegraphics[width=0.16\textwidth]{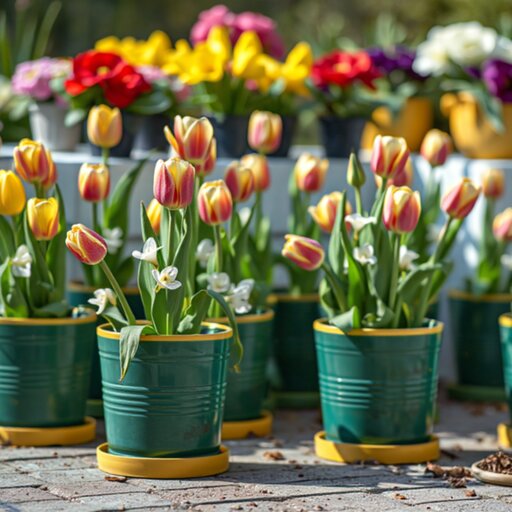} &
        \includegraphics[width=0.16\textwidth]{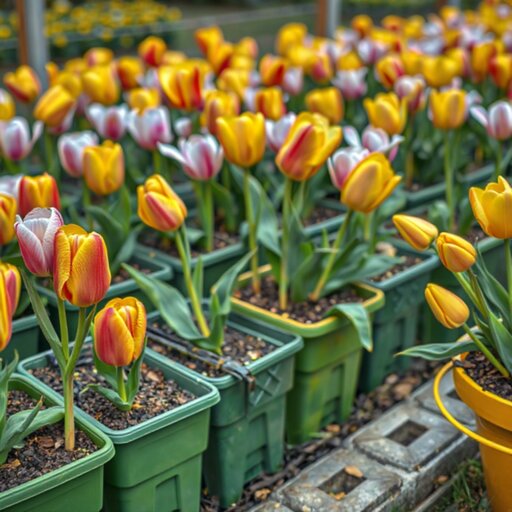} &
        \includegraphics[width=0.16\textwidth]{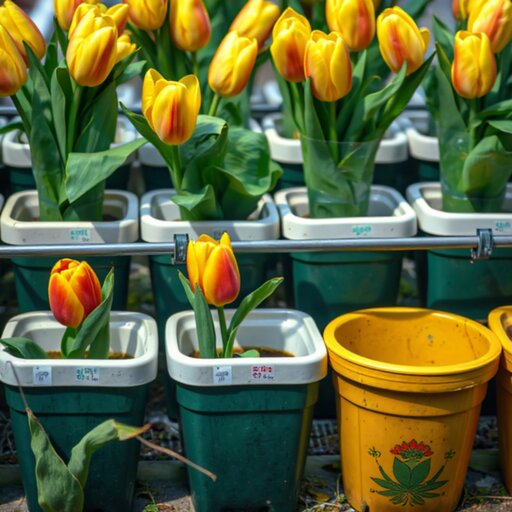} \\
        \bottomrule
    \end{tabularx}
  \caption{\textbf{Additional qualitative results of compositional text-to-image generation task}.}
  \label{fig:composition_methods_vp_appx}
}
\end{figure*}
\begin{figure*}[t!]
{\small
\setlength{\tabcolsep}{0.0em}
\def\arraystretch{0.0}
\newcolumntype{Z}{>{\centering\arraybackslash}m{0.166\textwidth}}
    \begin{tabularx}{\textwidth}{Z Z Z Z Z Z}
        \toprule
        BoN & SoP~\cite{Ma2025:SoP} & SMC~\cite{Kim:2025DAS} & CoDe~\cite{Singh:2025CoDE} & SVDD~\cite{Li2024:SVDD} & {\Oursbf{}} (Ours) \\
        \midrule
        \multicolumn{6}{c}{\textit{``Seven balloons, four bears and four swans.''}} \\
        \includegraphics[width=0.16\textwidth]{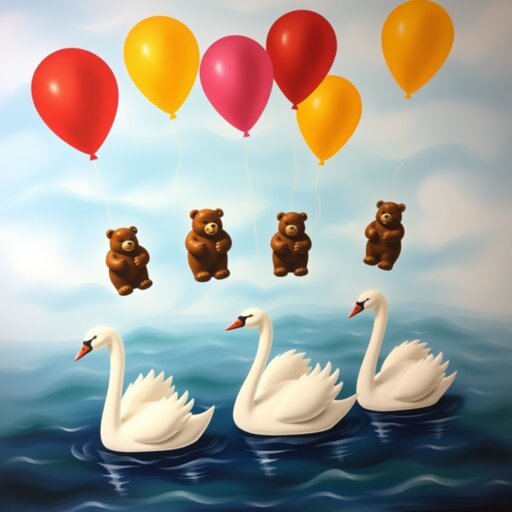} &
        \includegraphics[width=0.16\textwidth]{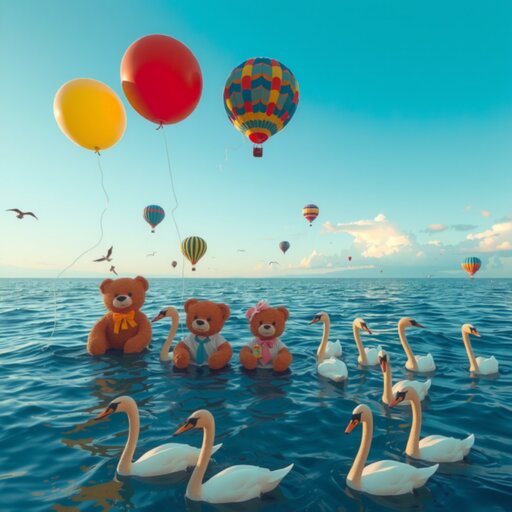} &
        \includegraphics[width=0.16\textwidth]{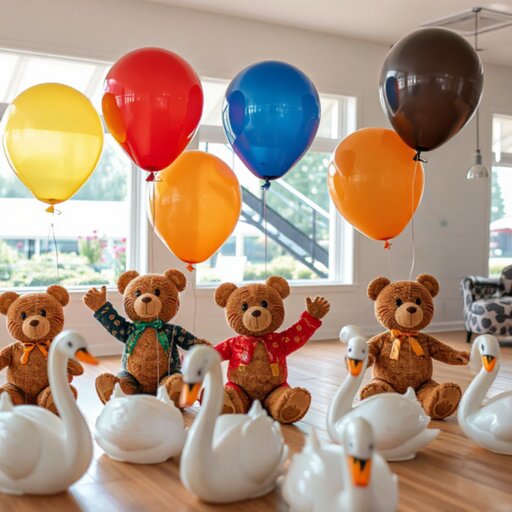} &
        \includegraphics[width=0.16\textwidth]{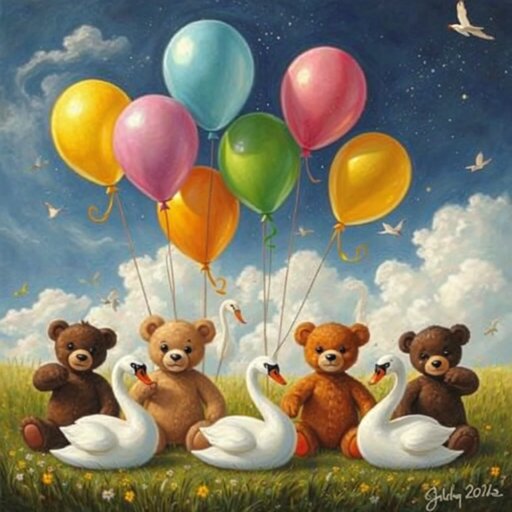} &
        \includegraphics[width=0.16\textwidth]{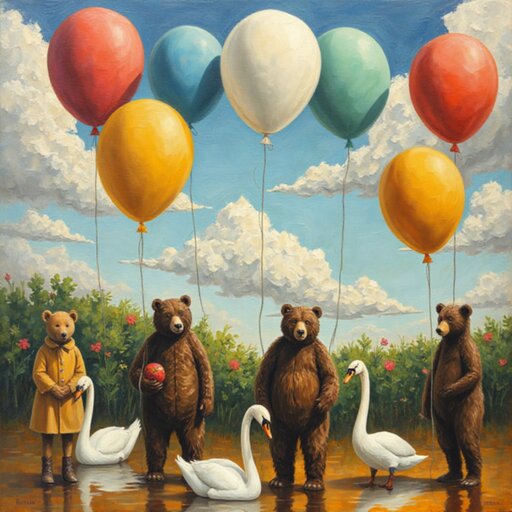} &
        \includegraphics[width=0.16\textwidth]{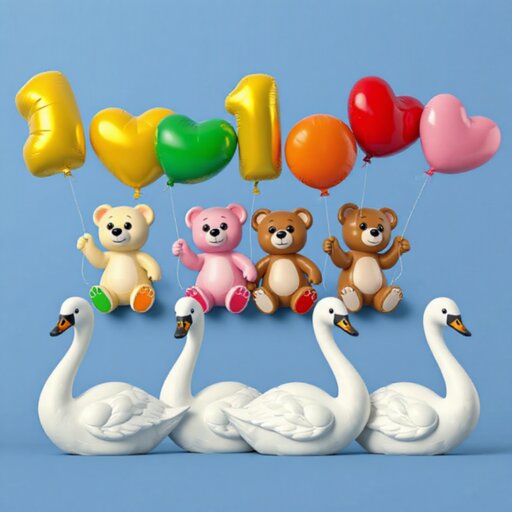} \\    
        \midrule
        \multicolumn{6}{c}{\textit{``Six horses and six deer and four balloons.''}} \\
        \includegraphics[width=0.16\textwidth]{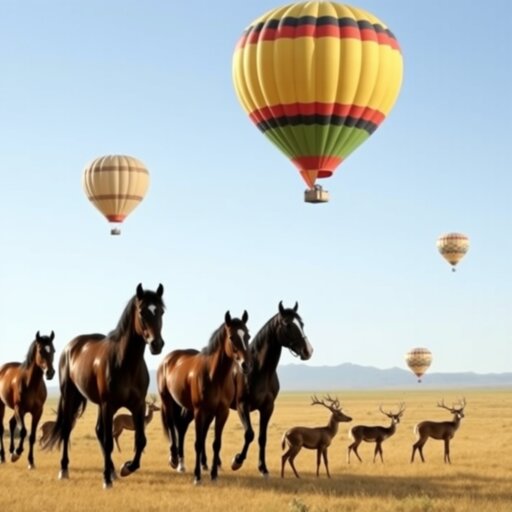} &
        \includegraphics[width=0.16\textwidth]{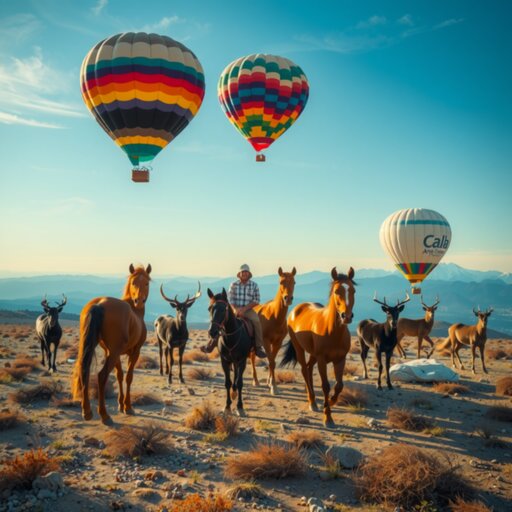} &
        \includegraphics[width=0.16\textwidth]{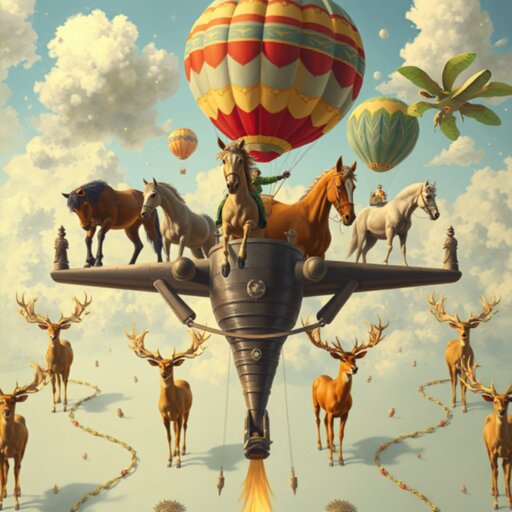} &
        \includegraphics[width=0.16\textwidth]{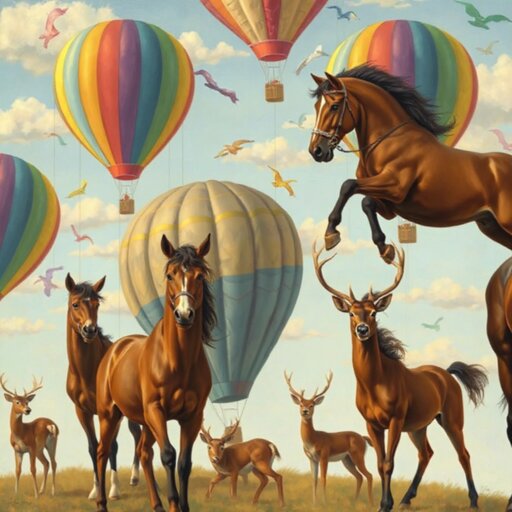} &
        \includegraphics[width=0.16\textwidth]{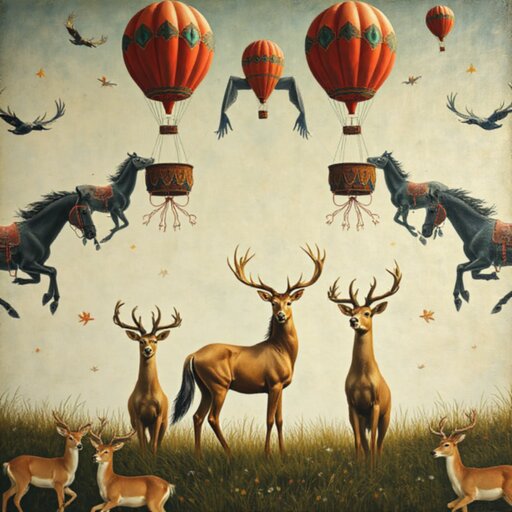} &
        \includegraphics[width=0.16\textwidth]{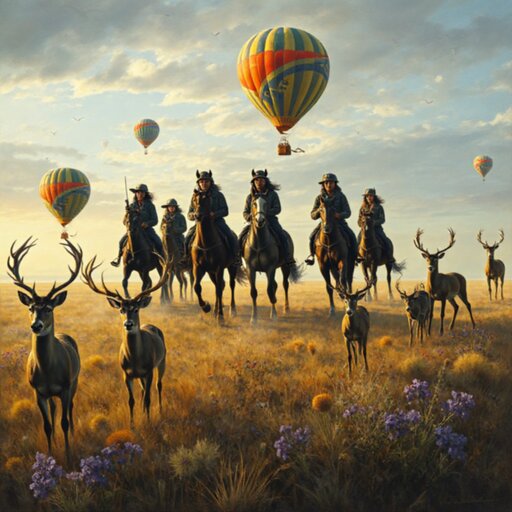} \\
        \midrule
        \multicolumn{6}{c}{\textit{``Eight apples, three bicycles and five rabbits.''}}\\
        \includegraphics[width=0.16\textwidth]{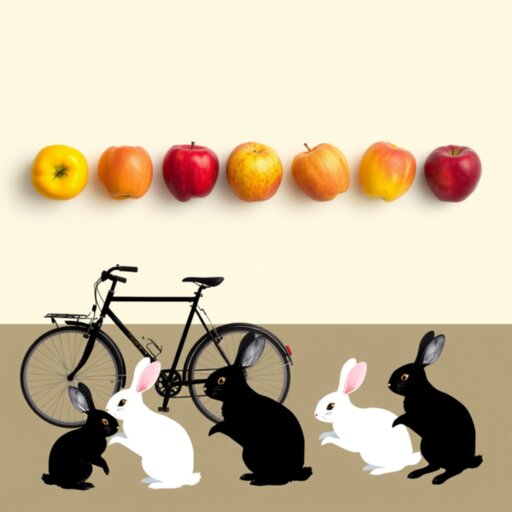} &
        \includegraphics[width=0.16\textwidth]{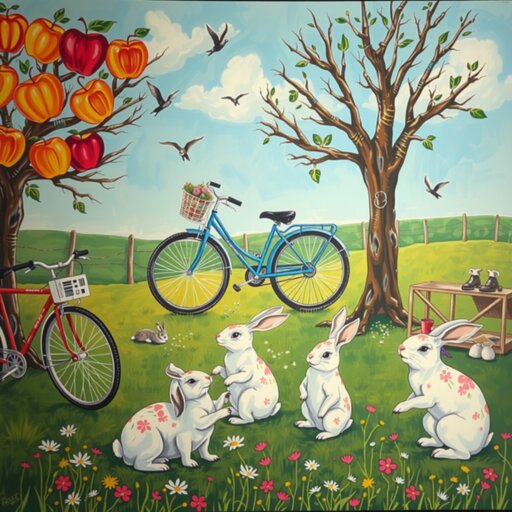} &
        \includegraphics[width=0.16\textwidth]{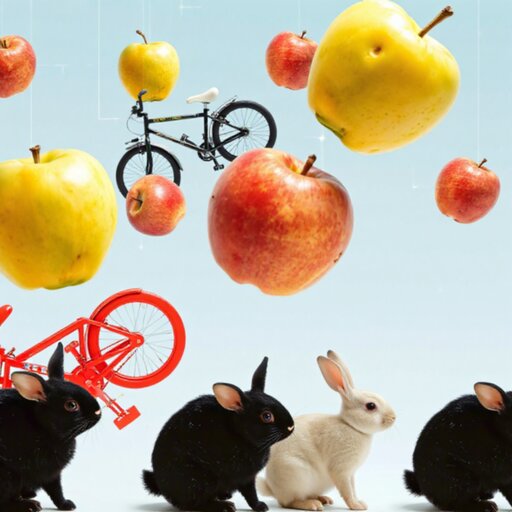} &
        \includegraphics[width=0.16\textwidth]{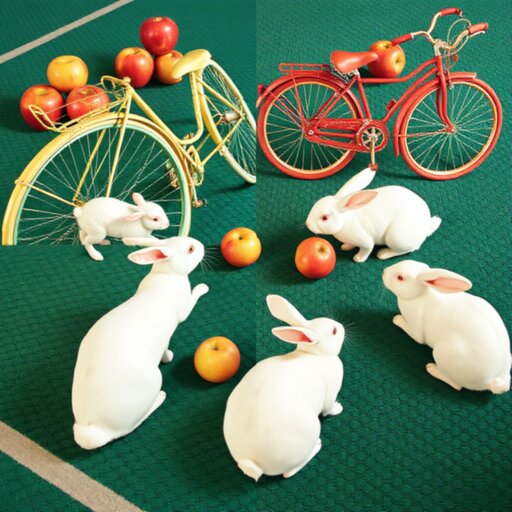} &
        \includegraphics[width=0.16\textwidth]{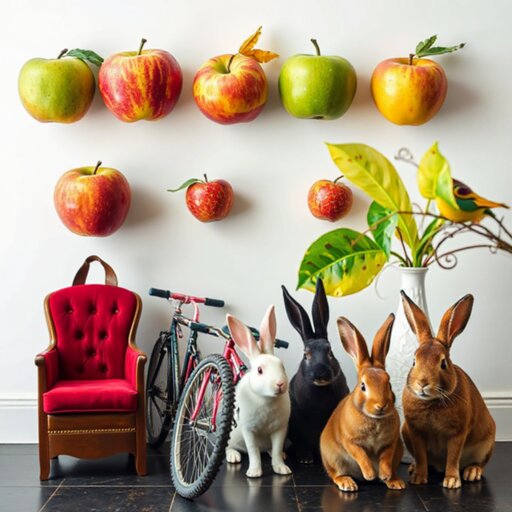} &
        \includegraphics[width=0.16\textwidth]{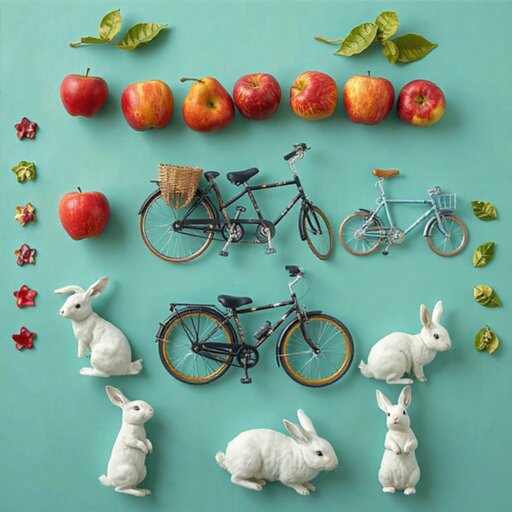} \\
        \midrule
        \multicolumn{6}{c}{\textit{``Six helicopters buzzed over eight pillows.''}}\\
        \includegraphics[width=0.16\textwidth]{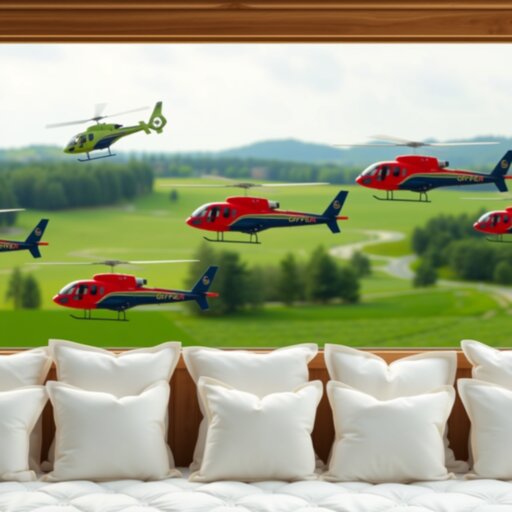} &
        \includegraphics[width=0.16\textwidth]{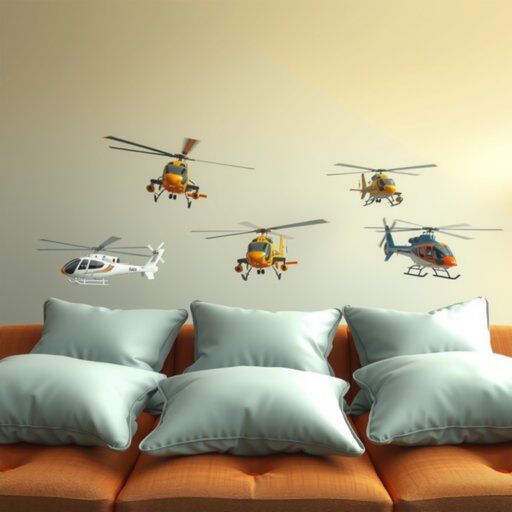} &
        \includegraphics[width=0.16\textwidth]{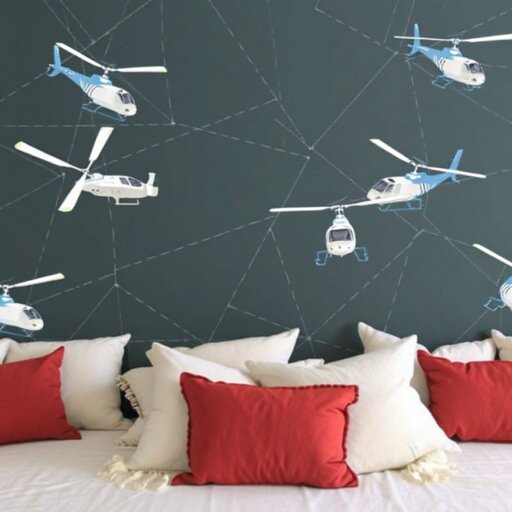} &
        \includegraphics[width=0.16\textwidth]{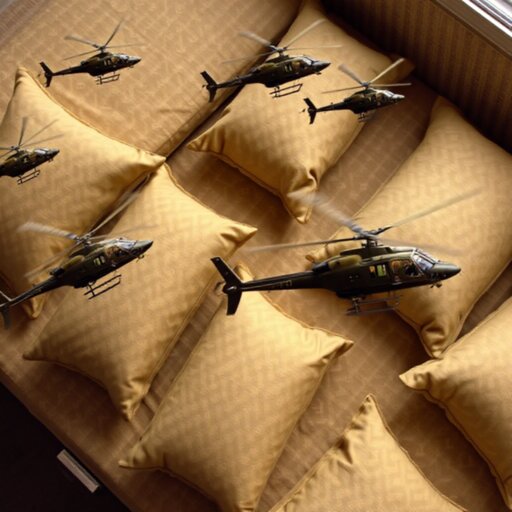} &
        \includegraphics[width=0.16\textwidth]{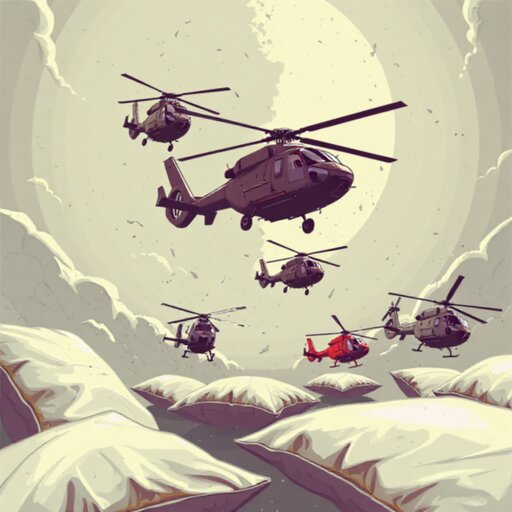} &
        \includegraphics[width=0.16\textwidth]{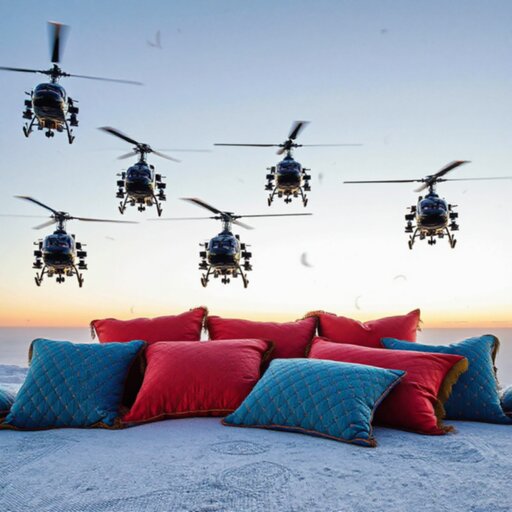} \\
        \midrule
        \multicolumn{6}{c}{\textit{``Five swans and seven ducks swam in the pond.''}}\\
        \includegraphics[width=0.16\textwidth]{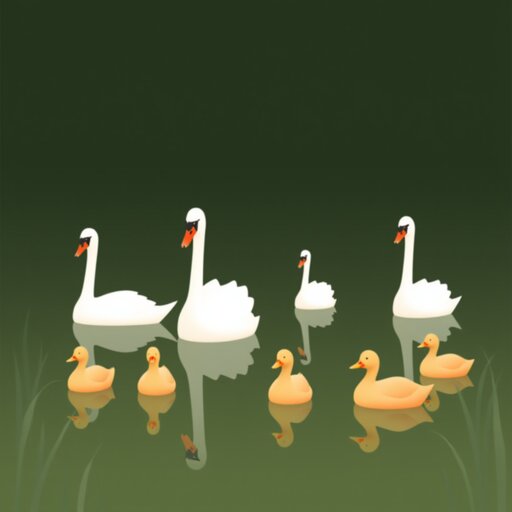} &
        \includegraphics[width=0.16\textwidth]{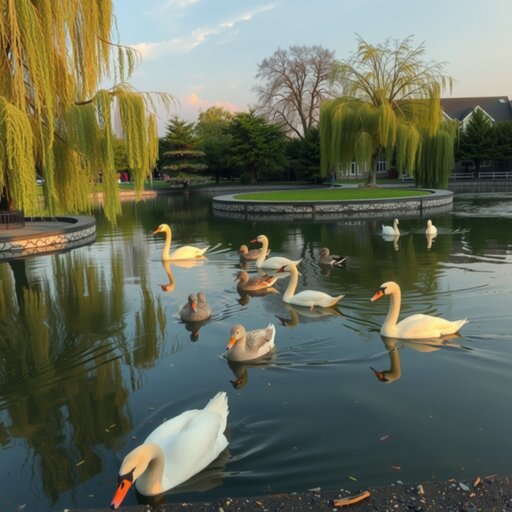} &
        \includegraphics[width=0.16\textwidth]{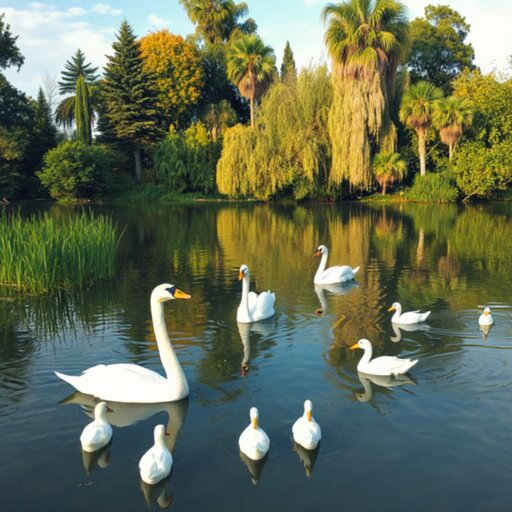} &
        \includegraphics[width=0.16\textwidth]{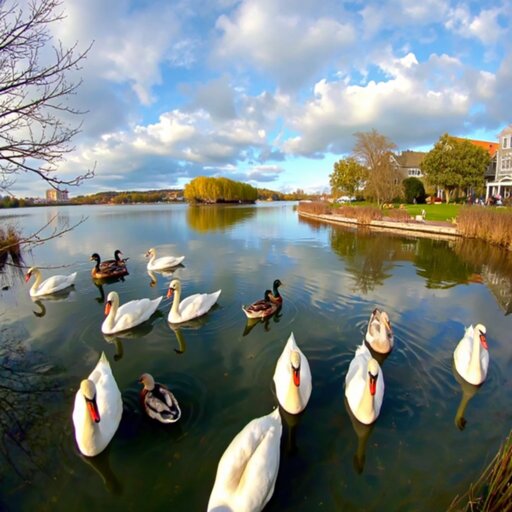} &
        \includegraphics[width=0.16\textwidth]{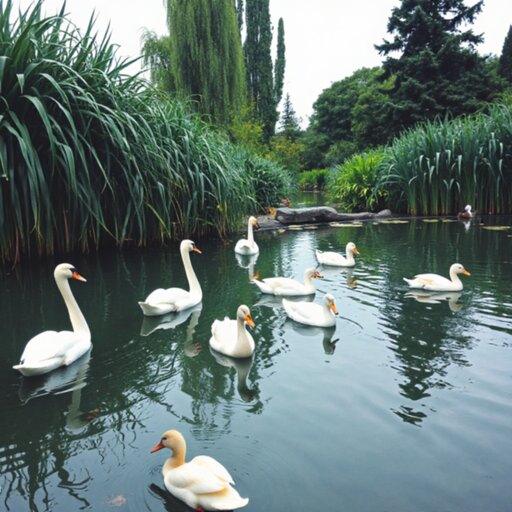} &
        \includegraphics[width=0.16\textwidth]{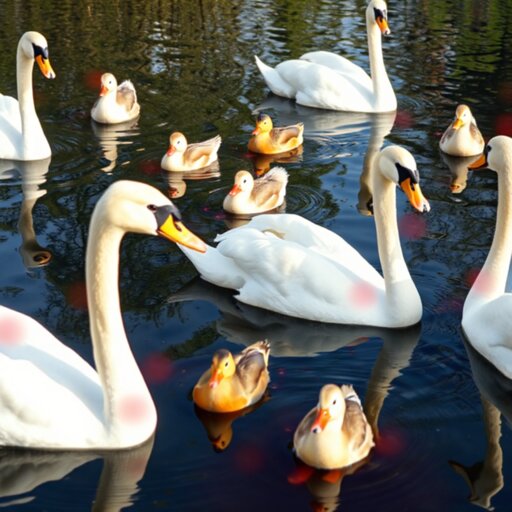} \\
        \midrule
        \multicolumn{6}{c}{\textit{``Four drums, seven tomatoes, and five candles.''}}\\
        \includegraphics[width=0.16\textwidth]{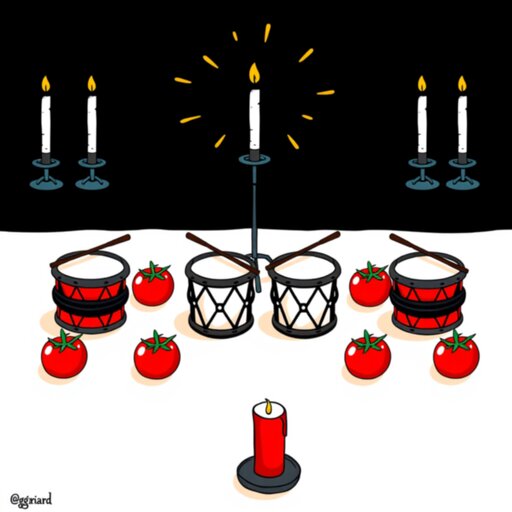} &
        \includegraphics[width=0.16\textwidth]{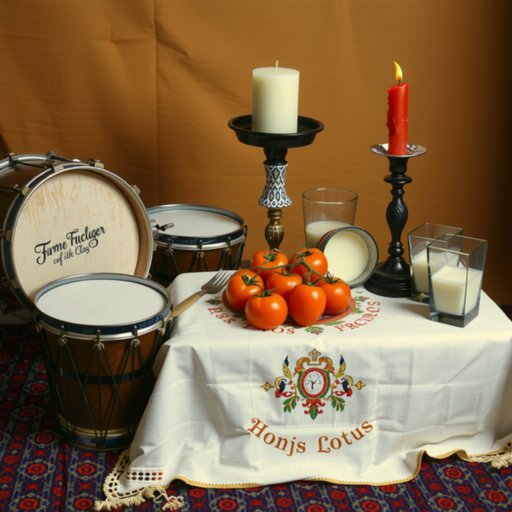} &
        \includegraphics[width=0.16\textwidth]{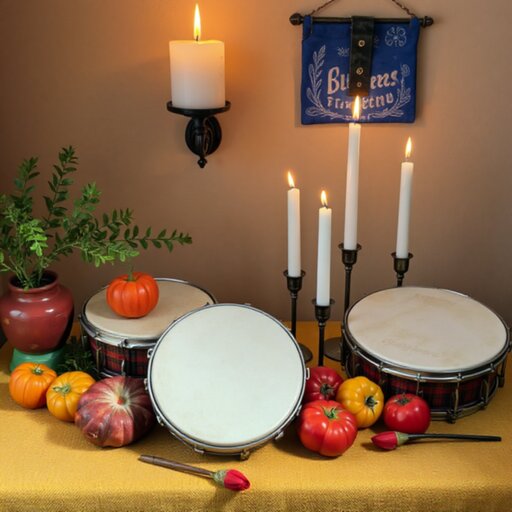} &
        \includegraphics[width=0.16\textwidth]{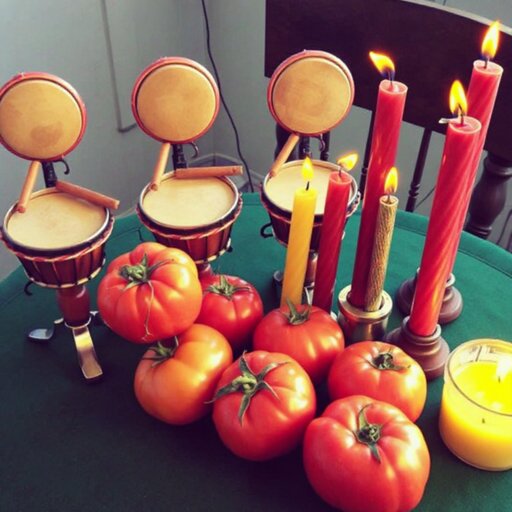} &
        \includegraphics[width=0.16\textwidth]{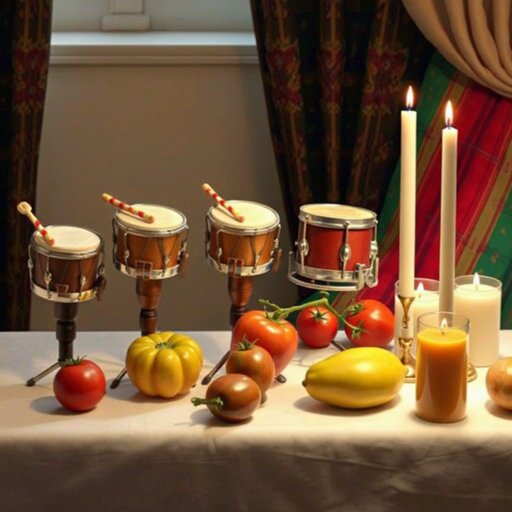} &
        \includegraphics[width=0.16\textwidth]{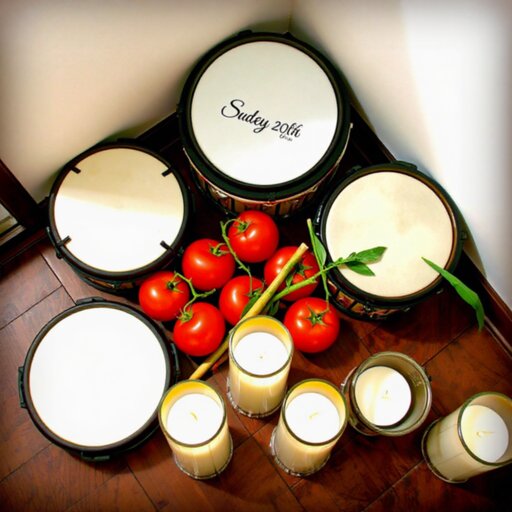} \\
        \midrule
        \multicolumn{6}{c}{\textit{``Three chickens, four birds, and eight pears.''}}\\
        \includegraphics[width=0.16\textwidth]{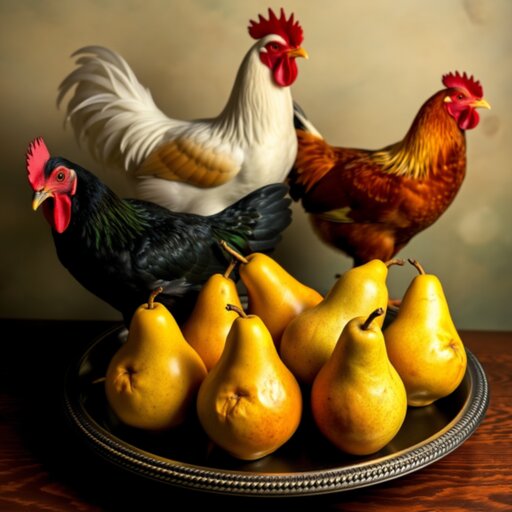} &
        \includegraphics[width=0.16\textwidth]{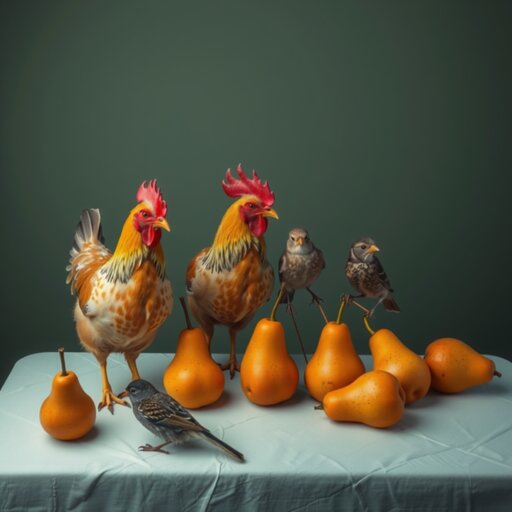} &
        \includegraphics[width=0.16\textwidth]{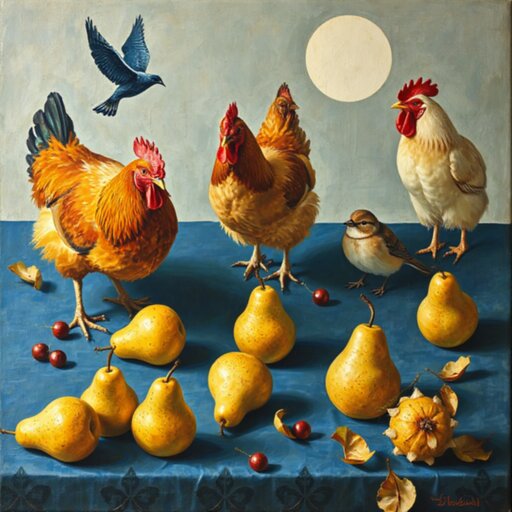} &
        \includegraphics[width=0.16\textwidth]{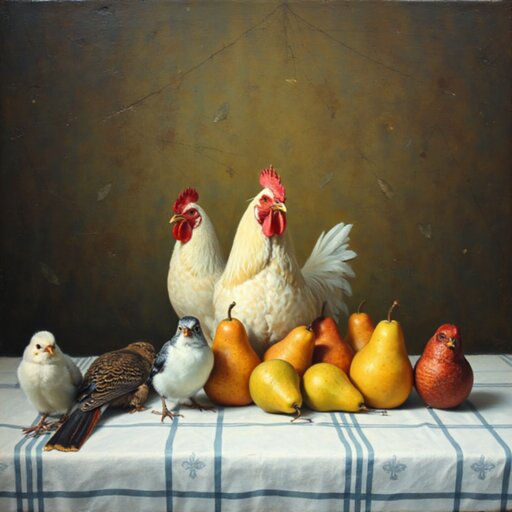} &
        \includegraphics[width=0.16\textwidth]{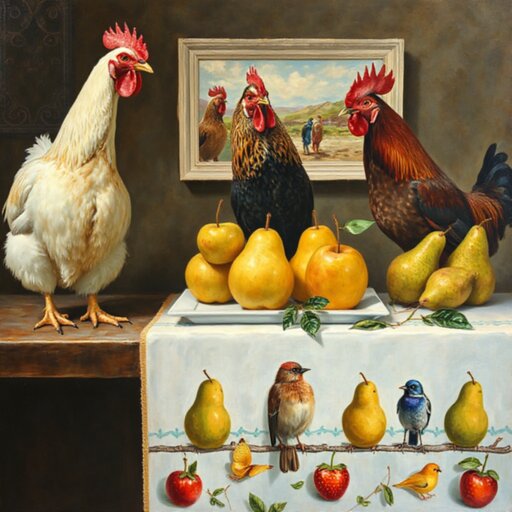} &
        \includegraphics[width=0.16\textwidth]{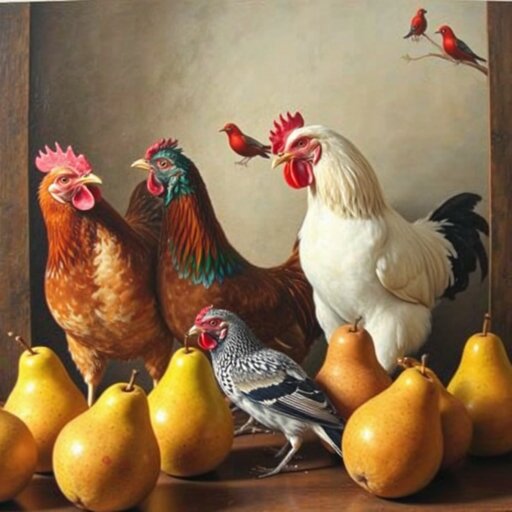} \\\midrule
        \multicolumn{6}{c}{\textit{``Six airplanes flying over a desert with seven camels walking below.''}}\\
        \includegraphics[width=0.16\textwidth]{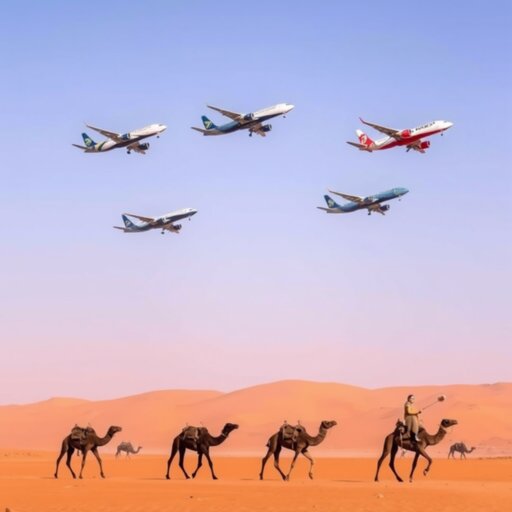} &
        \includegraphics[width=0.16\textwidth]{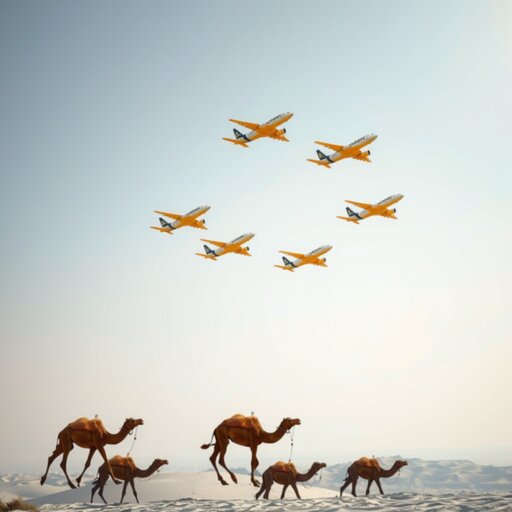} &
        \includegraphics[width=0.16\textwidth]{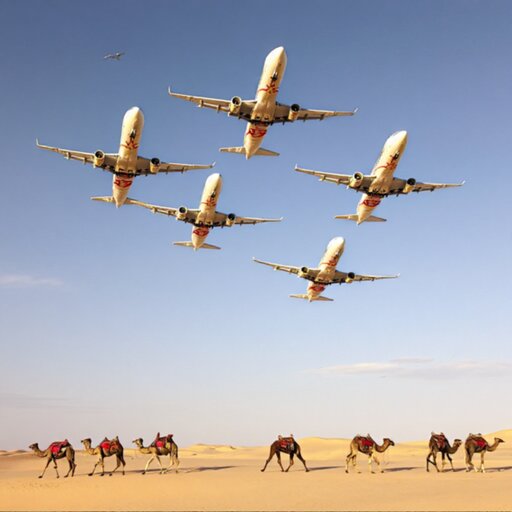} &
        \includegraphics[width=0.16\textwidth]{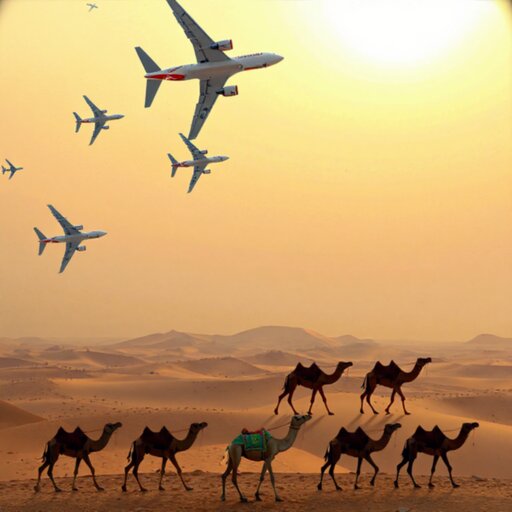} &
        \includegraphics[width=0.16\textwidth]{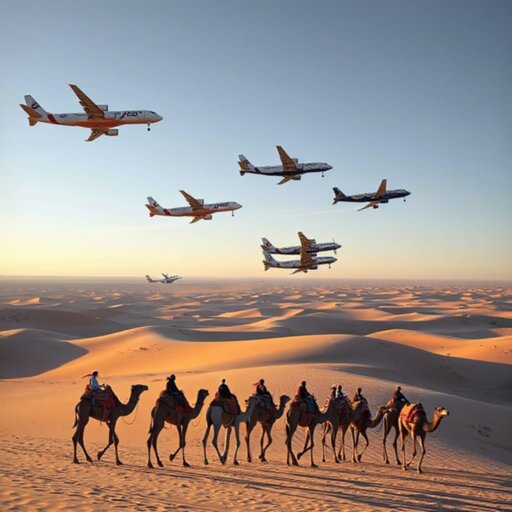} &
        \includegraphics[width=0.16\textwidth]{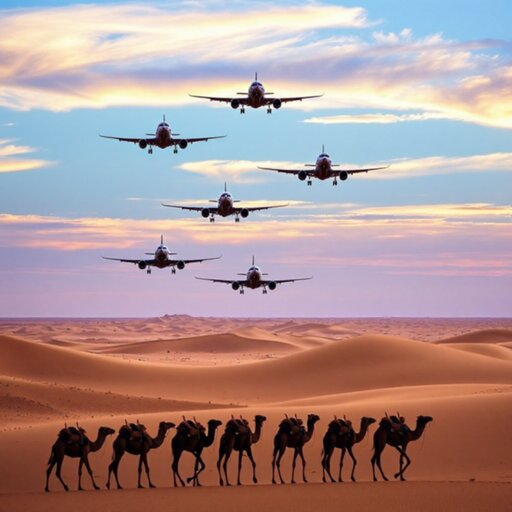} \\
        \bottomrule
    \end{tabularx}
  \caption{\textbf{Additional qualitative results of quantity-aware image generation task}.}
  \label{fig:appendix_counting_methods_vp}
}
\end{figure*}

\clearpage
\newpage
\fi

\end{document}